\documentclass[12pt]{amsart}

\pdfoutput=1

\usepackage{amsmath,amssymb,latexsym}
\usepackage{graphicx}
\newtheorem{theorem}{Theorem}
\newtheorem{lemma}{Lemma}
\newtheorem{definition}{Definition}

\begin{document}

\title{
Clustering, Coding,\\
And the Concept of Similarity}
\author{L. Thorne McCarty\\Department of Computer Science\\Rutgers University}
\date{April, 2018.  \;\copyright\, L. Thorne McCarty,  \texttt{ mccarty@cs.rutgers.edu}}
\maketitle

\markleft{\textsc{L. THORNE MCCARTY}}
\markright{\textsc{CLUSTERING, CODING, AND THE CONCEPT OF SIMILARITY}}

\begin{quotation}
{\bf ABSTRACT:}  This paper develops a theory of \emph{clustering} and \emph{coding} which combines a geometric model with a probabilistic model in a principled way.  The geometric model is a Riemannian manifold with a Riemannian metric, ${g}_{ij}({\bf x})$, which we interpret as a measure of \emph{dissimilarity}.  The probabilistic model consists of a stochastic process with an invariant probability measure which matches the density of the sample input data.  The link between the two models is a potential function, $U({\bf x})$, and its gradient, $\nabla U({\bf x})$.  We use the gradient to define the dissimilarity metric, which guarantees that our measure of dissimilarity will depend on the probability measure.  Finally, we use the dissimilarity metric to define a coordinate system on the embedded Riemannian manifold, which gives us a low-dimensional encoding of our original data.\\

\end{quotation}

\section{Introduction.}
\label{Intro}

Clustering algorithms have been studied for several decades \cite{dudaHart1973}, and they remain one of the main ingredients in unsupervised learning \cite{duda2001Ch10}.   Intuitively, a \emph{cluster} is both a geometric concept (e.g., a low-dimensional region in a high-dimensional space) and a probabilistic concept (e.g., a region of the input space in which the sample data density is high).
  
Recently, a variant of the traditional clustering algorithms has attracted some attention, under the rubric of \emph{manifold learning}:  \cite{TenenbaumEtAl2000}  \cite{Roweis2000}  \cite{belkin2003laplacian}.  In this variant, the learning task is to construct a low-dimensional \emph{manifold}, embedded in the original high-dimensional space, on which the probability density of the input data is high.  For example, in one recent paper, Rifai, et al.\  \cite{nipsRifaiDVBM11}, outline three hypotheses that motivate much of this work:
\begin{quotation}
\begin{enumerate}

\item[1.]  The {\bf semi-supervised learning hypothesis}, according to which learning aspects of the input
distribution $p(x)$ can improve models of the conditional distribution of the supervised
target $p(y | x)$, i.e., $p(x)$ and $p(y | x)$ share something $\dots$ [citations omitted] 

\item[2.]  The {\bf (unsupervised) manifold hypothesis}, according to which real world data presented in
high dimensional spaces is likely to concentrate in the vicinity of non-linear sub-manifolds
of much lower dimensionality $\dots$ [citations omitted] 

\item[3.]  The {\bf manifold hypothesis for classification}, according to which points of different classes
are likely to concentrate along different sub-manifolds, separated by low density regions of
the input space.

\end{enumerate}
\end{quotation}
The authors then present a ``Contractive Auto-Encoder (CAE)'' algorithm to exploit these hypotheses, and they combine this with an existing supervised learning algorithm to produce what they call a ``Manifold Tangent Classifier (MTC)," which performs very well on several datasets.  It is interesting to note that these algorithms are based, explicitly, on concepts from differential geometry, but they draw only implicitly on probability theory.  The informal language of probability theory abounds.  For example, the authors write that the ``data density concentrates near low-dimensional manifolds'' and ``different classes correspond to disjoint manifolds separated by low density'' (see abstract).  But there is no explicit probability model in the paper.  
 
In this paper, we will develop a theory of \emph{clustering} and \emph{coding} which combines a geometric model with a probabilistic model in a principled way.  The geometric model is a Riemannian manifold with a Riemannian metric, ${g}_{ij}({\bf x})$, which we interpret as a measure of \emph{dissimilarity}.  The probabilistic model consists of a stochastic process with an invariant probability measure which matches the density of the sample input data.  The link between the two models is a potential function, $U({\bf x})$, and its gradient, $\nabla U({\bf x})$.  We use the gradient to define the dissimilarity metric, which guarantees that our measure of dissimilarity will depend on the probability measure.  Roughly speaking, the dissimilarity will be small in a region in which the probability density is high, and vice versa.  Finally, we use the dissimilarity metric to define a coordinate system on the embedded Riemannian manifold, which gives us a low-dimensional encoding of our original data.

Section \ref{ThreeProblems} reviews the ``Mathematical Background'' of the paper, including several theorems which will play a central role in the subsequent discussion.  Section \ref{proto} then discusses ``Prototype Coding,'' our overall model, and explains how the dissimilarity metric and the low-dimensional coordinate system are related to the stochastic process with an invariant probability measure.  Section \ref{Exp:IntManifolds} investigates the differential geometry component of the model more carefully, with a focus on the important concept of an ``Integral Manifold.''   At this point in the paper, we restrict our analysis to ${\bf R}^{3}$ rather than ${\bf R}^{n}$, although we will see later (in Section \ref{FutureWork}) that this is not actually a limitation on the scope of the theory.  Instead, the restriction to three dimensions simplifies our calculations, and makes them much easier to visualize.  Accordingly, in Section \ref{Mathematica}, we present the results of a number of experiments using \emph{Mathematica}, including some full-color three-dimensional graphics of several examples which are intended to aid our intuitions about the main elements of the theory.  Section \ref{DiffuCoeffs&DissimMetrics} discusses an interesting technical result, which also helps to link the geometric model to the probabilistic model.  Finally, Section \ref{FutureWork} discusses ``Future Work,'' including a further analysis of the connections between the present theory and the current literature on manifold learning.

\section{Mathematical Background.}
\label{ThreeProblems}

Let's start with a model that will be familiar to most physicists:  the Feynman-Kac formula \cite{RPFeynman1948} \cite{MKac1949}.   We will write this formula as follows:
\begin{equation}
\label{FKformula}
u(t,{\bf x}) \;=\; \int_{\Omega}^{} f(X_{t}) \;\exp \left[ - \int_{0}^{t} V(X_{s}) \;ds \right] \;\mathcal{W}_{\bf x} (dX)
\end{equation}
Here, $X_{t} \equiv X(t,\omega)$ denotes a continuous path in ${\bf R}^{n}$, and $\mathcal{W}_{\bf x}$ denotes Wiener measure over all such paths beginning at $X_{0} = {\bf x}$.  If $V \colon {\bf R}^{n} \to {\bf R}$ is bounded below, then $u(t,{\bf x})$ is a solution to the Cauchy initial value problem:
\begin{equation}
\label{uCauchy}
\frac{\partial u}{\partial t} = \frac{1}{2}\Delta u - V({\bf x})\, u \;\; {\rm with} \;\; u(0,\cdot) = f
\end{equation}
in which $\Delta$ denotes the standard Laplacian in Cartesian coordinates. Conversely, any bounded solution to \eqref{uCauchy} is equal to the function defined by  \eqref{FKformula}.  See \cite{stroock1993}, Section 4.3.   Now, following Feynman's heuristic picture of formula \eqref{FKformula}, we can write a discrete approximation to Wiener measure as:  
\begin{equation*}
 \int \;\exp \left[  - \sum_{k = 1}^{m} \frac{s_{k} - s_{k - 1}}{2} \left(  \frac{ | X(s_{k}) - X(s_{k - 1}) |}{s_{k} - s_{k - 1}} \right)^{2} \right]  
 d X(s_{1}) \ldots d X(s_{m}) ,
\end{equation*}
multiplied by a normalization factor, 
so that the exponential function in the integrand of \eqref{FKformula} could be viewed, in the limit, as:
\begin{equation}
\label{expHamiltonian}
\exp \left[  -  \int_{0}^{t} \frac{1}{2}  { | \dot{X}(s) | }^{2}+ V(X(s)) \;ds \right]  
\end{equation}
See \cite{stroock1993}, Section 4.2, or \cite{stroock2011}, Section 8.1.   The quantity inside the integral sign is, of course, the Hamiltonian of a classical dynamical system with the potential function: $ V({\bf x}) $.  
 
This model obviously possesses some of the properties that we want:  Equations \eqref{FKformula} and \eqref{uCauchy} specify a stochastic process that depends on the potential function, $ V({\bf x}) $, and the exponent in formula \eqref{expHamiltonian} can be interpreted as an expression in differential geometry, which also depends on $ V({\bf x}) $.   Furthermore, the paths that minimize the ``energy'' in \eqref{expHamiltonian} will maximize the probability in \eqref{FKformula}.   Now imagine that we can choose the potential function, $V({\bf x})$, in such a way as to generate an invariant probability measure on ${\bf R}^{n}$. In other words, imagine that we can find a steady-state solution to equation \eqref{uCauchy}.  We can then project our stochastic process onto a nonlinear subspace of ${\bf R}^{n}$ --- i.e., onto an embedded Riemannian manifold --- and examine the probability density induced on that subspace.  Feynman's heuristic picture of the relationship between \eqref{FKformula} and \eqref{expHamiltonian} suggests that the subspaces of maximal probability will also be the subspaces of minimal energy, and the hope is that this will lead us to a solution to the clustering and coding problems in ${\bf R}^{n}$. 

However, there are several problems with this model:
\begin{itemize}
\item First, it is well known that Feynman's heuristic interpretation of formula \eqref{FKformula} is mathematical nonsense, since there is no analogue of Lebesgue measure in an infinite-dimensional space.   The relationship between \eqref{FKformula} and \eqref{uCauchy} holds rigorously, as stated, if $\mathcal{W}_{\bf x}$ is Wiener measure, or \emph{Brownian motion}, but there is still a gulf between \eqref{FKformula} and \eqref{expHamiltonian}.  To interpret the integral in \eqref{expHamiltonian} as an expression in differential geometry, the paths $X_{s} \equiv X(s) \equiv X(s,\omega)$ must be continuous and differentiable.  But, under Wiener measure, \emph{with probability one}, the paths $X(t,\omega)$ are continuous but \emph{nowhere differentiable}.  Thus there is a fundamental clash between the geometric model and the probabilistic model.  Stroock calls this ``a fact which $\ldots$ haunts every attempt to deal with Brownian paths," \cite{stroock1996}, p. 140.
\vspace{1ex}
\item Second, assuming that we can overcome our first problem, it is not a simple matter to project a stochastic process from ${\bf R}^{n}$ onto an embedded Riemannian manifold.  The mathematical problem itself has only been solved, in general, during the course of the past 20 or 30 years, and it is now part of a subject known as \emph{stochastic differential geometry}.  See \cite{emery1989stochastic} or \cite{hsu2002stochastic}.  But the calculations are not trivial. 
\vspace{1ex} 
\item Finally, it would be a mistake to assume that the Feynman-Kac formula can be used directly to generate a stochastic process, with a proper probabilistic interpretation.  Instead, we will need a new potential function, $ U({\bf x}) $, and we will need a further derivation from equations \eqref{FKformula} and \eqref{uCauchy}, in order to construct a stochastic process with an invariant probability measure.  This also means that we will not be able to define our dissimilarity metric, directly, by minimizing the energy functional in formula \eqref{expHamiltonian}.
\end{itemize}
\vspace{1ex}
\noindent
In the remainder of this section, we will address these three problems, in reverse order.    Our analysis will eventually lead us to a modification of the naive Feynman-Kac model, and to the definition of a dissimilarity metric which will achieve the goals articulated in Section \ref{Intro}.
\subsection{A Stochastic Process with an Invariant Measure.} 
\label{StochasticProcessInvariant}  
To see the problem with the basic Feynman-Kac formula, it is helpful to rewrite \eqref{FKformula} using an operator:
\begin{equation}
\label{defPt}
[ {\bf P}_{t} f ]({\bf x})
\;=\; \int_{\Omega}^{} f(X_{t}) \;\exp \left[ - \int_{0}^{t} V(X_{s}) \;ds \right] \;\mathcal{W}_{\bf x} (dX)
\end{equation}
It turns out that $ {\bf P}_{t} {\bf 1} \neq {\bf 1} $, which means that we cannot use this operator to construct a Markov process with a proper probabilistic interpretation.  Another manifestation of the same problem is the fact that $V({\bf x})$ has a natural interpretation as the ``killing rate'' for the process,  i.e., the probability per unit of time that a path starting at $X_{0} = {\bf x}$ will ``die'' by time $\delta t$.  Thus the process ``evaporates'' as time goes by.

To fix this problem, we need a new potential function.  If $\mu$ is a function that satisfies  $\frac{1}{2}\Delta \mu - V({\bf x})\, \mu = 0$, then
\[
V({\bf x}) \;=\; \frac{1}{2} \left(\frac{\Delta \mu}{\mu} \right) \;=\; \frac{1}{2} \left( \Delta \log \mu + \vert \nabla \log \mu \vert^{2} \right)
\]
The first equality is trivial, and the second equality follows from a straightforward computation, e.g., by expanding $ \Delta \log \mu$ in Cartesian coordinates.   This equation suggests that we should work with a potential function $U({\bf x})$ and define $V({\bf x})$ as follows:
\begin{equation}
\label{defVx}
V({\bf x}) \;=\;  \frac{1}{2} \left( \Delta U({\bf x}) + | \nabla U({\bf x}) |^{2} \right)
\end{equation}
Now consider the following initial value problem:
\begin{equation}
\label{wCauchy}
\frac{\partial w}{\partial t} = \frac{1}{2}\Delta w \;+\; \nabla U({\bf x}) \boldsymbol{\cdot} \nabla w \;\; {\rm with} \;\; w(0,\cdot) = f
\end{equation}
\begin{lemma}
$w(t,{\bf x})$ is a solution to \eqref{wCauchy} if and only if $e^{U(\bf x)}w(t,{\bf x})$ is a solution to \eqref{uCauchy} with initial value $u(0,\cdot) = e^{U}f$.
\end{lemma}
\begin{proof}
By a straightforward computation, using the definition in \eqref{defVx} of $V({\bf x})$ in terms of $U({\bf x})$.
\end{proof}
\noindent
We now use both $U$ and $V$ to define a new operator: 
\begin{align}
& [ {\bf Q}_{t} f ]({\bf x}) =  \label{defQt} \\
& \hspace{0.75em} \exp \left[ - U(X_{0}) \right]  \int_{\Omega}^{} f(X_{t}) \; \exp \left[ U(X_{t}) - \int_{0}^{t} V(X_{s}) \;ds \right] \;\mathcal{W}_{\bf x} (dX) \notag
\end{align}
\begin{theorem}
\label{QtfProcess}
If $U$ is bounded above and $V$ is bounded below, and if $w(t,{\bf x})$ is a solution to \eqref{wCauchy} and $e^{U(\bf x)}w(t,{\bf x})$ is also bounded, then $ w(t,{\bf x})$ is equal to $ [ {\bf Q}_{t} f ]({\bf x}) $ as defined in \eqref{defQt}.  Furthermore, $ {\bf Q}_{t} {\bf 1} = {\bf 1} $ for all $t \geq 0$, and $ ( {\bf Q}_{t} ) _{t \geq 0}$ is a semigroup of operators which defines a Markov process on ${\bf R}^{n}$ with an invariant probability measure proportional to $e^{\,2\,U({\bf x})} $.
\end{theorem}
\begin{proof}
See Theorem 4.3.36 in \cite{stroock1993} or Theorem 10.3.33 in \cite{stroock2011}.
\end{proof}

In the literature, \eqref{wCauchy} is known as a \emph{diffusion equation} with a \emph{drift vector} $\nabla U$.  It is a nice feature of our formalism that this drift vector is the gradient of a potential $U({\bf x})$, and that the invariant measure turns out to be an exponential of the potential $U({\bf x})$.  For a numerical example, if $U({\bf x})$ is a negative quadratic polynomial (which would be bounded above), then $V({\bf x})$ would be a positive quadratic polynomial (which would be bounded below), and the invariant measure would be a Gaussian. See Section \ref{Exp:Gauss} below.
 
{\bf Sources:} These results appear in \cite{stroock1993}, Section 4.3, but the analysis there uses a different definition of $V$ in terms of $U$.  In the second edition of his book, Stroock switches to the more natural definition in \eqref{defVx} above, but with the opposite sign.  See \cite{stroock2011}, Section 10.3.  {\O}ksendal also uses this example, with the same definition of $V$  and the same sign, in Exercises 8.15 and 8.16 of his text \cite{oksendal2003}.  
 
\subsection{Mapping a Diffusion to an Embedded Manifold.}
\label{MappingDiffusion}
The equations in the previous section were all expressed in Cartesian coordinates, and the results would be different in a different coordinate system.  For a simple example, if the standard 2-dimensional Laplacian were transformed into polar coordinates, it would acquire an additional first-order ``drift'' term.  This is a problem if we want to map a diffusion from ${\bf R}^{n}$ onto a nonlinear Riemannian manifold.
 
One approach to this problem is to analyze the diffusion by means of a \emph{stochastic differential equation}, in two versions, one due to It\^{o}, and one due to Stratonovich.   We will write a 1-dimensional It\^{o} process as:
\[
X(t) \; = \; X(0) \; + \; \int_{0}^{t} \sigma (s,\omega) \,d\mathcal{B}(s,\omega) 
             \; + \; \int_{0}^{t} b(s,\omega) \,ds
\]
where the first integral is an It\^{o} integral defined with respect to the Brownian motion $ \mathcal{B}(t,\omega)  $, and the second integral is an ordinary Riemann or Lebesgue integral.  In differential notation, this would be:
\begin{equation}
   dX(t) \; = \;  \sigma (t,\omega) \,d\mathcal{B}(t,\omega)  + b(t,\omega) \,dt
\end{equation}
Extending this notation to $n$ dimensions, let $ \mathcal{B}_{1}(t,\omega) , \dots ,  \mathcal{B}_{d}(t,\omega) $ be $d$ independent Brownian motion processes, assume that ${\bf b} \colon {\bf R}^{n} \to {\bf R}^{n}$ and ${\mathbf \sigma} \colon {\bf R}^{n} \to {\bf R}^{n \times d}$ are Lipschitz continuous, and define the $n$-dimensional It\^{o} process as follows:
\begin{equation}
\label{ItoXtDef}
   dX(t) \; = \;
   \begin{pmatrix} 
      \sigma^{1}_{1} & \dots & \sigma^{1}_{d} \\
      \vdots & & \vdots \\
      \sigma^{n}_{1}  & \dots & \sigma^{n}_{d} \\
   \end{pmatrix}
      \begin{pmatrix}
         d\mathcal{B}_{1}(t)  \\
         \vdots \\
         d\mathcal{B}_{d}(t)  \\
      \end{pmatrix}
      \; + \;    
          \begin{pmatrix}              b^{1} \\
             \vdots \\
             b^{n} \\
          \end{pmatrix} 
          dt
\end{equation}
We want to construct a differential operator associated with this process.  Setting ${\bf a} = {\mathbf \sigma} \mathbf{\sigma}^{T}$, define $\mathcal{L}$ for all $f \in C^{2}({\bf R}^{n} ; {\bf R})$ by:
\begin{equation}
\label{LopDef}
 [\mathcal{L} f ]({\bf x}) \;=\; 
 \frac{1}{2} \sum_{i,j} a^{ij} ({\bf x}) \frac{\partial^{2} f}{\partial{x}^{i} \partial{x}^{j}} 
 \; + \;
 \sum_{i} b^{i}({\bf x}) \frac{\partial f}{\partial{x}^{i} }
\end{equation}
\begin{theorem}
\label{Thm:ItoXtDef=LopDef}
The operator $\mathcal{L}$ defined in \eqref{LopDef} is the infinitesimal generator of the $n$-dimensional It\^{o} process given by \eqref{ItoXtDef}.
\end{theorem}
\begin{proof}
See Definition 7.3.1 and Theorem 7.3.3 in \cite{oksendal2003}.
\end{proof}
\noindent
Intuitively, ${\mathbf \sigma}$ is the ``square root'' of {\bf a}.  Note also that, if ${\bf a} = {\mathbf \sigma} \mathbf{\sigma}^{T}$ is the identity matrix and ${\bf b} = \nabla U$, then \eqref{ItoXtDef} and \eqref{LopDef} give us the same stochastic process in ${\bf R}^{n}$ as does \eqref{wCauchy}.
     
For our purposes, however, the It\^{o} process has a defect:  It is not invariant under coordinate transformations.  This can be seen by an examination of \emph{It\^{o}'s formula}, which functions as a ``chain rule'' for the stochastic calculus, but with a second-order correction term.  
Let $F \colon {\bf R}^{n} \to {\bf R}$ be a function with continuous second-order partial derivatives.  Then It\^{o}'s formula asserts that:
\begin{equation}
dF(X(t)) \; = \; \sum_{i}  \frac{\partial F(X(t))}{\partial{x}^{i} } \; dX_{i}(t) \; + \; 
\frac{1}{2} \sum_{i,j} \frac{\partial^{2} F(X(t))}{\partial{x}^{i} \; \partial{x}^{j} } \; dX_{i}(t) \; dX_{j}(t)
\notag
\end{equation}
See \cite{oksendal2003}, Chapter 4.  An alternative is to use the Stratonovich integral, which cancels out the correction term.  A common notational device is to insert the symbol ``$\circ$'' in \eqref{ItoXtDef} to indicate that the stochastic integral is intended to be interpreted in the Stratonovich sense rather than the It\^{o} sense.  Using this notation, the equation for $dF(X(t))$ would be written as:
\begin{equation}
\label{StratFormula}
dF(X(t)) \; = \;  \sum_{i}  \frac{\partial F(X(t))}{\partial{x}^{i} } \circ dX_{i}(t)  
\end{equation}
in accordance with the usual rules of the Newton-Leibniz calculus.  Since $F$ could be an arbitrary coordinate transformation, the use of the Stratonovich formula in \eqref{StratFormula}, instead of It\^{o}'s formula, makes it possible to combine the stochastic calculus with the traditional constructs of Riemannian geometry.

Fortunately, the It\^{o} integral and the Stratonovich integral can be developed in parallel, and it is possible to choose whichever version works best in a particular application.  In the 1-dimensional case, we will write the Stratonovich version of a stochastic process as follows:
\[
X(t) \; = \; X(0) \; + \; \int_{0}^{t} \sigma (s,\omega) \circ \,d\mathcal{B}(s,\omega) 
             \; + \; \int_{0}^{t} \tilde{b}(s,\omega) \,ds
\]
Notice the notation ``$ \circ \,d\mathcal{B}(s,\omega)$'' here, and the use of the function $ \tilde{b}(s,\omega)$ instead of $b(s,\omega)$.  Written as a differential, this would be:
\begin{equation}
   dX(t) \; = \;  \sigma (t,\omega) \circ \,d\mathcal{B}(t,\omega)  + \tilde{b}(t,\omega) \,dt
\end{equation}
Extending this notation to $n$ dimensions, we can define:
\begin{equation}
\label{StratXtDef}
   dX(t) \; = \;
   \begin{pmatrix} 
      \sigma^{1}_{1} & \dots & \sigma^{1}_{d} \\
      \vdots & & \vdots \\
      \sigma^{n}_{1}  & \dots & \sigma^{n}_{d} \\
   \end{pmatrix}
   \circ
      \begin{pmatrix}
         d\mathcal{B}_{1}(t)  \\
         \vdots \\
         d\mathcal{B}_{d}(t)  \\
      \end{pmatrix}
      \; + \;    
          \begin{pmatrix}              \tilde{b}^{1} \\
             \vdots \\
             \tilde{b}^{n} \\
          \end{pmatrix} 
          dt
\end{equation}
\begin{lemma}
\label{convertItoStrat}
The stochastic process defined by the It\^{o} integral in \eqref{ItoXtDef} is identical to the process defined by the Stratonovich integral in \eqref{StratXtDef} if and only if
\begin{equation}
\label{bi_conversion}
\tilde{b}^{i} \;=\;  b^{i} \, - \; \frac{1}{2} \sum_{k = 1}^{d} \sum_{j = 1}^{n} 
  \frac{\partial \sigma^{i}_{k} }{ \partial {x}^{j}}  \sigma^{j}_{k}
\end{equation}
\end{lemma}
\begin{proof}
See \cite{stratonovich1966} or \cite{ito1975}.
\end{proof} 
\noindent
We thus have a simple mapping between the two formalisms, with the advantage that the stochastic differential equation in Stratonovich form is invariant under coordinate transformations.
 
Lemma \ref{convertItoStrat} has an interesting consequence if we start out with the stochastic process given by \eqref{wCauchy}.  Recall that ${\bf a} = {\mathbf \sigma} \mathbf{\sigma}^{T}$ is the identity and ${\bf b} = \nabla U$ in this case.  Suppose we satisfy the condition ${\bf a} = {\bf I}$ by setting ${\mathbf \sigma} = {\bf I}$. Then the second term in \eqref{bi_conversion} vanishes, and $ \tilde {\bf b} = {\bf b}$. However, if we subsequently apply a nonlinear coordinate transformation to our process, or project it onto a nonlinear subspace, then the Ito and Stratonovich equations will diverge, and we will want to use the Stratonovich equation from then on.
    
Let us now reinterpret the preceding analysis as a general property of vector fields.  Define the column vectors
\[
{\bf A}_{0} = 
          \begin{pmatrix}              \tilde{b}^{1} \\
             \vdots \\
             \tilde{b}^{n} \\
          \end{pmatrix} 
\;\; \text{and} \; \;
{\bf A}_{k} = 
   \begin{pmatrix} 
      \sigma^{1}_{k}  \\
      \vdots  \\
      \sigma^{n}_{k}   \\
   \end{pmatrix}
\; \text{for}\; k = 1, \ldots ,d  
\]
and rewrite \eqref{StratXtDef} as:
\begin{equation}
\label{StratXtDefAk}
   dX(t) \; = \;
   \begin{pmatrix} 
      {\bf A}_{1} \vert  \dots \vert {\bf A}_{d} 
    \end{pmatrix}
   \circ
   d\mathcal{B}(t) 
      \; + \;    
      {\bf A}_{0} \,  dt
\end{equation}
We will think of a vector field as a differential operator, essentially the \emph{directional derivative} with respect to a given vector ${\bf V}$.  Let us write this in shorthand notation as ${\bf V} \partial$.  It then makes sense to talk about the ``square'' of a vector field, which we can define as the composition of the differential operator with itself:  $ ( {\bf V} \partial )^{2} = {\bf V} \partial \circ {\bf V} \partial $.  Expanding this formula in a coordinate system, we have:
\begin{align}
\label{V^2}
& \left( \sum_{i} V^{i} \frac{\partial}{\partial {x}^{i}} \right)
\circ
\left( \sum_{j} V^{j} \frac{\partial}{\partial {x}^{j}} \right) \; = \\
& \hspace{1.25in} \sum_{i,j} V^{i} V^{j} \frac{\partial ^{2}}{\partial {x}^{i} \partial {x}^{j}}  
\; + \;
\sum_{i,j} \frac{\partial V^{i}}{\partial {x}^{j}} V^{j} \frac{\partial}{\partial {x}^{i}}
\notag 
\end{align}
Now apply this equation to each of the vector fields ${\bf A}_{k} \partial$. 
\begin{theorem}
\label{Thm:L=sumAk+A0}
If $\mathcal{L}$ is the differential operator associated with the stochastic process defined in \eqref{StratXtDefAk}, then
\begin{equation}
\label{L=sumAk+A0}
\mathcal{L} \;=\;  \frac{1}{2} \sum_{k = 1}^{d} \left( {\bf A}_{k} \partial \right)^{2} + {\bf A}_{0} \partial
\end{equation}
\end{theorem}
\begin{proof}
By a straightforward computation, using \eqref{LopDef}, \eqref{bi_conversion} and \eqref{V^2}.
\end{proof}
 
With Theorem \ref{Thm:L=sumAk+A0} as a guide, we can bypass the It\^{o} or Stratonovich stochastic differential equations entirely, and work directly with vector fields.  This is our second (but closely related) approach to the problem of mapping diffusions to embedded manifolds.  Let ${\bf V}_{0} \partial $ and ${\bf V}_{k} \partial $, for $ k = 1, \ldots, d $, be arbitrary vector fields, and define the differential operator
\begin{equation}
\label{LdefHorm}
\mathcal{L} \;=\;  \frac{1}{2} \sum_{k = 1}^{d} \left( {\bf V}_{k} \partial \right)^{2} + {\bf V}_{0} \partial.
\end{equation}
This is known as the \emph{H\"{o}rmander form} for the operator $\mathcal{L}$, and it, too, can be shown to be invariant under coordinate transformations.  See \cite{hormander1967}.  Thus $\mathcal{L}$ works just as well in an arbitrary manifold $\mathcal{M}$ as it does in ${\bf R}^{n}$ endowed with Cartesian coordinates.  The only condition that we need to impose to guarantee that $\mathcal{L}$, as defined in \eqref{LdefHorm}, gives us a nondegenerate diffusion in $\mathcal{M}$ is to require that the vector fields $ \{ {\bf V}_{1} ({\bf x}) \partial, \ldots ,  {\bf V}_{d} ({\bf x}) \partial \}$ span the tangent space on $\mathcal{M}$ at ${\bf x}$.  For these reasons, Stroock relies on the H\"{o}rmander formalism extensively in his book on the analysis of Brownian paths on Riemannian manifolds \cite{stroock2000}.
 
{\bf Sources:}  For the basic results on stochastic differential equations, using It\^{o}'s formalism, the reader should consult \cite{oksendal2003}, but {\O}ksendal's text provides only a cursory treatment of Stratonovich's formalism.  The original paper by Stratonovich  \cite{stratonovich1966}  is still very readable, but his theory was only given a solid mathematical foundation some years later by It\^{o} \cite{ito1975}. Chapter 8 of \cite{stroock2003} is an excellent contemporary account of Stratonovich's theory, set in a broader context.

 \subsection{Integral Curves and Martingales on Manifolds.}
 \label{IntegralCurvesMartingales}
There remains the problem that ``haunts every attempt to deal with Brownian paths," \cite{stroock1996}, p. 140.  How do you reconcile the ``smooth'' curves of differential geometry with the ``rough'' paths that provide the support for Wiener measure?  One answer, suggested by Stroock, 
emerges from a study of the relationship between the integral curves of a vector field and the concept of a \emph{martingale}.
 
Let's examine this idea, first, in the ordinary Euclidean space ${\bf R}^{n}$.  Roughly speaking, a (continuous parameter) martingale $M_{t}$ is a stochastic process which is ``conditionally constant'' in the sense that
\[
\mathbb{E} [ \; M_{t} \;\vert\; \mathcal{F}_{s} \; ] = M_{s} \;\text{ for all }\; 0 \leq s \leq t,
\]
where the conditional expectation $ \mathbb{E} $ is taken with respect to an nondecreasing family of sub-$\sigma$-algebras $ \{ \mathcal{F}_{s} \}_{s \geq 0} $ with the property that each $ M_{t} $ is $ \mathcal{F}_{t} $-measurable.  Since we are only considering probability spaces $( \Omega , \mathcal{F} ,  \mathbb{P} )$ in which $ \Omega $ is the set of continuous paths in ${\bf R}^{n}$ and for which the $\sigma$-algebras $\mathcal{F}$ and $ \{ \mathcal{F}_{s} \}_{s \geq 0} $ are fixed, we will suppress these references in our notation, and refer simply to a ``martingale with respect to $\mathbb{P}$,'' or a $\mathbb{P}$-martingale.  We are interested in the relationship between martingales and differential operators.
\begin{definition}
Let $\mathcal{L}$ be a second-order differential operator, and let $\mathbb{P}_{{\bf x}}$ be a probability measure on the space $C([0,\infty) ; {\bf R}^{n} )$ of all continuous paths in ${\bf R}^{n}$ such that $  \mathbb{P}_{{\bf x}} ( X_{0} = {\bf x} ) = 1 $.  We say that $\mathbb{P}_{{\bf x}}$ \emph{solves the martingale problem for $\mathcal{L}$ starting at {\bf x}} if
\[
M_{t} \;\equiv\; f(X_{t}) - \int_{0}^{t} [ \mathcal{L} f  ] (X_{s}) ds
\]
is a $\mathbb{P}_{{\bf x}}$-martingale for every $f \in C^{\infty}( {\bf R}^{n} ; {\bf R} )$.
\end{definition}
\noindent
Not surprisingly:
\begin{lemma}
If $\mathcal{L} = \frac{1}{2} \Delta$, then the Wiener measure $\mathcal{W}_{\bf x}$ solves the martingale problem for $\mathcal{L}$ starting at {\bf x}.
\end{lemma}
\begin{proof}
See Corollary 7.1.20 and Remark 7.1.23 in \cite{stroock1993}.
\end{proof}
\noindent
Let us now consider the operator $ \mathcal{L} =  {\bf b} \boldsymbol{\cdot} \nabla $ and the integral equation:
\begin{equation}
\label{intcurveYt}
Y_{t} \;=\;  {\bf x} \;+\; \int_{0}^{t} {\bf b}( Y_{s} ) \, ds , \;\;  0 \leq t ,
\end{equation}
where $Y_{t} \equiv Y(t) $ is a continuous path in ${\bf R}^{n}$.  An equivalent differential equation is:
\begin{align}
\label{odeYt}
Y'(t)  \;&=\; {\bf b}( Y(t) ) \\ Y(0)  \;&=\; {\bf x} \notag
\end{align}
By the existence and uniqueness theorem for ordinary differential equations, \eqref{intcurveYt} and \eqref{odeYt}  have a unique solution, which would commonly be referred to as the \emph{integral curve} of the vector field $ {\bf b}$ starting at {\bf x}.  Intuitively, an integral curve is a curve whose tangent is identical to the given vector field at each point. Note, too, that an integral curve is a ``smooth'' curve if $ {\bf b}$ is a smooth vector field. We have the following result:
\begin{lemma}
\label{lemmaL=b}
Let $ \mathcal{L} =  {\bf b} \boldsymbol{\cdot} \nabla $, and let $ \mathbb{P}_{{\bf x}} $ be the unit point mass concentrated on the solution to \eqref{intcurveYt} or \eqref{odeYt} .  Then $ \mathbb{P}_{{\bf x}} $ solves the martingale problem for $\mathcal{L}$ starting at {\bf x}.
\end{lemma}
\begin{proof}
See Exercise 7.1.32 in \cite{stroock1993}.
\end{proof}
\noindent
We now put these two examples together, and consider the differential operator:
\begin{equation}
\label{LopDel+b}
\mathcal{L} \;=\; \frac{1}{2}\Delta  \;+\; {\bf b} \boldsymbol{\cdot} \nabla 
\end{equation}
along with the stochastic process determined by the integral equation:
\begin{equation}
\label{inteqnYtXt}
Y_{t} \;=\;  X_{t} \;+\; \int_{0}^{t} {\bf b}( Y_{s} ) \, ds , \;\;  0 \leq t 
\end{equation}
In this equation, we are assuming that $ X_{t} $ is our original stochastic process with the usual Wiener measure $\mathcal{W}_{\bf x}$, and $ Y_{t} $ is a derived process with a derived probability measure.
\begin{theorem}
\label{theoremL=Delta+b}
Let $ \mathcal{L} $ be the differential operator given by \eqref{LopDel+b}, and let $ \mathbb{Q}_{{\bf x}} $ be the probability measure determined by \eqref{inteqnYtXt} when $ X_{t} $ is a stochastic process whose probability law is given by Wiener measure.  Then $ \mathbb{Q}_{{\bf x}} $ solves the martingale problem for $\mathcal{L}$ starting at {\bf x}.
\end{theorem}
\begin{proof}
See Theorem 7.3.10 in \cite{stroock1993}.
\end{proof}
\noindent
Intuitively, these results show that a stochastic process defined by \eqref{wCauchy}, or \eqref{ItoXtDef}, or \eqref{StratXtDef}, has a ``pure'' diffusion part and a ``pure'' drift part, and the drift part follows the integral curve of the drift vector.  
 
The preceding analysis is not confined to Euclidean ${\bf R}^{n}$, since a similar construction works when $ \mathcal{L} $ is given in H\"{o}rmander form by \eqref{LdefHorm}, see \cite{stroockTan1994}, and this means that all results can be replicated in an arbitrary Riemannian manifold, see \cite{stroockTan1996}.  The theory is explicated further in \cite{stroock2000}, where it serves as the foundation for Stroock's construction and analysis of Brownian motion on a Riemannian manifold.  Specifically, Section 2.2.1 of \cite{stroock2000} includes a generalization of Lemma \ref{lemmaL=b} above, and Theorem 2.40 of \cite{stroock2000} is a generalization of Theorem \ref{theoremL=Delta+b}.

\section{Prototype Coding.}
\label{proto}
 
In discussing the mathematical background of the paper in the previous section, we were actually developing, implicitly, the main elements of our geometric and probabilistic models.  The potential function, $U({\bf x})$, and its gradient, $\nabla U({\bf x})$, were introduced in connection with equations \eqref{wCauchy} and \eqref{defQt} and Theorem  \ref{QtfProcess}.   Equation \eqref{wCauchy} is a diffusion equation with a drift vector, $\nabla U({\bf x})$, and it has an invariant probability density equal to $ e^{\,2\,U({\bf x})} $, modulo a normalization factor.   The stochastic process described by equation \eqref{wCauchy} can also be written as an It\^{o} process, using equations \eqref{ItoXtDef} and \eqref{LopDef} and Theorem \ref{Thm:ItoXtDef=LopDef}, or it can be written in Stratonovich form, using equation \eqref{StratXtDef} and Lemma \ref{convertItoStrat}.  An alternative view of equation \eqref{wCauchy} is given by Stroock's result, Theorem \ref{theoremL=Delta+b}, on the relationship between integral curves and martingales on manifolds.  
 
Recall that the main goal of our theory is to construct a lower-dimensional subspace of the original Euclidean space, ${\bf R}^{n}$, which is ``optimal'' in some sense. To be specific, let's say that the subspace should be a $k$-dimensional Riemannian manifold, embedded in ${\bf R}^{n}$, with a local coordinate system centered at $(0,0, \dots , 0)$.   We will use a form of \emph{prototype coding} for the coordinate system, measuring the \emph{distance} from the \emph{origin} (i.e., the ``prototype'') in $k-1$ specified \emph{directions}.  Extending this coordinate system to all of ${\bf R}^{n}$, we can assume that these $k-1$ coordinate directions have been chosen from among $n-1$ coordinate directions in the full space.  We will now follow the strategy suggested at the beginning of Section \ref{ThreeProblems} for the naive Feynman-Kac model.   Choose $U({\bf x})$ and $\nabla U({\bf x})$ so that the invariant probability density for the stochastic process given by equation \eqref{wCauchy} matches the density of our sample input data in ${\bf R}^{n}$.  We can then \emph{project} this stochastic process onto the embedded $k$-dimensional manifold, and examine the probability density induced on that manifold.  The hope is that this procedure will lead us to the ``best'' $k$-dimensional coordinate system for the purpose of encoding our initial data.  
 
How to do this? Our first step was described briefly in the text following Theorem  \ref{Thm:ItoXtDef=LopDef} above.  We start with \eqref{wCauchy}: a diffusion equation with a drift vector, $\nabla U({\bf x})$.  We then write the differential operator associated with \eqref{wCauchy} in the form given by \eqref{LopDef}:
\begin{equation}
 \mathcal{L}  \;=\; 
 \frac{1}{2} \sum_{i,j} a^{ij} ({\bf x}) \frac{\partial^{2}}{\partial{x}^{i} \partial{x}^{j}} 
 \; + \;
 \sum_{i} b^{i}({\bf x}) \frac{\partial}{\partial{x}^{i} }
 \notag 
\end{equation}
by setting ${\bf a}({\bf x})$ equal to the identity matrix, and setting ${\bf b}({\bf x}) = \nabla U({\bf x})$.  By Theorem \ref{Thm:ItoXtDef=LopDef}, $\mathcal{L}$ is the infinitesimal generator of the $n$-dimensional It\^{o} process given by \eqref{ItoXtDef}:
\begin{equation}
   dX(t) \; = \;
     \begin{pmatrix}
    \\
   {\mathbf \sigma}^{i}_{k}({\bf x}) \\
    \\
    \end{pmatrix}
      \begin{pmatrix}
         d\mathcal{B}_{1}(t)  \\
         \vdots \\
         d\mathcal{B}_{n}(t)  \\
      \end{pmatrix}
      \; + \;    
          \begin{pmatrix}  
               b^{1}({\bf x}) \\
             \vdots \\
              b^{n}({\bf x}) \\
          \end{pmatrix} 
          dt  \notag
\end{equation}
The choice of ${\mathbf \sigma}({\bf x})$  is arbitrary, as long as ${\bf a}({\bf x}) = {\mathbf \sigma}({\bf x}) \mathbf{\sigma}({\bf x})^{T}$ is the identity matrix, which means that ${\mathbf \sigma}({\bf x})$ must be an orthogonal transformation.  These equations are expressed in Cartesian coordinates.
 
To implement the idea of prototype coding, suppose we are given a radial coordinate, $\rho$, and the directional coordinates $\theta^{1},    
\theta^{2}, \ldots, \theta^{n - 1}$.  For convenience, we will use the symbol $\Theta$ to refer to the entire sequence of directional coordinates. Assume the existence of $n$ coordinate transformation functions, with the usual properties:
\begin{align}
x^{1} \; = \; & {\bf x}^{1}( \rho, \theta^{1}, \theta^{2}, \ldots, \theta^{n - 1}) \notag \\
x^{2} \; = \; & {\bf x}^{2}( \rho, \theta^{1}, \theta^{2}, \ldots, \theta^{n - 1}) \notag \\
\ldots \notag \\
x^{n} \; = \; & {\bf x}^{n}( \rho, \theta^{1}, \theta^{2}, \ldots, \theta^{n - 1}) \notag 
\end{align}
Let  $\mathbf{ J }(\rho, \Theta)$ denote the Jacobian matrix of these transformation functions.  We want to represent our stochastic process in this new coordinate system, and to do so we need to convert the It\^{o} equation, given by \eqref{ItoXtDef}, into a Stratonovich equation in the form given by \eqref{StratXtDef}.  We have two equalities:  
 \begin{align}
\label{dXtStrat}
   dX(t) & \; = \; 
     \begin{pmatrix}
    \\
   {\mathbf \sigma}^{i}_{k}({\bf x}) \\
    \\
    \end{pmatrix}
     \circ
      \begin{pmatrix}
         d\mathcal{B}_{1}(t)  \\
         \vdots \\
         d\mathcal{B}_{n}(t)  \\
      \end{pmatrix}
      \; + \;    
          \begin{pmatrix}  
               \tilde{b}^{1}({\bf x}) \\
             \vdots \\
              \tilde{b}^{n}({\bf x}) \\
          \end{pmatrix} 
          dt  
        \\
      & \notag \\
\label{dXtJacobian}
 dX(t)  &  \; = \;       
      \begin{pmatrix}
    \\
    \mathbf{ J }(\rho, \Theta) \\
    \\
    \end{pmatrix}
   \circ
      \begin{pmatrix}
         dX_{\rho}(t)  \\
         dX_{\theta ^{1}}(t)  \\
             \vdots \\
         dX_{\theta ^{n - 1}}(t)  \\
      \end{pmatrix}  
\end{align}
The first equality is justified by Lemma \ref{convertItoStrat}.  The second equality is justified by the Stratonovich formula for the ``chain rule,'' given by \eqref{StratFormula}. The notation  $dX_{\rho}(t),  dX_{\theta ^{1}}(t),   \ldots,  dX_{\theta ^{n - 1}}(t)$, in the second equation, expresses the fact that $X_{\rho}(t),  X_{\theta ^{1}}(t),   \ldots,$  and $X_{\theta ^{n - 1}}(t)$ are intended to represent the components of a new stochastic process defined on ($\rho, \Theta$).

We can now combine and solve equations \eqref{dXtStrat} and \eqref{dXtJacobian} to obtain:
 \begin{align}
       \begin{pmatrix}
         dX_{\rho}(t)  \\
         dX_{\theta ^{1}}(t)  \\
             \vdots \\
         dX_{\theta ^{n - 1}}(t)  \\
      \end{pmatrix} 
      \; = \; 
      &
       \begin{pmatrix}
   \\
    \mathbf{ J }(\rho, \Theta) \\
    \\
   \end{pmatrix} ^{-1}
     \begin{pmatrix}
    \\
   {\mathbf \sigma}^{i}_{k}({\bf x}(\rho, \Theta)) \\
    \\
    \end{pmatrix}
     \circ
      \begin{pmatrix}
         d\mathcal{B}_{1}(t)  \\
         \vdots \\
         d\mathcal{B}_{n}(t)  \\
      \end{pmatrix} \;\; + \;\;  \notag \\
      & \notag \\
      &   
       \begin{pmatrix}
   \\
    \mathbf{ J }(\rho, \Theta) \\
    \\
   \end{pmatrix} ^{-1}
             \begin{pmatrix}  
               \tilde{b}^{1}({\bf x}(\rho, \Theta)) \\
             \vdots \\
              \tilde{b}^{n}({\bf x}(\rho, \Theta)) \\
          \end{pmatrix} 
          dt \notag
 \end{align}
We thus have a representation of our original stochastic process, in Stratonovich form, but expressed entirely in the new ($\rho, \Theta$) coordinate system.  Note that the second term in this solution is just the transformation law for a contravariant vector, or a type $(1,0)$ tensor.
 
Now consider the decomposition of a Stratonovich stochastic differential equation as in \eqref{StratXtDefAk}:
\begin{equation}
       \begin{pmatrix}
         dX_{\rho}(t)  \\
         dX_{\theta ^{1}}(t)  \\
             \vdots \\
         dX_{\theta ^{n - 1}}(t)  \\
      \end{pmatrix} 
   \; = \;
   \begin{pmatrix} 
      {\bf A}_{1} \vert  \dots \vert {\bf A}_{n} 
    \end{pmatrix}
   \circ
   d\mathcal{B}(t) 
      \; + \;    
      {\bf A}_{0} \,  dt \notag
\end{equation}
By matching the components of this equation with the components of the preceding equation, we can determine the vector fields ${\bf A}_{0} \partial$ and ${\bf A}_{1}\partial, \ldots , {\bf A}_{n} \partial$.  Then, applying Theorem \ref{Thm:L=sumAk+A0} and expanding the expression inside \eqref{L=sumAk+A0}, we can compute a new infinitesimal generator, $\mathcal{L}$, for our stochastic process, expressed again entirely in the ($\rho, \Theta$) coordinate system.  Finally, whatever our result might be, it can be written in the following form:
\begin{equation}
\label{DiffusionCoeffs}
 \mathcal{L}  \;=\; 
\frac{1}{2} \sum_{i,j=0}^{n - 1} 
 \alpha^{ij} (\rho, \Theta) \frac{\partial^{2}}{\partial u^{i} \partial u^{j}} 
\, + \,
 \sum_{i=0}^{n - 1} \beta^{i}(\rho, \Theta) \frac{\partial}{\partial u^{i} } 
\end{equation}
where $u^{0} = \rho$ and  $u^{i} = \theta^{i}$, for $i = 1, \ldots, n - 1$.  (To distinguish this equation for $\mathcal{L}$ from the $\mathcal{L}$ we started out with, we have written the coefficients of the differential operators as $ \alpha^{ij} (\rho, \Theta)$ and $ \beta^{i}(\rho, \Theta)$ instead of $ a^{ij} ({\bf x})$ and $b^{i}({\bf x})$.)   Note that this is the infinitesimal generator of an It\^{o}  process, but we derived it by an excursion through Stratonovich!

Before proceeding further, we need to analyze the ($\rho, \Theta$) coordinate system.  How is it defined?  What are its properties? First, we want the radial coordinate, $\rho$, to follow the drift vector, $\nabla U({\bf x})$. We have already seen how to do this. Suppose $\hat{\rho}(t)$ is the solution to the following differential equation, based on \eqref{odeYt}:
\begin{align}
\hat{\rho}'(t)  \;&=\; \nabla U( \hat{\rho}(t) ) \notag \\ 
\hat{\rho}(0)  \;&=\; {\bf x}_{0} \notag
\end{align}
In other words, $\hat{\rho}(t)$ is the \emph{integral curve} of the vector field $\nabla U({\bf x})$ starting at ${\bf x}_{0}$.  This is almost the construction that we want for our radial coordinate, but not quite.  We will actually work with a generalization of the concept of an integral curve, known as an \emph{integral manifold}.  A one-dimensional integral manifold is, roughly speaking, just the image of an integral curve without the parametrization, and it always exists, for any vector field.  Since we want to be able to alter the parametrization of $\hat{\rho}(t)$, arbitrarily, in order to choose a suitable coordinate, $\rho$, the one-dimensional integral manifold is the device that we need.
 
For the directional coordinates, $\theta^{1}, \theta^{2}, \ldots, \theta^{n - 1}$, the obvious generalization would be an integral manifold of dimension $n - 1$, orthogonal to the integral manifold for $\rho$.   But, for $k \geq 2$, a $k$-dimensional integral manifold exists if and only if certain conditions are satisfied, known as the \emph{Frobenius integrability conditions}.  Fortunately, as we will see, if we are looking for an integral manifold orthogonal to a vector field that is proportional to the gradient of a potential function, such as $\nabla U({\bf x})$, then the Theorem of Frobenius gives us the results that we want.  Our analysis here is based on the standard literature in differential geometry. See, e.g., \cite{spivakcomprehensive}, Chapter 6; \cite{bishop1968tensor}, Chapter 3;  \cite{auslander1977introduction}, Chapter 8.   We will discuss these results in Section \ref{Exp:IntManifolds}.  
 
To summarize:  At this point, we have a one-dimensional integral manifold for the $\rho$ coordinate, and an orthogonal $n - 1$ dimensional integral manifold for the $\Theta$ coordinates.   But we want to construct a \emph{lower}-dimensional subspace by projecting our stochastic process onto a $k - 1$ dimensional subset of the coordinates  $\theta^{1}, \theta^{2}, \ldots, \theta^{n - 1}$.   Taken together with the $\rho$ coordinate, we want this operation to give us an ``optimal'' $k$ dimensional subspace.  The mathematical device that we need is a Riemannian metric, ${g}_{ij}({\bf x})$, which we will use to measure \emph{dissimilarity} on the integral manifolds.  And crucially:  \emph{the dissimilarity metric should depend on the probability measure}.  Roughly speaking, the dissimilarity should be small in a region in which the probability density is high, and large in a region in which the probability density is low.  We can then take the following steps:
\begin{itemize}

\item To find a principal axis for the $\rho$ coordinate, we minimize the Riemannian distance, ${g}_{ij}({\bf x})$, along the drift vector.

\item  To choose the principal directions for the $\theta^{1}, \theta^{2}, \ldots, \theta^{k - 1}$ coordinates, we diagonalize the Riemannian matrix, $\left( \,{g}_{ij}({\bf x})\, \right)$, and we use the eigenvectors of this matrix to compute the $k-1$ ``smallest'' infinitesimal initial directions.

\item  To compute the coordinate curves, we follow the geodesics of the Riemannian metric, ${g}_{ij}({\bf x})$, in each of the $k-1$ principal directions.

\end{itemize}
Thus, overall, we are minimizing dissimilarity, and maximizing probability.  We will show how to do this, using concrete examples, in Sections \ref{Exp:Gauss} and \ref{Exp:CurvGauss} of this paper.
  
In the following section, we will see how to construct an integral manifold orthogonal to $\nabla U$, and how to define a dissimilarity metric, ${g}_{ij}({\bf x})$, with the desired properties.    Because of the prominent role played by the Riemannian dissimilarity metric in our theory, it is natural to describe it as a theory of \emph{differential similarity}.

\section{Integral Manifolds in ${\bf R}^{3}$.}
\label{Exp:IntManifolds}
 
From this point on, for purposes of exposition, we will restrict our investigations  from ${\bf R}^{n}$ to ${\bf R}^{3}$.   We will see later (in Section \ref{FutureWork}) that this is not a limitation on the scope of the theory, since our results can easily be generalized again to ${\bf R}^{n}$.  Instead, the restriction to three dimensions simplifies our calculations, and makes them easier to visualize, as we will see in Section \ref{Mathematica}.
 
Since we are now working in three-dimensional Euclidean space, we are primarily interested in two-dimensional integral manifolds.   Is there a 
two-dimensional integral manifold orthogonal to the drift vector, $\nabla U$?   Consider, first, a more general case.  Suppose ${\bf G} = (P({\bf x}),Q({\bf x}),R({\bf x}))$ represents the coordinates of a vector field that is defined but not equal to $(0,0,0)$ in some open region $\mathcal{D} \subseteq {\bf R}^{3}$.
\begin{theorem}
\label{2DIntManifold}
There exists a two-dimensional integral manifold in $\mathcal{D}$ with tangent plane everywhere orthogonal to ${\bf G}$ if and only if 
\[
{\bf G} \cdot (\nabla \times  {\bf G}) = 0
\]
\end{theorem}
\begin{proof}
See \cite{bishop1964geometry}, Problem 29, p. 23; \cite{cartan1971differential}, pp. 97--98; \cite{lovelock1975tensors}, pp. 155--156.
\end{proof}
\noindent
Intuitively, this theorem states that ${\bf G}$ must be orthogonal to its own ``curl,'' a condition that is satisfied if ${\bf G}$ is proportional to the gradient of a scalar potential.  Thus, any 
${\bf G}$ in the form $N({\bf x}) \nabla U({\bf x})$ would work.
 
We still need a method to compute this integral manifold, however, and to define a curvilinear coordinate system on it.  One approach is to choose basis vectors for a two-dimensional subspace of the tangent space at ${\bf x}$ in the following form:
\begin{align}
\label{VWdef}
{\bf V} \partial \;&=\;  f({\bf x}) \frac{\partial}{\partial x}  \; +\; \frac{\partial}{\partial y} \\
{\bf W} \partial \;&=\; g({\bf x}) \frac{\partial}{\partial x} \; +\; \frac{\partial}{\partial z}  \notag
\end{align}
Now compute: ${\bf V} \times {\bf W} = (f,1,0) \times (g,0,1) = (1,-f,-g)$.  If ${\bf V} \times {\bf W}$ is proportional to ${\bf G}$, then ${\bf G}$ is orthogonal to the plane containing both ${\bf V}$ and ${\bf W}$, and conversely.  So we can set:
\begin{align}
{\bf V} \times {\bf W} \; &= \; \frac{1}{P({\bf x})} {\bf G} \; = \; \frac{1}{P({\bf x})} (P({\bf x}),Q({\bf x}),R({\bf x})) \notag \\ 
\; &= \; (1,-f({\bf x}),-g({\bf x})) \notag
\end{align}
and obtain the results $f({\bf x}) = - Q({\bf x}) / P({\bf x})$ and $g({\bf x}) = - R({\bf x}) / P({\bf x})$. If $G = \nabla U({\bf x})$, then Theorem \ref{2DIntManifold} applies.  In this case, the vector fields given by $P({\bf x}) {\bf V} = (-Q({\bf x}),P({\bf x}),0)$ and $P({\bf x}) {\bf W} = (-R({\bf x}),0,P({\bf x}))$ provide what we want, namely, a basis for the tangent plane to the two-dimensional integral manifold that is everywhere orthogonal to the drift vector, $\nabla U$. 
  
This construction can also be justified directly by the Theorem of Frobenius.  Geometrically, we interpret  ${\bf V} \partial$ and ${\bf W} \partial$ as the basis vectors for a \emph{tangent subbundle}, $E$, in ${\bf R}^{3}$.  (Historically, a tangent subbundle was called a ``distribution," but this term does not have the right connotations today.)  We compute the \emph{Lie bracket} of ${\bf V} \partial$ and ${\bf W} \partial$ as follows:
\begin{align}
\label{VWLieBracket}
\left[  {\bf V} \partial , {\bf W} \partial  \right] \; &= \;  {\bf V} \partial \circ {\bf W} \partial - {\bf W} \partial \circ {\bf V} \partial  \\
  \;&=\;  \left[ 
 \frac{\partial g}{\partial y} - \frac{\partial f}{\partial z} +  f(x,y,z) \frac{\partial g}{\partial x} - g(x,y,z) \frac{\partial f}{\partial x}  \notag
  \right] \frac{\partial}{\partial x} 
\end{align}
Now the geometric version of the Theorem of Frobenius asserts that, if $\left[  {\bf X} \partial , {\bf Y} \partial  \right]$ ``belongs to'' $E$ whenever ${\bf X} \partial$ ``belongs to'' $E$ and ${\bf Y} \partial$ ``belongs to'' $E$, for arbitrary ${\bf X}$ and ${\bf Y}$, then $E$ can be extended to a full integral manifold in ${\bf R}^{3}$.  But if ${\bf V} \partial$ and ${\bf W} \partial$ as defined by the equations in \eqref{VWdef} form a basis for $E$, then $\left[  {\bf V} \partial , {\bf W} \partial  \right]$ ``belongs to'' $E$ if and only if the bracketed expression on the right-hand side of \eqref{VWLieBracket} is identically zero.  This leads to the following classical statement of the Frobenius integrability conditions as a system of partial differential equations:
\begin{align}
 \frac{\partial g}{\partial y} + f(x,y,z) \frac{\partial g}{\partial x}  \; &= \; \frac{\partial f}{\partial z} + g(x,y,z) \frac{\partial f}{\partial x}   \notag
\end{align}
As a further check on Theorem \ref{2DIntManifold}, we can verify by a direct computation that the preceding equation holds for $f({\bf x})$ and $g({\bf x})$, as defined previously, when $G = \nabla U({\bf x})$.
 
To simplify the notation and the subsequent calculations, let us absorb the factor $P({\bf x})$ into the definition of the two tangential vector fields, and write:
\begin{align}
 \nabla U({\bf x}) \; &= \;  (P({\bf x}),\;Q({\bf x}),\;R({\bf x})) \notag \\
{\bf V} ({\bf x}) \; &= \;  (-Q({\bf x}),\;P({\bf x}),\;0) \notag \\
{\bf W} ({\bf x}) \; &= \;  (-R({\bf x}),\;0,\;P({\bf x})) \notag
\end{align}
In this form, it is easy to see that $\nabla U({\bf x})$ is orthogonal to both ${\bf V} ({\bf x})$ and ${\bf W} ({\bf x})$.  Note also that ${\bf V} ({\bf x})$ and ${\bf W} ({\bf x})$ are not orthogonal to each other, although the vector fields ${\bf V} \partial = {\bf V} ({\bf x}) / P({\bf x})$ and ${\bf W} \partial = {\bf W} ({\bf x}) / P({\bf x})$ commute, as we have seen, when viewed as differential operators.   Now one way to use these tangential vector fields is to compute a \emph{global} ($\rho, \theta, \phi$) coordinate system.  For example, we can compute the integral curves of the vector field ${\bf V} ({\bf x})$ and use these for a coordinate called $\theta$, and we can compute the integral curves of the vector field ${\bf W} ({\bf x})$ and use these for a coordinate called $\phi$.  Note that the $\theta$ coordinate curves will all lie in the global $xy$ plane, and the $\phi$ coordinate curves will all lie in the global $xz$ plane, if we take this approach.
 
But another approach is to use these vector fields to construct a \emph{local} coordinate system.  Any linear combination of ${\bf V} ({\bf x})$ and ${\bf W} ({\bf x})$ could be taken as one of the basis vectors for the tangent subbundle, and we can vary this linear combination as we move around the integral manifold.  To implement this idea, it is useful to define a \emph{Riemannian metric} on the integral manifold.  The most natural way to do this is to define a metric tensor on all of ${\bf R}^{3}$, using the inner products of  $\nabla U({\bf x})$, ${\bf V} ({\bf x})$ and ${\bf W} ({\bf x})$, in that order.  We thus define:
\begin{align}
&
 \begin{pmatrix}
    \\
   {g}_{ij}({\bf x}) \\
    \\
\end{pmatrix}
 \; = \;  \notag \\
 \notag \\
 &
\left(\begin{array}{ccc} P^{2}({\bf x}) + Q^{2}({\bf x}) + R^{2}({\bf x}) & 0 & 0 \\0 & P^{2}({\bf x}) + Q^{2}({\bf x})  & Q({\bf x}) R({\bf x}) \\0 & Q({\bf x}) R({\bf x}) & P^{2}({\bf x}) + R^{2}({\bf x}) \end{array}\right)
\notag
\end{align}
To remain consistent with the coordinate notation introduced in Section \ref{proto}, we let $i$ and $j$ range over $0$, $1$, $2$, and we stipulate that $u^{0} = \rho$, $u^{1} = \theta$, $u^{2} = \phi$.  Since $P({\bf x})$, $Q({\bf x})$, $R({\bf x})$, are the components of the drift vector,  $\nabla U({\bf x})$, and since the diffusion equation in which $\nabla U({\bf x})$ appears has an invariant probability density that is determined by the exponential of the potential function, $U({\bf x})$, it should be clear that ${g}_{ij}({\bf x})$ has at least some of the properties that we have been looking for.   We thus adopt this formula, provisionally, as the definition of our \emph{dissimilarity metric}. 
 
The matrix $ \left( {g}_{ij}({\bf x}) \right) $ is not diagonal, in general, but it can easily be diagonalized.  The eigenvectors are:
\begin{align}
\xi_{0} =\left(\begin{array}{c} 1 \\ 0 \\ 0\end{array}\right), \;
\xi_{1} =\left(\begin{array}{c} 0 \\ Q({\bf x}) \\ R({\bf x})  \end{array}\right), \;
\xi_{2} = \left(\begin{array}{c} 0 \\ - R({\bf x}) \\ Q({\bf x}) \end{array}\right), \;
 \notag 
\end{align}
and the corresponding eigenvalues are: $ \lambda_{0}({\bf x}) =  \lambda_{1}({\bf x}) = P^{2}({\bf x}) + Q^{2}({\bf x}) + R^{2}({\bf x})$ and $ \lambda_{2}({\bf x}) = P^{2}({\bf x})$.  This analysis leads to a spectral decomposition of $ \left( {g}_{ij}({\bf x}) \right) $ as:
\vspace{1ex}
\begin{align}
& \lambda_{0}({\bf x}) \left(\begin{array}{ccc} 1 & 0 & 0 \\0 & 0 & 0 \\0 & 0 & 0 \end{array}\right) \;+\; \lambda_{1}({\bf x}) \, \kappa({\bf x}) \left(\begin{array}{ccc} 0 & 0 & 0 \\0 & Q^{2}({\bf x}) & Q({\bf x}) R({\bf x})  \\0 & Q({\bf x}) R({\bf x})  & R^{2}({\bf x}) \end{array}\right)
 \notag \\
 \notag \\
& +\; \lambda_{2}({\bf x}) \, \kappa({\bf x}) \left(\begin{array}{ccc} 0 & 0 & 0 \\0 & R^{2}({\bf x})  & - Q({\bf x}) R({\bf x})  \\0 & - Q({\bf x}) R({\bf x}) & Q^{2}({\bf x}) \end{array}\right)  \notag 
\end{align}
\vspace{1ex}

\noindent
where $ \kappa({\bf x}) = 1/( Q^{2}({\bf x}) + R^{2}({\bf x}) )$.  Obviously, the first term in this expression corresponds to the $\rho$ coordinate.  The second and third terms can be rearranged, as follows: 
\begin{align}
 \left(\begin{array}{ccc} 0 & 0 & 0 \\0 & Q^{2}({\bf x}) & Q({\bf x}) R({\bf x})  \\0 & Q({\bf x}) R({\bf x})  & R^{2}({\bf x}) \end{array}\right)
 \;+\;  \lambda_{2}({\bf x}) \left(\begin{array}{ccc} 0 & 0 & 0 \\0 & 1 & 0 \\0 & 0 & 1 \end{array}\right)  \notag 
\end{align}
which is just another way of decomposing $ \left( {g}_{ij}({\bf x}) \right) $.  It is important to keep in mind the fact that  $\xi_{1}$ and  $\xi_{2}$ are represented here by their coefficients with respect to the basis vectors ${\bf V} ({\bf x})$ and ${\bf W} ({\bf x})$, which is not an orthogonal coordinate system, in general.  But a simple calculation shows that
\begin{align}
( \; Q({\bf x})  {\bf V} ({\bf x}) +  R({\bf x})  {\bf W} ({\bf x}) \; )  \cdot  ( \; - R({\bf x})  {\bf V} ({\bf x}) +  Q({\bf x})  {\bf W} ({\bf x}) \; ) \; = \; 0
\notag 
\end{align}
Thus $\xi_{1}$ and $\xi_{2}$ are orthogonal to each other in ${\bf R}^{3}$, even if ${\bf V} ({\bf x})$ and ${\bf W} ({\bf x})$ are not. 
  
The main application of our dissimilarity metric, however, is to compute \emph{geodesics} on the surface of the integral manifold orthogonal to $\nabla U({\bf x})$.  Recall that any linear combination of ${\bf V} ({\bf x})$ and ${\bf W} ({\bf x})$ yields a vector in the tangent subbundle, $E$, and thus we can construct vector fields in $E$ in the form $v(t) {\bf V} ({\bf x}) + w(t) {\bf W} ({\bf x})$ for arbitrary functions $v(t)$ and $w(t)$.  For a geodesic, we are looking for a curve $\gamma(t)$ with values in ${\bf R}^{3}$ which minimizes the ``energy'' functional: 
\begin{equation}
\label{EnergyFunctional}
\frac{1}{2}  \int_{0}^{T} \left(\begin{array}{ccc}  v(t) & w(t) \end{array}\right)  
\left(\begin{array}{cc} {g}_{11}(\gamma(t)) & {g}_{12}(\gamma(t)) \\ {g}_{21}(\gamma(t)) & {g}_{22}(\gamma(t)) \end{array}\right)
\left(\begin{array}{c}  v(t) \\ w(t) \end{array}\right) dt
\end{equation}
subject to the constraint:
\begin{equation}
\label{GammaConstraint}
\gamma \, '(t)  \;=\;  v(t) {\bf V}( \gamma(t) ) +  w(t) {\bf W}( \gamma(t) ) 
\end{equation}
This variational problem leads to a system of Euler-Lagrange equations for the curves $\gamma(t) = (x(t),y(t),z(t))$ and $(v(t),w(t))$, plus three Lagrange multipliers.  For initial conditions, we specify $(x(0),y(0),z(0))$ and we use $\xi_{1}$ and  $\xi_{2}$, the eigenvectors of  $ {g}_{ij}(x(0),y(0),z(0)) $, to help us determine the initial values $(v(0),w(0))$.  This is a complicated system of equations, but it can be solved numerically in \emph{Mathematica}. 
 
The preceding analysis was based on a global coordinate system centered on the $x$ axis, since our initial vector fields were determined by the equations $P({\bf x}) {\bf V} = (-Q({\bf x}),P({\bf x}),0)$ and $P({\bf x}) {\bf W} = (-R({\bf x}),0,P({\bf x}))$.  But we could also work with a coordinate system centered on the $y$ axis, using the equations $Q({\bf x}) {\bf V} = (Q({\bf x}),-P({\bf x}),0)$ and $Q({\bf x}) {\bf W} = (0,-R({\bf x}),Q({\bf x}))$, or the $z$ axis, using $R({\bf x}) {\bf V} = (R({\bf x}),0,-P({\bf x}))$ and $R({\bf x}) {\bf W} = (0,R({\bf x}),-Q({\bf x}))$.   In fact, it is useful to be able to switch from one such coordinate system to another, as we move around the integral manifold.  Since ${g}_{00}({\bf x}) = P^{2}({\bf x}) + Q^{2}({\bf x}) + R^{2}({\bf x}) = \| \nabla U({\bf x}) \| ^{2} $, it is obvious that the $\rho$ coordinate is independent of the global coordinate system used to define it.  But the same is true of ${g}_{ij}({\bf x})$ when $ i \neq 0 $ and $ j \neq 0 $.  To see this, let  $u^{1}$ and $u^{2}$ denote the $\Theta$ coordinates centered on the $x$ axis, and 
let  $\bar{u}^{1}$ and $\bar{u}^{2}$ denote the $\Theta$ coordinates centered on the $y$ axis.  The Jacobian matrix of the coordinate transformation from $\bar{u}^{k}$ to $u^{i}$ can be computed as follows:
\begin{align}
& \left(\begin{array}{c} \;\\ \partial u^{i} / \partial \bar{u}^{k} \\ \; \end{array}\right) \notag \\
& = \left(\begin{array}{ccc}\partial x / \partial \rho & -Q({\bf x}) & -R({\bf x}) \\\partial y / \partial \rho & P({\bf x}) & 0 \\\partial z / \partial \rho & 0 & P({\bf x})\end{array}\right) ^{-1} \left(\begin{array}{ccc}\partial x / \partial \rho & Q({\bf x}) & 0  \\\partial y / \partial \rho & -P({\bf x}) & -R({\bf x})  \\\partial z / \partial \rho & 0 & Q({\bf x}) \end{array}\right) \notag \\[1ex]
& = \left(\begin{array}{ccc}1 & 0 & 0 \\0 & -1 & -R({\bf x}) / P({\bf x})  \\0  & 0 & Q({\bf x}) / P({\bf x})\end{array}\right) \notag
\end{align}
Now let ${g}_{ij}({\bf x})$ and ${\bar{g}}_{kl}({\bf x})$ denote the dissimilarity metric based on the $u^{i}$ and $\bar{u}^{k}$ coordinates, respectively.  Restricting our attention to the $2 \times 2$ matrix for the $\Theta$ coordinates, we compute:
\begin{align}
& \left(\begin{array}{cc} {\bar{g}}_{11}({\bf x})  & {\bar{g}}_{12}({\bf x}) \\ {\bar{g}}_{21}({\bf x}) & {\bar{g}}_{22}({\bf x}) \end{array}\right) \notag \\[1ex]
& = \left(\begin{array}{cc}  P^{2}({\bf x}) + Q^{2}({\bf x}) & P({\bf x}) R({\bf x}) \\ P({\bf x}) R({\bf x}) & Q^{2}({\bf x}) + R^{2}({\bf x})  \end{array}\right) \notag \\[1ex]
& = \left(\begin{array}{cc} -1 & -R({\bf x}) / P({\bf x}) \\ 0 & Q({\bf x}) / P({\bf x}) \end{array}\right) ^{T}
 \left(\begin{array}{cc} {g}_{11}({\bf x})  & {g}_{12}({\bf x}) \\ {g}_{21}({\bf x}) & {g}_{22}({\bf x}) \end{array}\right) 
\left(\begin{array}{cc} -1 & -R({\bf x}) / P({\bf x}) \\ 0 & Q({\bf x}) / P({\bf x}) \end{array}\right) \notag
\end{align}
But this is just an instantiation of the transformation law for a type (0,2) tensor:
\begin{align}
{\bar{g}}_{kl}({\bf x}) \;=\;  \sum_{i,j=1}^{2} \frac{\partial u^{i}}{\partial \bar{u}^{k}} \; {g}_{ij}({\bf x}) \; \frac{\partial u^{j}}{\partial \bar{u}^{l}}  \notag
\end{align} 
The same calculations obviously lead to the same results for all pairwise transformations among the three global coordinate systems.  Thus, on a two-dimensional integral manifold, for a fixed $\rho$, the dissimilarity metric, ${g}_{ij}({\bf x})$, is independent of the global coordinate system used to define it.

\section{Experiments with \emph{Mathematica}.}
 \label{Mathematica}
  
To sharpen our intuitions, and before developing the theory of differential similarity any further, let's look at some experiments in ${\bf R}^{3}$ using the computational and graphical facilities of \emph{Mathematica}.   Section \ref{Exp:Gauss} is a comprehensive study of the Gaussian case, which is the one example that can be solved analytically.  Section \ref{Exp:CurvGauss} then considers what we will refer to as the ``curvilinear Gaussian" case.  Here, we apply a quadratic potential function to the output of a cubic polynomial coordinate transformation, producing an example that cannot be solved analytically, but which still retains some degree of tractability.  Finally, in Section \ref{Exp:BiCurvGauss}, we put two ``curvilinear Gaussians'' together in a mixture distribution.   
 
The source code for these examples is available in three \emph{Mathematica} notebooks:
\vspace{1ex}

 {\tt Gaussian.nb} 
 
 {\tt CurvilinearGaussian.nb}
 
 {\tt BimodalCurvilinearGaussian.nb}

\subsection{The Gaussian Case.}
\label{Exp:Gauss}

Consider, first, the case of a quadratic potential, for which most results can be obtained analytically in closed form.  Define $U({\bf x})$ as follows:
\[
U(x,y,z) \;=\;  - \frac{1}{2}( a x^2 + b y^2 + c z^2)
\]
Then the gradient is: $\nabla U(x,y,z) = (- a x, - b y, - c z)$, and the derived potential $V({\bf x})$ is:
\[
V(x,y,z) \;=\;  \frac{1}{2}( a^2 x^2 + b^2 y^2 + c^2 z^2) - \frac{1}{2}(a + b + c)
\]
(We can ignore the constant term.)  It is well known that the Feynman-Kac formula, given by either  \eqref{defPt} or \eqref{defQt}, has a closed-form solution whenever $U({\bf x})$ and $V({\bf x})$ are quadratic polynomials. Furthermore, the invariant probability measure, $ e^{\,2\,U({\bf x})} $, given by Theorem  \ref{QtfProcess}, is obviously a Gaussian.  Computing the normalization factor and assuming that $a > 0$, $b > 0$, $c > 0$, the invariant probability density function is:
\begin{equation}
\label{gaussianpdf}
\sqrt{a b c} \, \pi^{- \frac{3}{2}} \,  \exp{\left[ - ( a x^2 + b y^2 + c z^2) \right]} 
\end{equation}
Note that the covariance matrix in \eqref{gaussianpdf} is already in diagonalized form.

\begin{figure}[htbp]
\begin{center}
\includegraphics[width=3.5in]{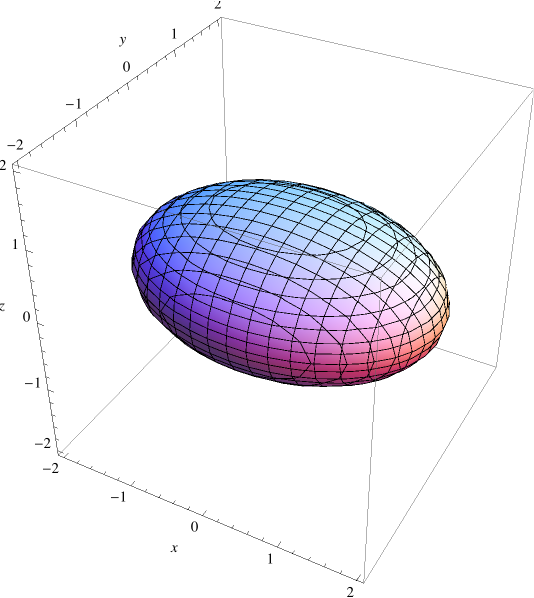}
\caption{Contour plot for the surface of a quadratic potential at $U(x,y,z) = -2$.}
\label{Gaussian}
\end{center}
\end{figure}
  
For a numerical example, set $a = 1, b = 2, c = 4$.  Figure \ref{Gaussian} then shows the surface defined by the equation $U(x,y,z) = -2$.
Figure \ref{StreamPlotGaussian} shows a {\tt StreamPlot} of the gradient vector field generated by $\nabla U(x,y,z)$ at $z=0$.  This picture makes sense, intuitively. Notice that the drift vector is ``transporting probability mass towards the origin,'' to counteract the dissipative effects of the diffusion term in the stochastic process.  If the system is in perfect balance, of course, we have an invariant probability measure, which in this case is a Gaussian.  
 
\begin{figure}[htbp]
\begin{center}
\includegraphics[width=3.5in]{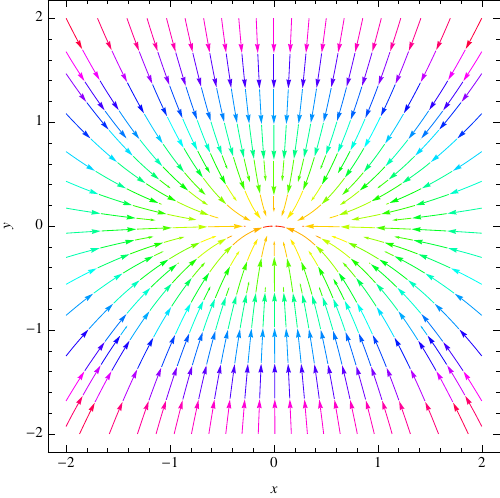}
\caption{The gradient vector field at $z = 0$ for the quadratic potential in Figure 1.}
\label{StreamPlotGaussian}
\end{center}
\end{figure}
 
The Gaussian case is simple enough that we can solve the differential equations explicitly in  \emph{Mathematica}, using {\tt DSolve}.  
First, the integral curve of the vector field $\nabla U({\bf x}) =  (P({\bf x}),Q({\bf x}),R({\bf x}))$ starting at ${\bf x}_{0} = (x_{0},y_{0},z_{0})$ is given by:
\begin{equation}
\label{intcurve.rho}
\hat{\rho}(t) \;=\;
\begin{pmatrix}
& x_{0} \; e^{-a t} & \\
& y_{0} \; e^{-b t} \\
& z_{0} \; e^{-c t}
\end{pmatrix}  \notag
\end{equation}
For the tangential vector fields, we will start with a global coordinate system centered on the $x$ axis, so that ${\bf V}({\bf x}) = (-Q({\bf x}),P({\bf x}),0) = ( b y, - a x, 0 ) $ and ${\bf W}({\bf x})  = (-R({\bf x}),0,P({\bf x})) = ( c z, 0, - a x)$.   Then the integral curve of the vector field ${\bf V}({\bf x})$ starting at ${\bf x}_{1} = (x_{1},y_{1},z_{1})$ is given by:
\begin{equation}
\label{intcurve.theta}
\hat{\theta}(t) \;=\;
\begin{pmatrix}
& x_{1} \; \cos{\sqrt {a b} \; t } \;+\; y_{1} \; \sqrt {b/a} \; \sin{\sqrt {a b} \; t} & \\
& y_{1} \; \cos{\sqrt {a b} \; t } \;-\; x_{1} \; \sqrt {a/b} \; \sin{\sqrt {a b} \; t} \\
& z_{1} 
\end{pmatrix} \notag
\end{equation}
and the integral curve of ${\bf W}({\bf x})$  starting at ${\bf x}_{2} = (x_{2},y_{2},z_{2})$ is given by:
\begin{equation}
\label{intcurve.phi}
\hat{\phi}(t) \;=\;
\begin{pmatrix}
& x_{2} \; \cos{\sqrt {a c} \; t } \;+\; z_{2} \; \sqrt {c/a} \; \sin{\sqrt {a c} \; t} & \\
& y_{2}\\
& z_{2} \; \cos{\sqrt {a c} \; t } \;-\; x_{2} \; \sqrt {a/c} \; \sin{\sqrt {a c} \; t} 
\end{pmatrix}  \notag
\end{equation}
Figure \ref{GaussianManifold} shows the global coordinate system on a two-dimensional integral manifold that would be generated by these curves.   Note that the $\theta$ coordinate curves lie in the $xy$ plane, and the $\phi$ coordinate curves lie in the $xz$ plane, as expected.
  
\begin{figure}[htbp]
\begin{center}
\includegraphics[width=3.5in]{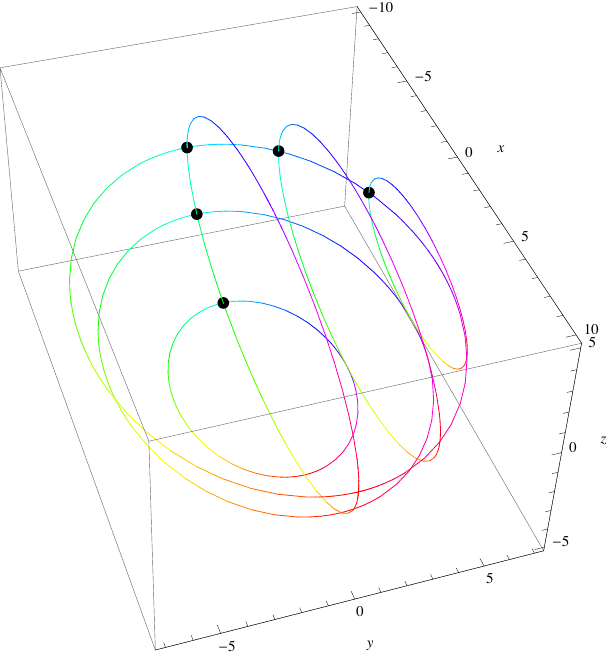}
\caption{An integral manifold with a global coordinate system for the quadratic potential in Figure 1.}
\label{GaussianManifold}
\end{center}
\end{figure}
 
Given the curves $\hat{\rho}(t)$, $\hat{\theta}(t)$ and $\hat{\phi}(t)$, what does it mean to say that a point in ${\bf R}^{3}$ has the coordinates $(\rho, \theta, \phi)$?  We adopt the following conventions:  Starting with the $x$ axis as the \emph{principal axis}, choose a ``maximal'' point ${\bf x}_{0} = (x_{0},0,0)$ and follow the curve $\hat{\rho}(t)$ towards the origin.  There are two natural measures of distance along this curve: the Euclidean arc length, which in this case is just the value of the $x$-coordinate, and the Riemannian arc length, which is determined by our dissimilarity metric, ${g}_{ij}({\bf x})$.  Our choice here is to use the Euclidean arc length to specify the $\rho$ coordinate.  (We will subsequently see another role for the Riemannian arc length.)  In the Gaussian case, therefore, $\rho$ has the value $x_{0} \; e^{-a t}$, which ranges over the interval $( 0 , x_{0}]$ as $t$ ranges from $\infty$ to $0$.  But by choosing a value for $\rho$, we are also choosing the integral manifold on which  $\hat{\theta}(t)$ and $\hat{\phi}(t)$ are defined.  Therefore, to interpret the coordinates  $\theta$ and $\phi$, starting at $(\rho,0,0)$, we traverse the distance $\theta$ along the $\hat{\theta}(t)$ curve from the slice $\theta = 0$, and we traverse the distance $\phi$ along the $\hat{\phi}(t)$ curve from the slice $\phi = 0$, until we arrive at the point $(\rho, \theta, \phi)$.  Note, too, that we can traverse the $\hat{\theta}(t)$ and $\hat{\phi}(t)$ curves in either order, as long as we remain within a neighborhood of $(\rho,0,0)$ in which these curves intersect.  The black dots in Figure \ref{GaussianManifold} may be helpful in visualizing this procedure.
 
Once again, the Gaussian case is simple enough that we can analyze the coordinate transformation from $(\rho, \theta, \phi)$ to $(x,y,z)$,  and derive an explicit expression for its Jacobian matrix.  First, let $\vec{\theta}_{s}({\bf x}) = \hat{\theta}_{{\bf x}}(s)$ denote the \emph{flow} of the vector field ${\bf V}({\bf x})$ starting at ${\bf x}$, and similarly let $\vec{\phi}_{t}({\bf x}) = \hat{\phi}_{{\bf x}}(t)$ denote the \emph{flow} of the vector field ${\bf W}({\bf x})$ starting at ${\bf x}$.  Applying the composition, $ \vec{\theta}_{s} \circ \vec{\phi}_{t} $,  of the flows  $ \vec{\theta}_{s}$ and $\vec{\phi}_{t} $ to the point ${\bf x} = (\rho,0,0)$, we obtain the following equations, for arbitrary $s$ and $t$:
\begin{align} 
\label{coordexpr}
x \;=\; \vec{x}(\rho, s, t) \;=\;  & {\rho} \; \cos{\sqrt {a b} \,s} \;  \cos{\sqrt {a c} \,t } \\
y \;=\; \vec{y}(\rho, s, t) \;=\; & - {\rho} \; \sqrt {a/b} \; \sin{\sqrt {a b} \,s}\;  \cos{\sqrt {a c} \,t }  \notag \\
z \;=\; \vec{z}(\rho, s, t) \;=\; & - {\rho} \; \sqrt {a/c}\; \sin{\sqrt {a c} \,t} \notag
\end{align}
By a simple calculation:
\begin{equation}
\frac{\partial}{\partial s} \left(\begin{array}{c} \vec{x}(\rho, s,t) \\  \vec{y}(\rho, s,t)\\ \vec{z}(\rho, s,t) \end{array}\right)
\; = \; \left(\begin{array}{c} b \;  \vec{y}(\rho, s, t)  \\ - a \; \vec{x}(\rho, s, t) \\ 0 \end{array}\right) \notag
\end{equation}
In other words, $\partial / \partial s = {\bf V}({\bf x}) = P({\bf x}) {\bf V} \partial $.  By another simple calculation, setting $ s = 0 $ in \eqref{coordexpr}, we have:
\begin{equation}
\frac{\partial}{\partial t} \left(\begin{array}{c} \vec{x}(\rho, 0,t) \\ \vec{y}(\rho, 0, t) \\ \vec{z}(\rho, 0, t) \end{array}\right)
=  \left(\begin{array}{c} - {\rho} \; \sqrt {a c}\; \sin{\sqrt {a c} \, t} \\0 \\  - a \; {\rho} \; \cos{\sqrt {a c} \, t} \end{array}\right) 
=  \left(\begin{array}{c}  c \;  \vec{z}(\rho, 0, t)  \\0 \\ - a \; \vec{x}(\rho, 0, t) \end{array}\right) \notag
\end{equation}
In other words, $\partial / \partial t = {\bf W}({\bf x}) = P({\bf x}) {\bf W} \partial $ when $ s = 0 $.   We obtain a similar result if we reverse the composition of the flows $ \vec{\theta}_{s}$ and $\vec{\phi}_{t} $, and apply  $  \vec{\phi}_{t} \circ \vec{\theta}_{s} $ to the point ${\bf x} = (\rho,0,0)$.  In this case, we can compute $\partial / \partial t = {\bf W}({\bf x}) = P({\bf x}) {\bf W} \partial $ for all $s$ and $t$, and $\partial / \partial s = {\bf V}({\bf x}) = P({\bf x}) {\bf V} \partial $ for $t = 0$.  
 
Now consider the coordinate transformation itself.  Applying the composition $ \vec{\theta}_{s} \circ \vec{\phi}_{t} $ to the point ${\bf x} =(\rho,0,0)$, we follow the $ \vec{\phi}_{t} $ curve with $ s = 0 $ until we reach the point at which $ t = \phi $, then follow the $ \vec{\theta}_{s} $ curve until we reach the point at which $ s = \theta $.   Or, applying the composition $  \vec{\phi}_{t} \circ \vec{\theta}_{s} $ to the point ${\bf x} =(\rho,0,0)$, we follow the $ \vec{\theta}_{s} $ curve with $ t = 0 $ until we reach the point at which $ s = \theta $, then follow the $ \vec{\phi}_{t} $ curve until we reach the point at which $ t = \phi $.  In either case, we can see from the equations above that the Jacobian matrix of  $(x,y,z) = (x(\rho, \theta, \phi),y(\rho, \theta, \phi),z(\rho, \theta, \phi)) $ can be written explicitly as:
\begin{equation}
\label{Jgauss}
\mathbf{ J }(\rho, \theta, \phi) \; = \; 
\left(\begin{array}{ccc} x(\rho, \theta,\phi) / \rho & b \;  y(\rho, \theta, \phi)  & c \;  z(\rho, \theta, \phi)  
\\ y(\rho, \theta,\phi) / \rho & - a \; x(\rho, \theta, \phi) & 0 
\\ z(\rho, \theta,\phi) / \rho  & 0 & - a \; x(\rho, \theta, \phi) \end{array}\right) 
\end{equation}
This is all we need to carry out the calculations described in Section \ref{proto}, including the calculation of the coefficients $ \alpha^{ij} (\rho, \theta, \phi)$ and $ \beta^{i}(\rho, \theta, \phi)$ in Equation \eqref{DiffusionCoeffs}. We will analyze these results further in Section \ref{DiffuCoeffs&DissimMetrics}. 

\begin{figure}[htbp]
\begin{center}
\includegraphics[width=3.5in]{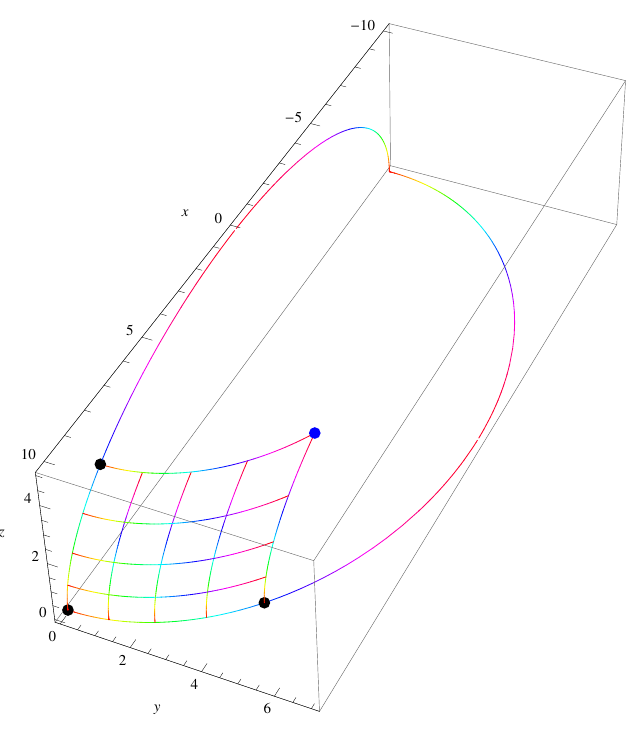}
\caption{A coordinate system for the quadratic potential  in Figure 1, based on commutative flows.}
\label{CommutativeCoords}
\end{center}
\end{figure}

However, as defined above, the flows $ \vec{\theta}_{s} $ and $ \vec{\phi}_{t} $ do not commute, i.e., $ \vec{\theta}_{s} \circ \vec{\phi}_{t} \neq \vec{\phi}_{t} \circ \vec{\theta}_{s} $.  If we wanted to work with commutative flows, we could divide out the \emph{scale factor}, $ P({\bf x}) $, and use $  {\bf V} \partial = {\bf V}({\bf x}) / P({\bf x}) $ and $ {\bf W} \partial =  {\bf W}({\bf x}) / P({\bf x}) $ as the basis vectors of our tangent subbundle, $E$.    In this case, $  {\bf V} \partial \circ  {\bf W} \partial = {\bf W} \partial \circ  {\bf V} \partial $, as we have seen, and it follows that $ \vec{\theta}_{s} \circ \vec{\phi}_{t} = \vec{\phi}_{t} \circ \vec{\theta}_{s} $.  See, e.g.,  \cite{spivakcomprehensive}, Lemma 5.13; \cite{bishop1968tensor}, Theorem 3.7.1;  \cite{bishop1964geometry}, Theorem 1.5. A coordinate system for the Gaussian case based on commutative flows is illustrated in Figure \ref{CommutativeCoords} , where the coordinates for the blue dot are $\theta = \pi \sqrt{2}$ and $\phi = \pi $, computed in either order.  We can even write out a closed-form solution for the composition of the flows in this case:
\begin{align} 
\vec{\theta}_{s} \circ \vec{\phi}_{t} \; (x,y,z) & \;=\;   \left( \sqrt{\frac{a x^{2} - b s (2 y + s) - c t (2 z + t)}{a}}, y + s, z + t \right)  \notag \\
& \;=\;  \vec{\phi}_{t} \circ \vec{\theta}_{s} \; (x,y,z)  \notag
\end{align}
Unfortunately, there are serious disadvantages in using ${\bf V} \partial$ and ${\bf W} \partial$ as basis vectors in this way, especially when we try to extend these results to the Riemannian dissimilarity metric and to the solution of the Euler-Lagrange equations for a geodesic.  The cost of computing the commutative flows is high, and the coordinate patch that they cover tends to be very small.  The better approach is to use ${\bf V}({\bf x})$ and ${\bf W}({\bf x})$ as the basis vectors, and to compute the coordinate maps in a fixed order, as we did in the previous paragraph.  Since our ultimate goal is to find the ``best'' lower-dimensional coordinate system, it is natural to be computing coordinates in the ``best'' possible order.  We will see how this works in the numerical calculations that follow.

\begin{figure}[htbp]
\begin{center}
\includegraphics[width=4.0in]{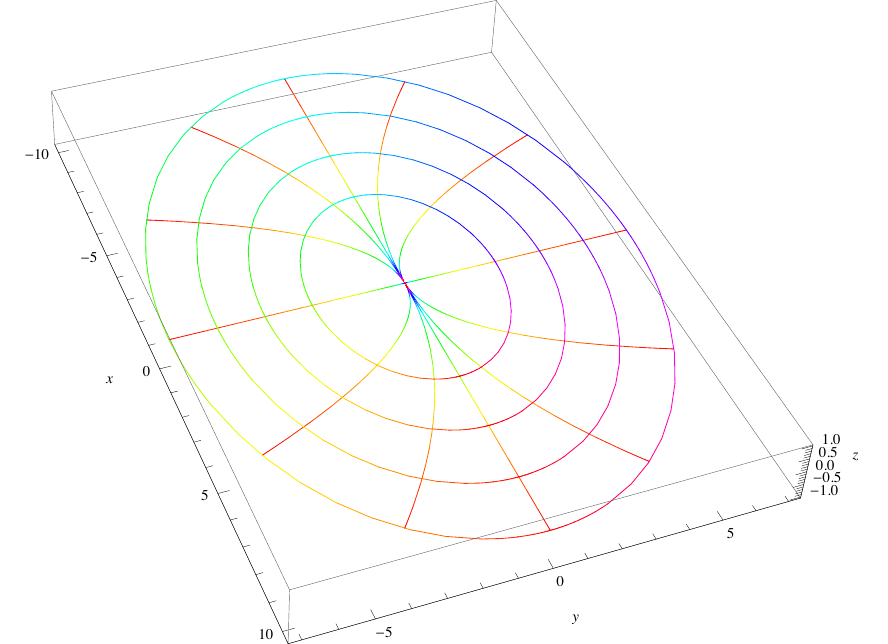}
\caption{A coordinate system for the $\rho,\theta$ surface of the quadratic potential in Figure 1.}
\label{Gaussian.frb.rhotheta}
\end{center}
\end{figure}

We have referred to the $x$ axis in Figure \ref{GaussianManifold} as the ``principal axis'' because of its correspondence to the results of Principal Component Analysis (PCA) in traditional linear statistics \cite{Pearson1901}.  For the Gaussian probability density given by \eqref{gaussianpdf}, with $a = 1, b = 2, c = 4$, the first component identified by PCA would be the $x$ axis, and the second component would be the $y$ axis.  Thus the ``principal surface'' would be defined by the $xy$ plane, which corresponds to the $( \rho, \theta )$ surface in our curvilinear coordinate system.  Figure \ref{Gaussian.frb.rhotheta} depicts this surface, with the $\rho$ and $\theta$ coordinates illustrated.  The maximal point on the principal axis is $(10,0,0)$, and the $\theta$ coordinate curves have been evenly spaced along the $\rho$ coordinate curve from $(10,0,0)$ to $(0,0,0)$.  Similarly, the $\rho$ coordinate curves have been evenly spaced along the maximal $\theta$ coordinate curve, which passes through the point $(10,0,0)$.  

However, although the $( \rho, \theta )$ surface in Figure \ref{Gaussian.frb.rhotheta} coincides with the $xy$ plane in this case, the complete PCA solution will not coincide, in general, with the solution that we are looking for in a curvilinear coordinate system. Principal Components Analysis projects data onto a linear subspace, and it seeks to maximize the \emph{variance} of the projected points, or to minimize the \emph{reconstruction error} resulting from the projection. These two objectives are equivalent in a linear system. In a curvilinear coordinate system, however, there are several possible definitions of the ``variance'' \cite{Pennec2006} and there are several ways to define the ``projection'' and the ``reconstruction error.'' We will examine these choices, below, as we continue our analysis of the simple Gaussian case. A related concept in linear statistics is Mahalanobis distance \cite{Mahalanobis1936}, which scales Euclidean distance in the sample space by the inverse of the covariance matrix. In fact, the first principal axis in the PCA solution (i.e., the direction that maximizes the variance) is also the direction that minimizes the Mahalanobis distance.  We will see that a similar principle applies in our curvilinear coordinate system, in which we seek to minimize the Riemannian dissimilarity metric.

There is another comparison (and another contrast) with Principal Components Analysis in our use of eigenvalues and eigenvectors.  The PCA solution is usually computed by diagonalizing the covariance matrix, and choosing as the principal components the eigenvectors associated with the maximal eigenvalues.  As we have seen in Section \ref{Exp:IntManifolds}, it is straightforward to diagonalize the Riemannian dissimilarity matrix $ \left( {g}_{ij}({\bf x}) \right) $.  This will give us the maximal and minimal  \emph{infinitesimal} directions for the integrand of the energy functional in \eqref{EnergyFunctional}.  However, the infinitesimal eigenvectors computed in this way are not quite what we want for the solution of the Euler-Lagrange equations, for two reasons.  First, minimizing the initial directions in the Euler-Lagrange equations cannot guarantee that we are also minimizing the geodesic curves over a finite distance, and it is this latter condition that we are primarily interested in. Second, it turns out that the eigenvectors of $ \left( {g}_{ij}({\bf x}) \right) $ are not tensor invariants, but depend on the coordinate system in which they are computed.  Nevertheless, diagonalizing the matrix $ \left( {g}_{ij}({\bf x}) \right) $ is a good start:  We can \emph{rotate} this solution to maximize or minimize the geodesic curves over a finite distance, and the solution to this global optimization problem is then guaranteed to be a tensor invariant.  

\begin{figure}[htbp]
\begin{center}
\includegraphics[width=5.0in]{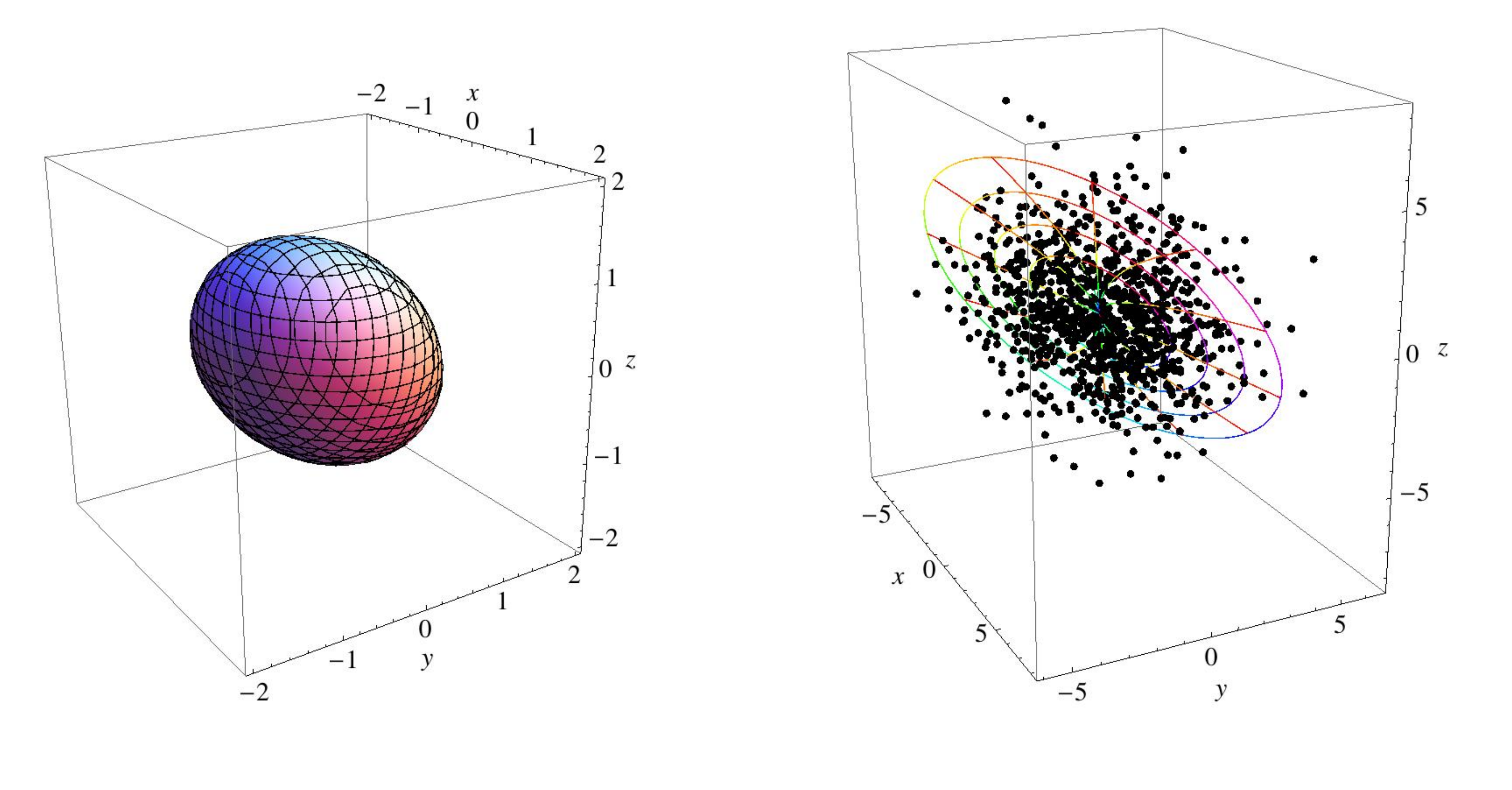}
\caption{A rotated quadratic potential: (a) the contour plot at $U(x,y,z) = -2$; (b) the rotated ($\rho,\theta$) surface from Figure \ref{Gaussian.frb.rhotheta}, superimposed on a scatter plot of sample data.}
\label{GaussianRotated60}
\end{center}
\end{figure}
 
To study these issues in more detail, let us now consider a variant of the simple Gaussian case. 
Figure \ref{GaussianRotated60}(a) shows an example in which the potential function depicted in Figure \ref{Gaussian} has been rotated through the angle $ \pi / 3 $ around the line from $(0,0,0)$ to $(1,1,1)$.  Under this rotation, the maximal point on the principal axis, $(10,0,0)$, would be displaced to the position $( 20/3, -10/3, 20/3 )$.  If our basis vectors, ${\bf V} ({\bf x})$ and ${\bf W} ({\bf x})$, were also rotated in the same way, we could still compute closed-form solutions to the differential equations, using {\tt DSolve}, and this procedure would still give us explicit expressions for the functions $\hat{\rho}(t)$, $\hat{\theta}(t)$ and $\hat{\phi}(t)$, although these expressions would be more complex than they were previously.  Continuing in this way, as before, we would eventually produce the $( \rho, \theta )$ surface shown in Figure \ref{Gaussian.frb.rhotheta}, but rotated through the angle $ \pi / 3 $ around the line from $(0,0,0)$ to $(1,1,1)$. This surface is depicted in Figure  \ref{GaussianRotated60}(b).  However, rather than repeating the same analytical calculations in a rotated coordinate system, which is not very interesting, what would happen if we treated the quadratic potential in Figure \ref{GaussianRotated60}(a) on its own terms, in the original $xyz$ coordinates?  If we did not know the rotation, \emph{a priori}, could we still compute an ``optimal'' curvilinear coordinate system, using just the Riemannian dissimilarity metric and the Euler-Lagrange equations?
 
Figure  \ref{GaussianRotated60}(b) also includes a scatter plot of sample data, 1000 points in all.  These data points have been generated according to the probability density function in \eqref{gaussianpdf}, scaled up by a factor of 20 and rotated to match the $( \rho, \theta )$ surface.  Thus the variance of the sample data along the $x$ axis is 10.0, which means that the point $(10,0,0)$ is located slightly more than 3 standard deviations from the origin. We are interested in seeing how these points are mapped in our ``optimal'' curvilinear coordinate system.

Since we are not just computing integral curves now, but are trying to minimize the Riemannian dissimilarity metric and solve the Euler-Lagrange equations, we cannot expect to find closed-form solutions in \emph{Mathematica}, using symbolic methods such as {\tt DSolve}.  Instead, we will rely on numerical methods, such as  {\tt NDSolve}.  Our plan is to follow the three steps outlined at the end of Section \ref{proto}: (1) Find a principal axis for the $\rho$ coordinate; (2) Determine the principal directions for the $\Theta$ coordinates; (3) Compute the geodesic coordinate curves for each of the principal $\Theta$ directions.  But we must now iterate these three steps multiple times, to convert a local solution (based on infinitesimal eigenvectors) into a global solution (based on geodesic curves over finite distances).
 
We need to address a preliminary issue:  When we were working with {\tt DSolve} in the simple Gaussian case, we were able to compute an explicit expression for $\hat{\rho}(t)$ and convert it into a formula for the $\rho$ coordinate measured in Euclidean arc length.  Basically, we were constructing a new parametrization of $\hat{\rho}(t)$.  This is not easy to do in the general case, however, because it would require us to invert the general formula for arc length.  Fortunately, there is a simpler approach, which works very well using {\tt NDSolve}.  In place of the differential equation derived from \eqref{odeYt}, we use the normalized version:
\begin{align}
{\gamma}'(t)  \;&=\; \frac{\nabla U( {\gamma}(t) )}{\| \nabla U( {\gamma}(t) ) \|}  \notag \\ 
{\gamma}(0)  \;&=\; {\bf x}_{0} \notag
\end{align}
Since our tangent vector now has length $1$, the integral curve that solves this equation will be parametrized by Euclidean arc length, but otherwise it will be identical to $\hat{\rho}(t)$.  The formula for Riemannian arc length, using our dissimilarity metric, ${g}_{ij}({\bf x})$, is also very simple when ${\gamma}(t)$ is defined in this way:
\begin{equation}
\label{RiemannianDistance}
\int_{0}^{T} 
\sqrt{
\left(\begin{array}{ccc}1 & 0 & 0\end{array}\right)
 \begin{pmatrix}
    \\
   \; {g}_{ij}({\gamma}(t)) \; \\
    \\
\end{pmatrix}
\left(\begin{array}{c}1 \\0 \\0\end{array}\right)
} \; dt 
 \;=\; 
\int_{0}^{T} 
\sqrt{
{g}_{00}({\gamma}(t))
} \; dt 
\notag
\end{equation}
This solves the parametrization problem for the $\rho$ coordinate.

We need to solve a similar problem for the $\Theta$ coordinates.  We can rely on two mathematical facts: First, the parametrization of the geodesic of an energy functional is proportional to its Euclidean arc length.  See, e.g.,  \cite{spivakcomprehensive}, Theorem 9.12.  Thus, computing and applying the proportionality factor, we can set the parametrization of a geodesic coordinate curve to be identical to its Euclidean arc length.  Second, the Euclidean distance along a curve on the Frobenius integral manifold is equal to the Riemannian distance along that curve, since the manifold is embedded in Euclidean ${\bf R}^{3}$.  Thus, we can construct a coordinate system in which the distance along the coordinate axes is a measure of the Riemannian dissimilarity along those axes.  For the coordinate curves that are transverse to the coordinate axes, we define the following flows:
\begin{align}
\label{TransverseCurves}
\vec{\theta}_{t}({\bf x}) =  \hat{\theta}_{{\bf x}}(t)  &= {\bf x} + \int_{0}^{t}  v_{\theta}(s) {\bf V}(  \hat{\theta}_{{\bf x}}(s) ) +  w_{\theta}(s) {\bf W}(  \hat{\theta}_{{\bf x}}(s) )  \; ds \\
\vec{\phi}_{t}({\bf x}) =  \hat{\phi}_{{\bf x}}(t)  &= {\bf x} + \int_{0}^{t}  v_{\phi}(s) {\bf V}(  \hat{\phi}_{{\bf x}}(s) ) +  w_{\phi}(s) {\bf W}(  \hat{\phi}_{{\bf x}}(s) )  \; ds \notag
\end{align}
These flows use the same $v(t)$ and $w(t)$ functions that were computed for the geodesics, but with a different starting point, ${\bf x}$.  (Compare \eqref{intcurveYt} and \eqref{odeYt} with \eqref{GammaConstraint}.) The parametrization of the curves given by \eqref{TransverseCurves} will be the same as the parametrization of the geodesic curves, and both curves will be identical on the coordinate axes, but the parametrizations elsewhere will not correspond to Euclidean arc length. Note also that the flows in \eqref{TransverseCurves} and their integral curves are not tensor invariants, in general, although they are invariant (by definition) whenever they coincide with the geodesic coordinate curves.
 
We are now ready to proceed through the three steps at the end of Section \ref{proto}. We will start off with the basis vectors ${\bf V} ({\bf x})$ and ${\bf W} ({\bf x})$ centered on the $x$ axis, and we will do the calculations initially using the infinitesimal eigenvectors of the matrix $ \left( {g}_{ij}({\bf x}) \right) $.  To fix our notation, let's use $\theta^{1}$ to denote the coordinate axis determined by the maximal eigenvalue $\lambda_{1}({\bf x})$ and its eigenvector $\xi_{1}({\bf x})$, and let's use $\theta^{2}$ to denote the coordinate axis determined by the minimal eigenvalue $\lambda_{2}({\bf x})$ and its eigenvector $\xi_{2}({\bf x})$. Here are the three steps: 
 \begin{enumerate}
\vspace{1ex} 

\item Find a principal axis for the $\rho$ coordinate.
\vspace{1ex} 
 
The basic idea is to find a point $(x_{0},y_{0},z_{0})$ at a fixed Euclidean distance from the origin, and an integral curve ${\gamma}(t)$ which solves the normalized differential equation for $ \nabla U({\bf x})$ starting at ${\bf x} = (x_{0},y_{0},z_{0})$, and for which the Riemannian distance, ${g}_{ij}({\bf x})$, measured along ${\gamma}(t)$ for a fixed interval, $t$, is minimal.  In short, we are looking for the \emph{least} Riemannian distance for a fixed Euclidean distance.  
 
We use {\tt NDSolve} to compute ${\gamma}(t)$, and we use {\tt NIntegrate} to compute the 
Riemannian distance along ${\gamma}(t)$.   {\tt FindMinimum} then searches for the minimal point $(x_{0},y_{0},z_{0})$ satisfying these constraints.  In our rotated Gaussian example, we can start the search at $(10,0,0)$ with the constraint that $(x_{0},y_{0},z_{0})$ must lie on the sphere $x^{2} + y^{2} + z^{2} = 100$, and {\tt FindMinimum} will return the value $(x_{0},y_{0},z_{0}) = ( 6.66666, -3.33335, 6.66667)$.  This is a reasonably good match with the analytical value, $(x_{0},y_{0},z_{0}) = ( 20/3, -10/3, 20/3 )$. 

An alternative computation is to minimize  $\| \nabla U(x,y,z) \| ^{2}$ on the sphere $x^{2} + y^{2} + z^{2} = 100$, which yields the value $(x_{0},y_{0},z_{0}) = ( 6.66667, -3.33333, 6.66667)$, an even closer match. These two solutions will be approximately the same, as they are here, as long as $\| \nabla U({\bf x}) \|$ is monotonic.

\vspace{1ex} 
\item Determine the principal directions for the $\Theta$ coordinates.
\vspace{1ex} 

We want to compute the eigenvalues, $ \lambda_{1}({\bf x}) $ and $ \lambda_{2}({\bf x}) $, and the associated eigenvectors, $\xi_{1}({\bf x})$ and $\xi_{2}({\bf x})$, for the dissimilarity matrix, $ \left(\begin{array}{c} {g}_{ij}({\bf x}) \end{array}\right) $, at the point ${\bf x} = (x_{0},y_{0},z_{0})$.  For expository purposes, let's initially use the analytical value $(x_{0},y_{0},z_{0}) = ( 20/3, -10/3, 20/3 )$. Then the eigenvalues are $100$ and $400/9$ and the eigenvectors are $(0, 10/3, -20/3)$ and $(0, 20/3, 10/3)$, respectively.  If we use the numerical value $(x_{0},y_{0},z_{0}) = ( 6.66666, \\-3.33335, 6.66667)$ and normalize the eigenvectors, we have  $\xi_{1} = ( 0, 0.44722, -0.894424 )$ and $\xi_{2} = ( 0, 0.894424, 0.44722 )$.  We can then confirm that 
\begin{align}
\xi_{1}^{T}  \left(\begin{array}{c} {g}_{ij}(x_{0},y_{0},z_{0}) \end{array}\right) \xi_{1}  &=   100.0 \notag \\
\xi_{2}^{T}  \left(\begin{array}{c} {g}_{ij}(x_{0},y_{0},z_{0}) \end{array}\right) \xi_{2}  &=  44.4443 \approx 400/9  \notag
\end{align}

\begin{figure}[htbp]
\begin{center}
\includegraphics[width=5.0in]{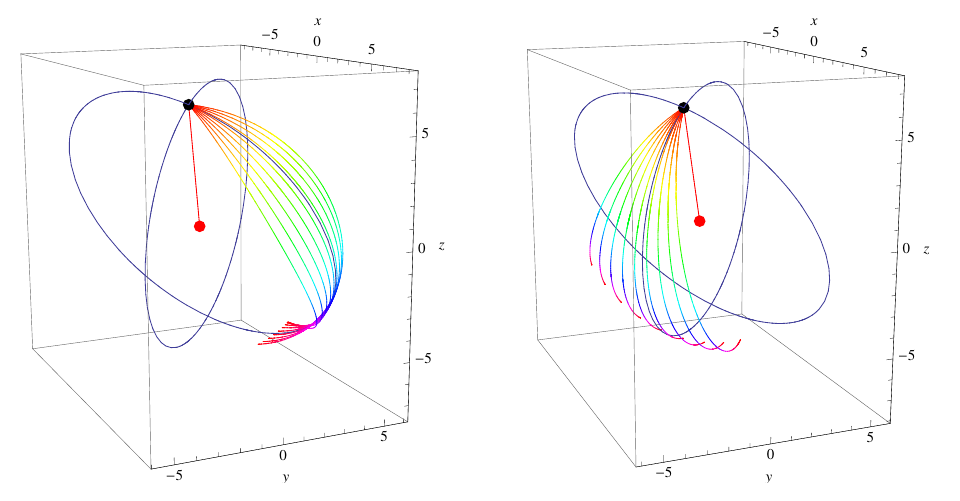}
\caption{Rotating the infinitesimal eigenvectors to maximize and minimize the global geodesic curves: (a) for the maximal eigenvalue $\lambda_{1}({\bf x})$;  (b)  for the minimal eigenvalue $\lambda_{2}({\bf x})$.}
\label{AlphaMaxMin}
\end{center}
\end{figure}

\vspace{1ex} 
\item Compute the geodesic coordinate curves for each of the principal $\Theta$ directions.
\vspace{1ex} 
 
In the final step, we compute the geodesic curves that solve the variational problem given by \eqref{EnergyFunctional} and \eqref{GammaConstraint}, with the initial value $(x(0),y(0),z(0)) = ( 6.66666, -3.33335, 6.66667)$ and with $(v(0),w(0))$ equal to $ \pm \, ( 0.44722, -0.894424  )$ for the $\theta^{1}$ coordinate, and $ \pm \, ( 0.894424, 0.44722 )$ for the $\theta^{2}$ coordinate.   \emph{Mathematica} has a {\tt VariationalMethods} package which computes the Euler-Lagrange equations symbolically from the specification of a variational problem.  We use this package, and then solve the resulting equations numerically with {\tt NDSolve}.
 
Let's examine some of the properties of these curves.   First, consider the distance measured along the $\rho$ coordinate curve from a point on either of the geodesic curves to the origin:  The Riemannian distance is constant, 50.0, but the Euclidean distance varies from a maximum of 10.0 at the point $(x_{0},y_{0},z_{0})$ to a minimum of  6.95688 along the $\theta^{1}$ curve and a minimum of 5.88452 along the $\theta^{2}$ curve.  Second, consider the distance along each of the geodesic curves from $(x_{0},y_{0},z_{0})$ to a point at an angle of $\pi/2$ from the origin.  For the $\theta^{1}$ curve, the Riemannian and Euclidean distance is 13.3259. For the $\theta^{2}$ curve, the Riemannian and Euclidean distance is 12.4379.  These are, of course, the properties of the shortest paths on the surface of an ellipsoid at a constant Riemannian distance from the origin.   
\end{enumerate}

\begin{figure}[htbp]
\begin{center}
\includegraphics[width=3.5in]{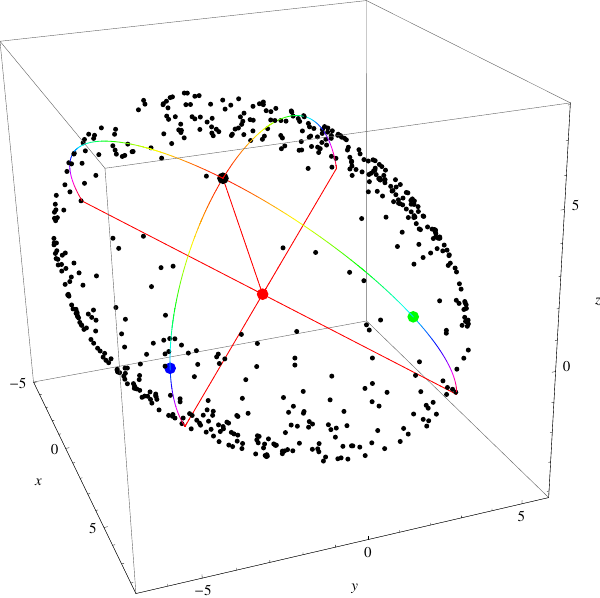}
\caption{Projecting the data points along the $\rho$ coordinate curves to the Frobenius integral manifold.}
\label{IntManScatter}
\end{center}
\end{figure}

However, as we will see, these are not yet the ``optimal'' coordinate curves that we are looking for.  The situation is illustrated in Figure \ref{AlphaMaxMin}.  The red dot is the origin, and the black dot is the point $(x_{0},y_{0},z_{0})$ on the principal axis. The multi-colored curves in Figures \ref{AlphaMaxMin}(a) and \ref{AlphaMaxMin}(b) are the computed geodesic curves for $\theta^{1}$ and $\theta^{2}$, respectively.  In each case, the curves at the furthest clockwise positions are the curves that were computed above using the maximal and minimal infinitesimal eigenvectors. As we move in a counter-clockwise direction, the additional multi-colored curves are those geodesics that would be computed by rotating $\xi_{1}$ and $\xi_{2}$ through an angle $\alpha$, in increments of $0.1$ radians.  For each curve, we compute the Euclidean (and Riemannian) distance from $(x_{0},y_{0},z_{0})$ to a point at an angle of $\pi/2$ from the origin, and then compute the angle of rotation, $\alpha$, that minimizes this distance in Figure \ref{AlphaMaxMin}(b), and thereby maximizes this distance in Figure \ref{AlphaMaxMin}(a).  The optimal value is $ \alpha = 0.463649$. The new initial directions for this rotation are $\xi_{1} = ( 0,  0.894431,  -0.447205)$ and $\xi_{2} = ( 0, 0.707101, 0.707112 )$, and we have:
\begin{align}
\xi_{1}^{T}  \left(\begin{array}{c} {g}_{ij}(x_{0},y_{0},z_{0}) \end{array}\right) \xi_{1}  &=   79.9998 \approx 80 \notag \\
\xi_{2}^{T}  \left(\begin{array}{c} {g}_{ij}(x_{0},y_{0},z_{0}) \end{array}\right) \xi_{2}  &=  49.9999 \approx 50 \notag
\end{align}
Thus, although the new initial directions are not optimal as infinitesimals, they do maximize and minimize the metric globally.   For $\theta^{1}$, the Riemannian and Euclidean distance is now 13.5064, and for $\theta^{2}$, the Riemannian and Euclidean distance is now 12.1106.  Furthermore, as Figure \ref{AlphaMaxMin} suggests, the new $\theta^{1}$ and $\theta^{2}$ curves match the rotated curves from the $xy$ plane and the $xz$ plane, respectively, that were identified in Figures \ref{GaussianManifold} and \ref{Gaussian.frb.rhotheta}.

We can now investigate the mapping of sample data in these coordinates.  Figure \ref{IntManScatter} shows the coordinate system and the data points, restricted to the positive $x$ axis before it was rotated through the angle $ \pi / 3 $ around the line from $(0,0,0)$ to $(1,1,1)$. There are 526 points in this half-space. The green dot is a point on the $\theta^{1}$ curve at a distance of 7.15541 from $(x_{0},y_{0},z_{0})$, and the blue dot is a point on the $\theta^{2}$ curve at a distance of 7.07106 from $(x_{0},y_{0},z_{0})$.  The data points have been projected along the $\rho$ coordinate curve to the Frobenius integral manifold at a constant Riemannian distance of 50.0 from the origin. Notice that the density of the data is higher near the $\theta^{2}$ coordinate curve than it is near the $\theta^{1}$ coordinate curve.

\begin{figure}[htbp]
\begin{center}
\includegraphics[width=5in]{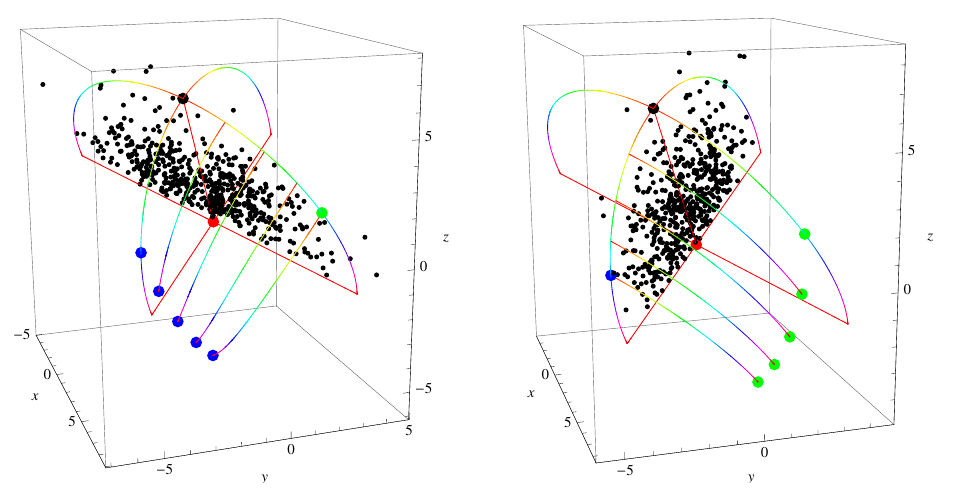}
\caption{Computing $\Theta$ coordinates on the Frobenius integral manifold: (a) Computing $\theta^{1}$ followed by $\theta^{2}$, and projecting the data points onto the $\theta^{1}$ surface; (b) Computing $\theta^{2}$ followed by $\theta^{1}$, and projecting the data points onto the $\theta^{2}$ surface.}
\label{ProjectMaxMin}
\end{center}
\end{figure}

We now compute the values of the $\Theta$ coordinates on the Frobenius integral manifold, for each of the 526 points.  There are two ways to do this: Figure \ref{ProjectMaxMin}(a) shows how to measure the distance along the $\theta^{1}$ coordinate curve towards the green dot, and then along the transverse coordinate flows, as defined in \eqref{TransverseCurves}, towards the blue dots. Let's call this result: $(\rho,\theta^{1},\theta^{2})$. Figure \ref{ProjectMaxMin}(b) shows how to measure the distance along the $\theta^{2}$ coordinate curve towards the blue dot, and then along the transverse coordinate flows, as defined in \eqref{TransverseCurves}, towards the green dots. Let's call this result: $(\rho,\theta^{2},\theta^{1})$. For example, proceeding to the furthest  blue and green dots in each case, we would be computing $(\rho,7.15541,7.07106)$ in Figure \ref{ProjectMaxMin}(a) and $(\rho,7.07106,7.15541)$ in Figure \ref{ProjectMaxMin}(b), but these would be two different points on the manifold!  Taking measurements along these flows only gives us a direct mapping from $(\rho,\Theta)$ to $(x,y,z)$, of course, but we can then invert the functions to obtain a mapping from $(x,y,z)$ to either $(\rho,\theta^{1},\theta^{2})$ or $(\rho,\theta^{2},\theta^{1})$. There is an annoying technical problem when we try to extend these results beyond the quadrant in the forefront of Figure \ref{ProjectMaxMin}. With the basis vectors ${\bf V} ({\bf x})$ and ${\bf W} ({\bf x})$ centered on the $x$ axis, we encounter singularities when we try to solve the differential equations for the coordinate flows.  But we can avoid these problems by switching to a $y$-centered basis for the back side of the $\theta^{1}$ curve, and a $z$-centered basis for the back side of the $\theta^{2}$ curve.  
 
There is no error in the mapping we have just constructed. But we are now in a position to drop one of the $\Theta$ coordinates, to obtain a lower-dimensional encoding of our data.  Which one?  We can either use $(\rho,\theta^{1},\theta^{2})$ and truncate it to $(\rho,\theta^{1})$, or use $(\rho,\theta^{2},\theta^{1})$ and truncate it to $(\rho,\theta^{2})$, and we would like to know the error in each case.  For specificity, let's focus on the first case, in which we drop $\theta^{2}$. One way to conceptualize the error is to measure the Euclidean distance along the $\theta^{2}$ coordinate curve that we are dropping, and then scale this distance down, proportionately, given the position of the data point along the $\rho$ coordinate curve.  For example, the point $(x,y,z) = (3.07959, 0.121701, 0.476748)$ is mapped to $(\rho,\theta^{1},\theta^{2}) = (3.29006,6.9341,2.70928)$. The Euclidean distance along the $\theta^{2}$ coordinate curve is computed to be $4.68592$.  (Recall that the parametrizations of the transverse coordinate flows in \eqref{TransverseCurves} are not equivalent to Euclidean arc length, except along the main coordinate axes.) The Euclidean distance from $(x,y,z)$ along the $\rho$ coordinate curve to the Frobenius integral manifold is computed to be $4.00782$.  Thus the ``reconstruction error'' for this data point is
\begin{equation*}
4.68592 * \left( \frac{3.29006}{3.29006 + 4.00782} \right) \; = \; 2.11253
\end{equation*}
We can now compute the \emph{root-mean-squared} (RMS) reconstruction error for the 526 sample data points in our half-space, using each encoding.  For the truncation from $(\rho,\theta^{1},\theta^{2})$ to $(\rho,\theta^{1})$, the RMS error is $3.27489$, and for the truncation from $(\rho,\theta^{2},\theta^{1})$ to $(\rho,\theta^{2})$, the RMS error is $2.97351$. Thus, according to this measure, the better lower-dimensional encoding is $(\rho,\theta^{2})$. 
 
There are other ways to define the reconstruction error, however, and they might yield different results.  One crude approach is to actually project the data along the transverse coordinate curves to the $\theta^{1}$  and $\theta^{2}$ surfaces, and to measure distances in the ambient Euclidean space ${\bf R}^{3}$.  Such projections are illustrated in Figures \ref{ProjectMaxMin}(a) and \ref{ProjectMaxMin}(b).  We can then compute an analogue of the ``variance'' on each surface, as in Principal Components Analysis. For the projection onto the $\theta^{1}$ surface in Figure \ref{ProjectMaxMin}(a), the RMS deviation from the origin is $5.01522$, and for the projection onto the $\theta^{2}$ surface in Figure \ref{ProjectMaxMin}(b), the RMS deviation from the origin is $4.36232$. We can also measure the distance in ${\bf R}^{3}$ from the original data point $(x,y,z)$ to its projection onto one of these surfaces, a quantity that we might call the ``discrepancy.''  For the projection onto the $\theta^{1}$ surface, the RMS discrepancy is $2.61531$, and for the projection onto the $\theta^{2}$ surface, the RMS discrepancy is $2.87559$.  
 
\subsection{The Curvilinear Gaussian.}
\label{Exp:CurvGauss}

\begin{figure}[htbp]
\begin{center}
\includegraphics[width=3.5in]{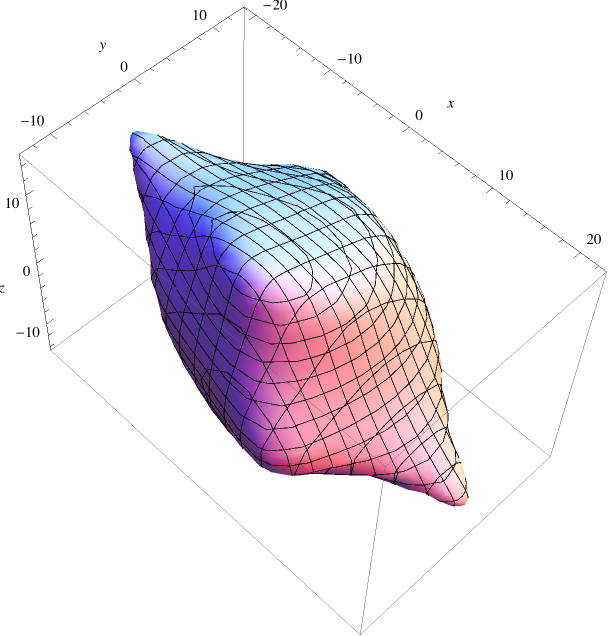}
\caption{Contour plot for the surface of a curvilinear Gaussian potential at $U(x,y,z) = -10$.}
\label{CurvilinearGaussian}
\end{center}
\end{figure}

The methodology of Section \ref{Exp:Gauss} was exploratory. The quadratic potential can always be solved analytically, no matter how it is rotated, but we were interested in determining whether an ``optimal'' curvilinear coordinate system could be computed numerically, using just the Riemannian dissimilarity metric and the Euler-Lagrange equations, without prior knowledge of the rotation.  And what do we mean by an ``optimal'' curvilinear coordinate system? Using a reasonable definition of the ``reconstruction error,'' we saw that the truncation from $(\rho,\theta^{2},\theta^{1})$ to $(\rho,\theta^{2})$ was better than the truncation from $(\rho,\theta^{1},\theta^{2})$ to $(\rho,\theta^{1})$, although an analogue of Principal Components Analysis would suggest the opposite. 

In this section, we will consider an example for which analytical results are not available, and in which we will be free to apply rotations whenever they would simplify the numerical calculations. In particular, we will rotate the original $xyz$ coordinate system to align the $x$ axis with the principal axis, once we have computed it, and we will apply additional rotations in the directions of $\Theta$ to simplify the computation of the transverse coordinate curves.  We will also study further the reconstruction error for a  $(\rho,\Theta)$ coordinate system, using simulated data.

Let's start with a cubic polynomial: $C(t) = t^3 - t^2 - t$.  We then define a cubic polynomial coordinate transformation from $(x,y,z)$ to $(u,v,w)$ as follows:
\begin{align} 
u \;=\; u(x,y,z) =  & \; C(1.4 \; y) + 2 x (y^2 + z^2) \notag \\
v \;=\; v(x,y,z) = & \; C(1.2 \; z) + 2 y (z^2 + x^2) \notag \\
w \;=\; w(x,y,z) = & \; C(1.0 \; x) + 2 z (x^2 + y^2) \notag 
\end{align}
Finally, we define $U({\bf x})$ as a quadratic potential function in the variables $u$, $v$ and $w$:
\begin{equation}
U(x,y,z) = - \frac{1}{2} ( a \,u(x,y,z)^2 + b \,v(x,y,z)^2 + c \,w(x,y,z)^2) * 10^{-6} \notag
\end{equation}
Thus $U({\bf x})$ is a sixth-degree polynomial in $x$, $y$ and $z$, and the gradient, $\nabla U({\bf x})$, is a fifth-degree polynomial. There are no known closed-form solutions to the Feynman-Kac formula, given by either  \eqref{defPt} or \eqref{defQt}, when $U({\bf x})$ and $V({\bf x})$ are higher-order polynomials.  However, it is possible to discretize the Feynman-Kac ``path integral,'' and obtain approximate numerical solutions.  See, for example, \cite{Lyasoff:1999:PIM}.

 \begin{figure}[htbp]
\begin{center}
\includegraphics[width=5in]{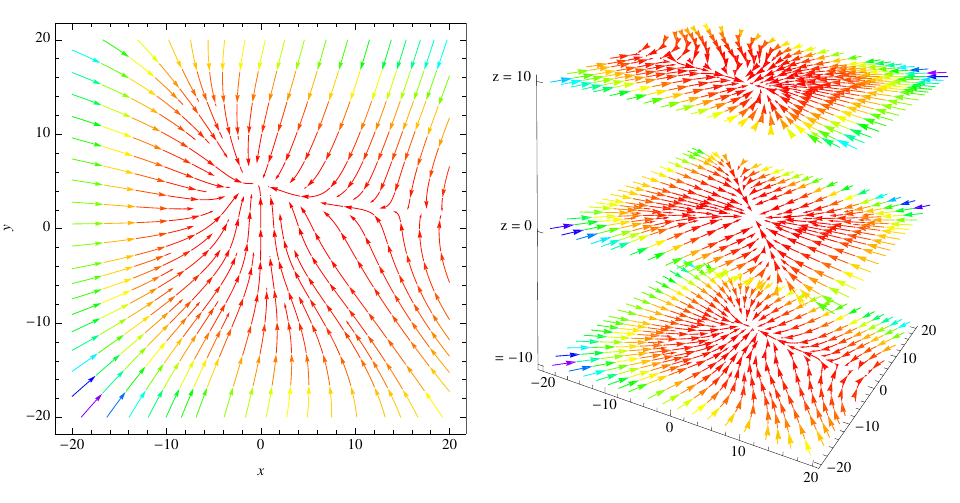}
\caption{Gradient vector field for the curvilinear Gaussian potential in Figure \ref{CurvilinearGaussian}: (a) at $z = -10$; (b) at $z = 10$, $z = 0$ and $z = -10$.}
\label{StreamPlotStack}
\end{center}
\end{figure}

For a numerical example, set $a = 1, b = 2, c = 4$.  Figure \ref{CurvilinearGaussian} then shows the surface defined by the equation $U(x,y,z) = -10$.  Figure \ref{StreamPlotStack}(a) shows a {\tt StreamPlot} of the gradient vector field generated by $\nabla U(x,y,z)$ at $z=-10$.  Figure \ref{StreamPlotStack}(b) shows a stack of such stream plots, at the values $z = 10$, $z = 0$ and $z = -10$.  Notice how the drift vector twists and turns to counteract the dissipative effects of the diffusion term, and maintain an invariant probability measure.

\begin{figure}[htbp]
\begin{center}
\includegraphics[width=3.5in]{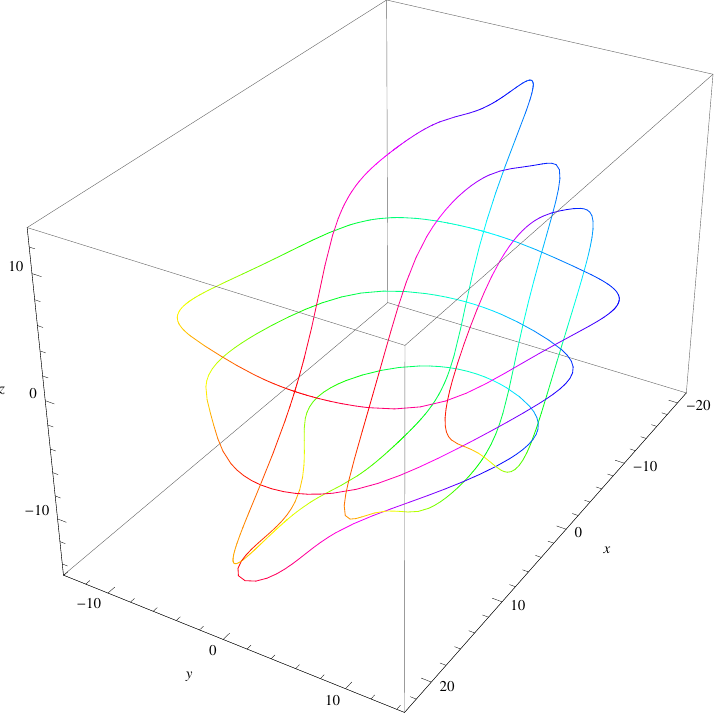}
\caption{An integral manifold with a global coordinate system for the curvilinear Gaussian potential in Figure \ref{CurvilinearGaussian}.}
\label{CurvGaussIntManifold}
\end{center}
\end{figure}
  
Figure \ref{CurvGaussIntManifold} is analogous to Figure \ref{GaussianManifold} in the Gaussian case, and depicts the integral manifold that passes through the point $( 20, 0, -10 )$.  The coordinate curves in Figure \ref{CurvGaussIntManifold} are generated by a global coordinate system centered on the $x$ axis, with $P({\bf x}) {\bf V} = (-Q({\bf x}),P({\bf x}),0)$ and $P({\bf x}) {\bf W} = (-R({\bf x}),0,P({\bf x}))$.  These curves are thus analogous to the global $\theta$ and $\phi$ coordinate curves shown in Figure \ref{GaussianManifold}.  
 
Figure \ref{CGCoordsWithEig} shows the $( \rho, \Theta )$ surfaces computed by our numerical techniques, and analogous to the $( \rho, \Theta )$ surfaces in Figure \ref{ProjectMaxMin}.   As before, we start off with the basis vectors ${\bf V} ({\bf x})$ and ${\bf W} ({\bf x})$ centered on the $x$ axis, and we proceed through the three steps outlined at the end of Section \ref{proto}, with iterations to convert a local solution (based on infinitesimal eigenvectors) into a global solution (based on geodesic curves over finite distances). Here are the three steps:
\begin{enumerate}
\item Find a principal axis for the $\rho$ coordinate.
\vspace{1ex} 

We saw in Section \ref{Exp:Gauss} that there are two ways to find a maximal point on the principal axis, which yield approximately the same results as long as $\| \nabla U({\bf x}) \|$ is monotonic. In the curvilinear Gaussian case, we first minimize  $\| \nabla U(x,y,z) \| ^{2}$ on a sphere through the point $( 20, 0, -10 )$ to obtain the value:  $(20.4316, 1.27953, -8.99505)$. The integral curve ${\gamma}(t)$ from this point towards the origin has Euclidean length $20.9043$ and Riemannian length $6.30873$.  We now use {\tt NDSolve}, {\tt NIntegrate} and {\tt FindMinimum} to compute another integral curve, ${\gamma}(t)$, possibly distinct, which starts on the surface $x^{2} + y^{2} + z^{2} = 500$ and extends for the distance $t = 20.904$, i.e., just short of the singularity at the origin, and which has minimal Riemannian length.  The starting point for this curve turns out to be $(x_{0},y_{0},z_{0}) = (20.4317, 1.27944, -8.9949)$ and the Riemannian length turns out to be $6.30863$.  We take this to be the maximal point on the principal axis.  See the black dot in the lower right quadrant in Figure \ref{CGCoordsWithEig}.

\begin{figure}[htbp]
\begin{center}
\includegraphics[width=4.5in]{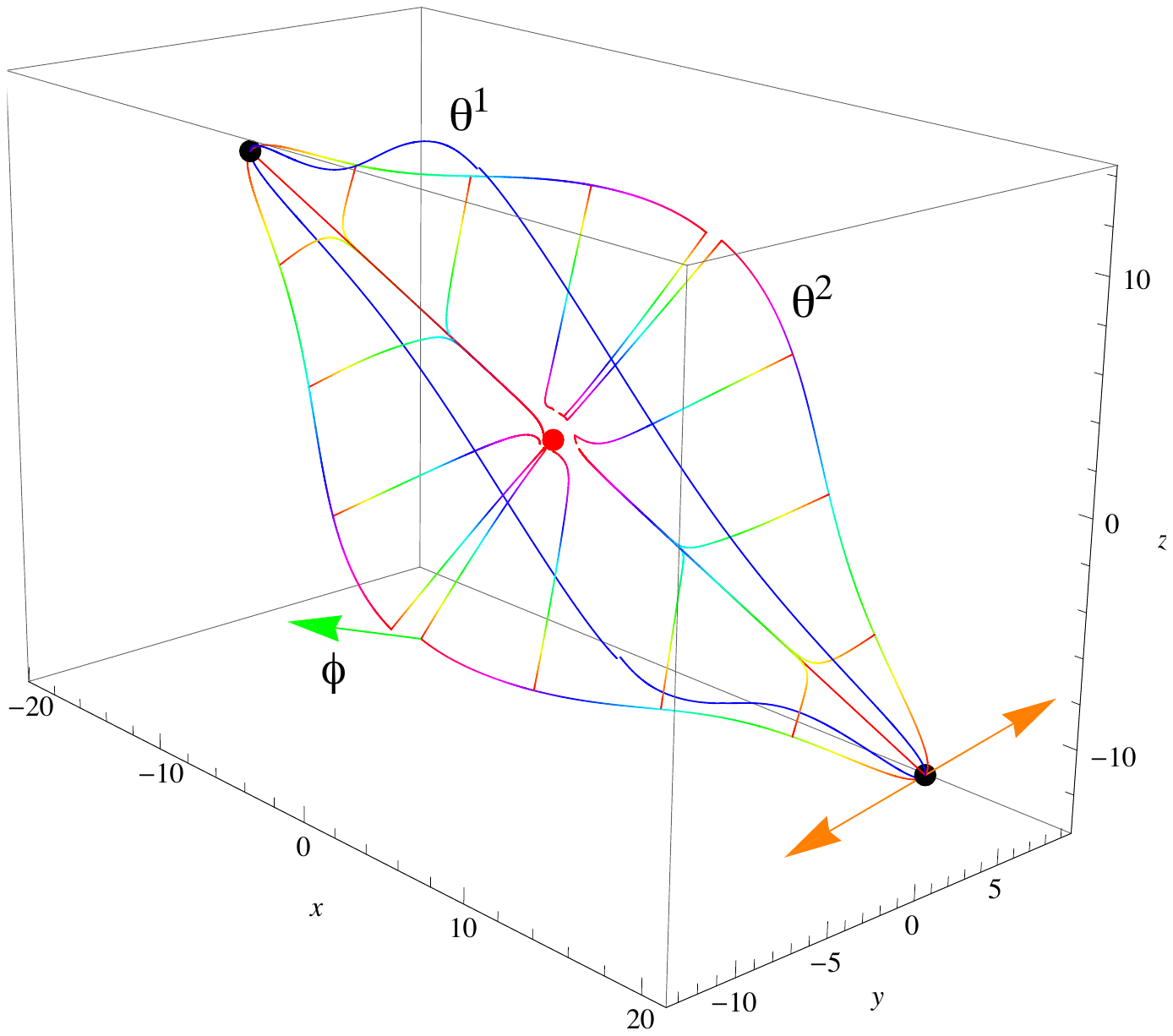}
\caption{Geodesic coordinate curves for the curvilinear Gaussian potential in Figure \ref{CurvilinearGaussian}.}
\label{CGCoordsWithEig}
\end{center}
\end{figure}

We also need to compute the location of the black dot in the upper left quadrant in Figure \ref{CGCoordsWithEig}, which we call the \emph{antipodal point}.  For $(x_{0},y_{0},z_{0})$, we were looking for a point with a fixed Euclidean distance from the origin and a minimal Riemannian distance.  We are now looking for a point with a fixed Riemannian distance from the origin and a maximal Euclidean distance.  But every point on the Frobenius integral manifold has a constant Riemannian distance from the origin.  Thus, to locate the antipodal point, we first follow the global coordinate curve in the \emph{xz} plane (see Figure \ref{CurvGaussIntManifold}), from $(x_{0},y_{0},z_{0})$ halfway around the loop to a point in the vicinity of the solution: $(-17.3636, 1.27944, 7.55937)$. We then search along the Frobenius integral manifold, using the global coordinate curves in the \emph{xy} and \emph{xz} planes, to find a point at a maximal Euclidean distance from the origin: $(-19.2034,-1.25668,9.25639)$. We take this to be the value of $(x_{1},y_{1},z_{1})$, the antipodal point. 
 
Now that we have computed the principal axis, we can rotate our original \emph{xyz} coordinate system to align the \emph{x}-axis with $(x_{0},y_{0},z_{0})$, which simplifies many of the calculations that we want to do in a $( \rho, \Theta )$ coordinate system with an \emph{x}-centered basis.  In the rotated coordinate system, $(x_{0},y_{0},z_{0})$ is mapped into $(22.3607,0.0,0.0)$ and $(x_{1},y_{1},z_{1})$  is mapped into $(-21.3422,$ $-0.0444194,0.733775)$.  To facilitate comparison of the figures, however, we will continue to generate graphics in the original orientation.
 
\vspace{1ex} 
\item Determine the principal directions for the $\Theta$ coordinates.
\vspace{1ex} 

The orange arrows in Figure \ref{CGCoordsWithEig} depict the eigenvectors in the original \emph{xyz} coordinate system associated with the minimal eigenvalue, in the positive \emph{y}-direction and the negative \emph{y}-direction, respectively.  But the geodesic coordinate curves for the $\theta^{2}$ surface are determined by rotating these eigenvectors in a counter-clockwise direction through an angle $\alpha$ that minimizes the ratio of (i) the Riemannian length to (ii) the Euclidean angle from the origin, up to the Euclidean angle $\pi/2$.  For the eigenvector in the positive direction, $\alpha = 0.952169$, and for the eigenvector in the negative direction, $\alpha = 1.12681$.  Furthermore, when we examine these optimal geodesic coordinate curves, we see that they both extend beyond the Euclidean angle $\pi/2$ from the origin, so we terminate them at this point. 

The details are slightly different for the eigenvectors associated with the maximal eigenvalue.  In this case, the optimal geodesic coordinate curves have different lengths, one extending to a Euclidean angle substantially more than $\pi/2$, and one extending to a Euclidean angle substantially less. We thus combine the coordinate curves in the positive and negative directions, and maximize jointly the ratio of their Riemannian lengths to the Euclidean angles they subtend. The optimal result is a rotation of the maximal eigenvector $\xi_{1}({\bf x})$ in the counter-clockwise direction through an angle $\alpha = 0.114166$. These geodesic coordinate curves are illustrated in blue and labeled as $\theta^{1}$ in Figure \ref{CGCoordsWithEig}.

We have applied the same constructions to the antipodal point $(x_{1},y_{1},z_{1})$.  The results are illustrated in Figure \ref{CGCoordsWithEig}, but we will not discuss them in detail.

\vspace{1ex} 
\item Compute the geodesic coordinate curves for each of the principal $\Theta$ directions.
\vspace{1ex} 
  
Using the initial values $(x_{0},y_{0},z_{0})$ and $(x_{1},y_{1},z_{1})$ computed in step (1) and the various principal directions computed in step (2), we construct the Euler-Lagrange equations for the variational problem given by \eqref{EnergyFunctional} and \eqref{GammaConstraint}, and we solve them using {\tt NDSolve}.  We have already discussed the results of these calculations, and they are illustrated in Figure \ref{CGCoordsWithEig}.  Figure \ref{CGCoordsWithEig} also shows the $\rho$ coordinate curves drawn from fixed intervals along the $\theta^{2}$ geodesics, which gives us a good sense of the shape of the $\theta^{2}$ surface.

In the simple Gaussian case in Section \ref{Exp:Gauss}, we only made use of the two coordinate curves, $\theta^{1}$ and $\theta^{2}$, each one serving as the source of the transverse coordinate curves for the other, as illustrated in Figure \ref{ProjectMaxMin}. In the curvilinear Gaussian case, however, it is convenient to add another coordinate, $\phi$, which is orthogonal to both $\theta^{1}$ and $\theta^{2}$, and which can be used to define the transverse coordinate curves for each coordinate axis.  Consider the green arrow in Figure \ref{CGCoordsWithEig}.  This is a vector orthogonal to the $\theta^{2}$ geodesic coordinate curve, which is attached to the curve at a Euclidean angle of $\pi/2$ from the origin, and which lies in the tangent plane to the Frobenius integral manifold at that point.  We use this vector as the initial direction for the construction of another geodesic curve on the Frobenius integral manifold, and we construct similar geodesics at all the maximal points along $\theta^{1}$ and $\theta^{2}$.  Sometimes, these geodesics encounter singularities, but we can avoid this problem by (i) using a \emph{y}-centered basis instead of an \emph{x}-centered basis, and (ii) rotating the coordinate system around the \emph{x}-axis. Since we previously rotated the original \emph{xyz} coordinate system so that $(x_{0},y_{0},z_{0}) = (22.3607,0.0,0.0)$, we now have the option of rotating again to align the positive y-axis with the maximal eigenvector $\xi_{1}({\bf x})$, or its displacement through the angle $\alpha$, or any other convenient quantity. By a judicious choice of rotations, we can guarantee that our coordinate system covers the entire Frobenius integral manifold.
  
\vspace{1ex} 
\end{enumerate}
Finally, to fill out the coordinate system, we need to define the flows in equation \eqref{TransverseCurves} for the coordinate axes $\theta^{1}$, $\theta^{2}$, and $\phi$.  We will see an example in our discussion of Figure \ref{CGError}, below.

\begin{figure}[htbp]
\begin{center}
\includegraphics[width=4.0in]{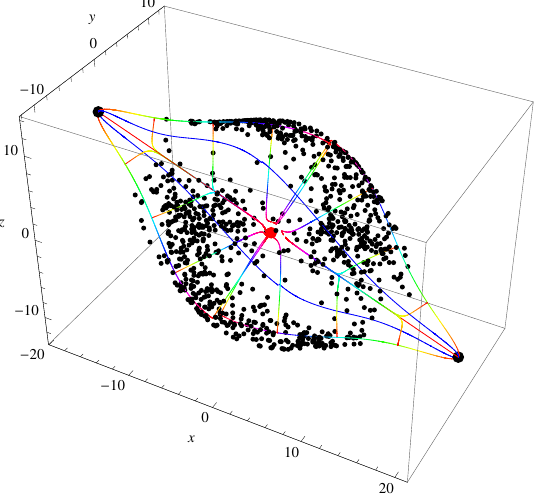}
\caption{Projecting data points from the curvilinear Gaussian potential in Figure \ref{CurvilinearGaussian} along the $\rho$ coordinate curves to the Frobenius integral manifold.}
\label{CGCoordsWithData}
\end{center}
\end{figure}
 
Figure \ref{CGCoordsWithData} shows 1000 sample data points projected onto the Frobenius integral manifold, analogous to Figure \ref{IntManScatter} in Section \ref{Exp:Gauss}.  The data was generated from our curvilinear Gaussian probability distribution, using Gibbs sampling \cite{CasellaGeorge1992}. (The Gibbs sampler is easy to implement, since the conditional distributions of \emph{x} given \emph{y} and \emph{z}, \emph{y} given \emph{x} and \emph{z}, and \emph{z} given \emph{x} and \emph{y}, can be defined analytically.) For each data point, in \emph{xyz} coordinates, the ${\gamma}(t)$ curve is computed inwards to determine the value of the $\rho$ coordinate, and then computed outwards to a constant Riemannian distance of 6.30863 from the origin.  Notice that the density of the data is higher near the $\theta^{2}$ coordinate curve than it is near the $\theta^{1}$ coordinate curve. We can quantify this observation by computing the ``reconstruction error,'' as we did in Section \ref{Exp:Gauss}. 

\begin{figure}[htbp]
\begin{center}
\includegraphics[width=4.0in]{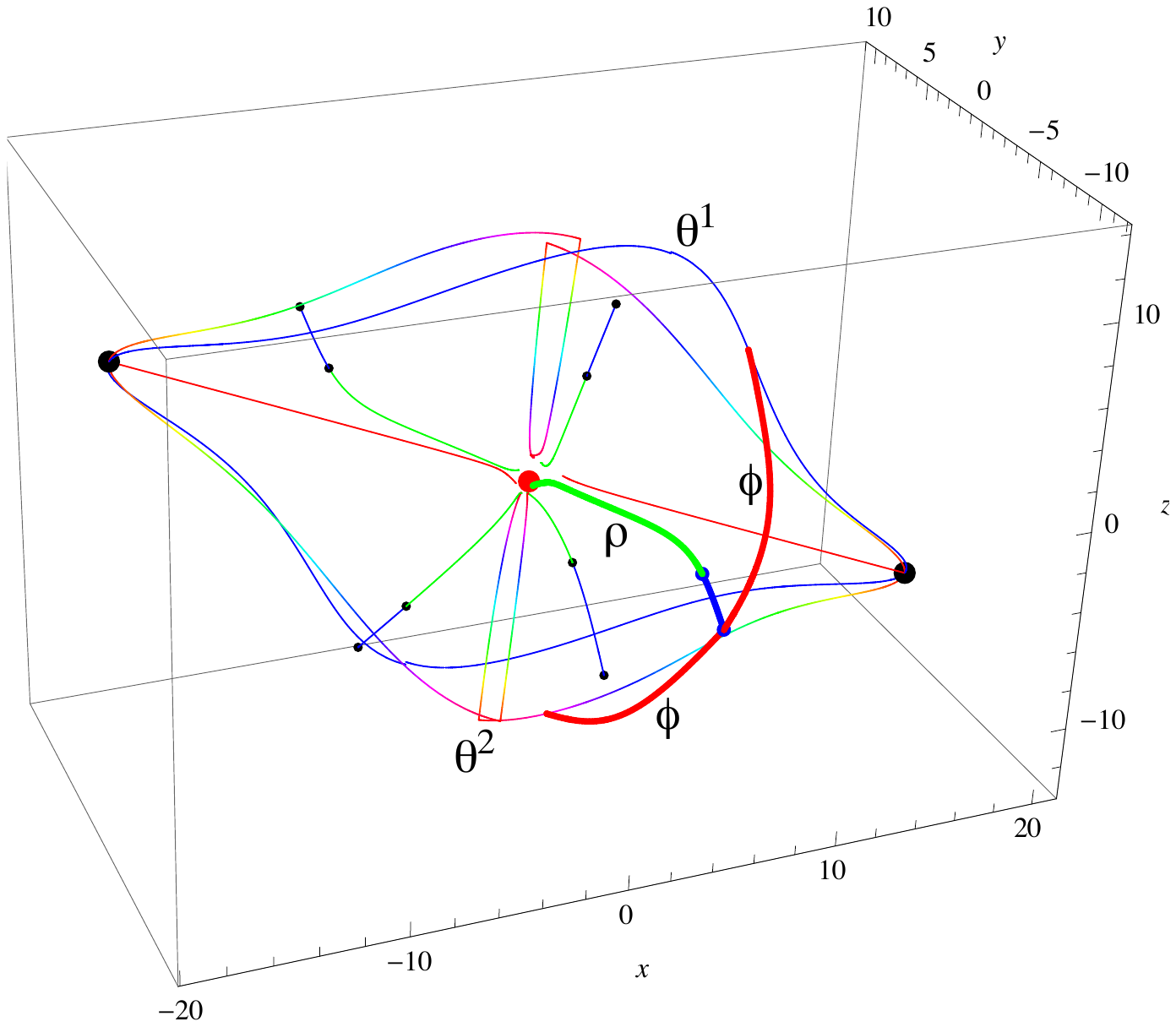}
\caption{The two curvilinear coordinate systems, $(\rho,\theta^{1},\phi)$ and $(\rho,\theta^{2},\phi)$, and the quantities needed to compute the ``reconstruction error'' for a single point.}
\label{CGError}
\end{center}
\end{figure}

Figure \ref{CGError} shows how to define two curvilinear coordinate systems, $(\rho,\theta^{1},\phi)$ and $(\rho,\theta^{2},\phi)$.  Five sample data points are plotted here, along with their ${\gamma}(t)$ coordinate curves.  For each data point, the curve inwards to the origin is shown in green, delineating the $\rho$ coordinate, and the curve outwards to the Frobenius integral manifold is shown in blue.  For the data point in the foreground, which is highlighted, we also see the geometric interpretation of $(\rho,\theta^{1},\phi)$ and $(\rho,\theta^{2},\phi)$.  In \emph{xyz} coordinates, this point is located at $(4.68576,-8.11895,-1.25188)$.  The value of the $\rho$ coordinate is 9.65892, which is the distance along the green curve, and the distance along the blue curve to the Frobenius integral manifold is 3.52365.  What are the values of the $\Theta$ coordinates? Using the $\theta^{1}$ coordinate axis, we compute the numerical approximation $(\rho,\theta^{1},\phi) = (9.65892,22.0456,14.1106)$.  This means that we proceed along the flow $\vec{\theta }_{s}^{\,1}({\bf x}) $ starting at ${\bf x} = (x_{0},y_{0},z_{0})$ and with $s = 22.0456$, until we reach the point $(7.71353, -5.10099, 7.52968)$ in \emph{xyz} coordinates.  We then proceed along the flow $\vec{\phi }_{t}({\bf x}) $ with ${\bf x} = (7.71353, -5.10099, 7.52968)$ and $t = 14.1106$ to the point $(4.1299,$ $-11.5028,-2.04734)$.  Note that the exact location of this point on the Frobenius integral manifold is $(4.1302, -11.5029, -2.04746)$.  The Euclidean distance along the $\phi$ coordinate curve is 13.3184.  (Recall again that the parametrizations of the transverse coordinate flows in \eqref{TransverseCurves} are not equivalent to Euclidean arc length, except along the geodesic coordinate axes.) Thus the reconstruction error from truncating $(\rho,\theta^{1},\phi)$ to $(\rho,\theta^{1})$ is:  
\begin{equation*}
13.3184 * \left( \frac{9.65892}{9.65892 + 3.52365} \right) \; = \; 9.75848
\end{equation*}
The other alternative is to use the $\theta^{2}$ coordinate axis, for which we compute the approximation $(\rho,\theta^{2},\phi) = (9.65892, 21.4622, 13.0936)$.  This means that we proceed along the flow $\vec{\theta }_{s}^{\,2}({\bf x}) $ starting at ${\bf x} = (x_{0},y_{0},z_{0})$ and with $s = 21.4622$, until we reach the point $(-0.271266,$ $ -3.58481,$ $ -9.85897)$ in \emph{xyz} coordinates.  We then proceed along the flow $\vec{\phi }_{t}({\bf x}) $ with ${\bf x} = (-0.271266, -3.58481, -9.85897)$ and $t = 13.0936$ to the point $(4.12554,-11.5019, -2.04706)$.  Note again that the exact location of this point on the Frobenius integral manifold is $(4.1302,$ $ -11.5029, -2.04746)$. The Euclidean distance along the $\phi$ coordinate curve in this case is 12.5779,  and the reconstruction error from truncating $(\rho,\theta^{2},\phi)$ to $(\rho,\theta^{2})$ is:  
\begin{equation*}
12.5779 * \left( \frac{9.65892}{9.65892 + 3.52365} \right) \; = \; 9.21587
\end{equation*}
Thus, for this one data point, the $(\rho,\theta^{2})$ encoding is slightly better than the $(\rho,\theta^{1})$ encoding.

Furthermore, for the majority of data points,  we see that the ranking goes the same way. On our sample of 1000 points, the RMS reconstruction error for the truncation from $(\rho,\theta^{1},\phi)$ to $(\rho,\theta^{1})$ is 5.9431, and the RMS reconstruction error for the truncation from $(\rho,\theta^{2},\phi)$ to $(\rho,\theta^{2})$ is 4.82787.

These calculations confirm our impressions from Figure \ref{CGCoordsWithData}, and they are consistent with the second hypothesis quoted from \cite{nipsRifaiDVBM11}:
\begin{quotation}
\begin{enumerate}

\item[1.]  $\ldots$

\item[2.]  The {\bf (unsupervised) manifold hypothesis}, according to which real world data presented in
high dimensional spaces is likely to concentrate in the vicinity of non-linear sub-manifolds
of much lower dimensionality $\dots$ [citations omitted] 

\item[3.]  $\ldots$

\end{enumerate}
\end{quotation}
Indeed, the only mismatch with our example is one of dimensionality:  By ``high dimensional'' we mean 3, and by ``much lower dimensionality'' we mean 2! We will address this issue in Section \ref{FutureWork}, below.

Despite this simplification, the curvilinear Gaussian example illustrates clearly the synergistic link between the probabilistic model and the geometric model in the theory of differential similarity.  The geodesic curves on the Frobenius integral manifold tend to follow the \emph{modes} of the probability distribution. 
First, the origin of the coordinate system is a point at which $\nabla U({\bf x}) = (0,0,0)$, which maximizes the probability density.  Second, to compute the principal axis, we are looking for a point with a minimal Riemannian distance for a fixed Euclidean distance, or a maximal Euclidean distance for a fixed Riemannian distance.  Under either formulation, this is an axis that maximizes probability. Third, for the directional coordinates, we are looking for a geodesic curve on the Frobenius integral manifold that covers a minimal Riemannian distance for a fixed angular Euclidean distance, or a maximal angular Euclidean distance for a fixed Riemannian distance.  Under either formulation, again, this is a curve that maximizes probability. Thus, in general, we are \emph{minimizing} dissimilarity and \emph{maximizing} probability.  This is the primary intuition behind the claim that we are constructing an ``optimal'' lower-dimensional coordinate system.  

\begin{figure}[htbp]
\begin{center}
\includegraphics[width=4.0in]{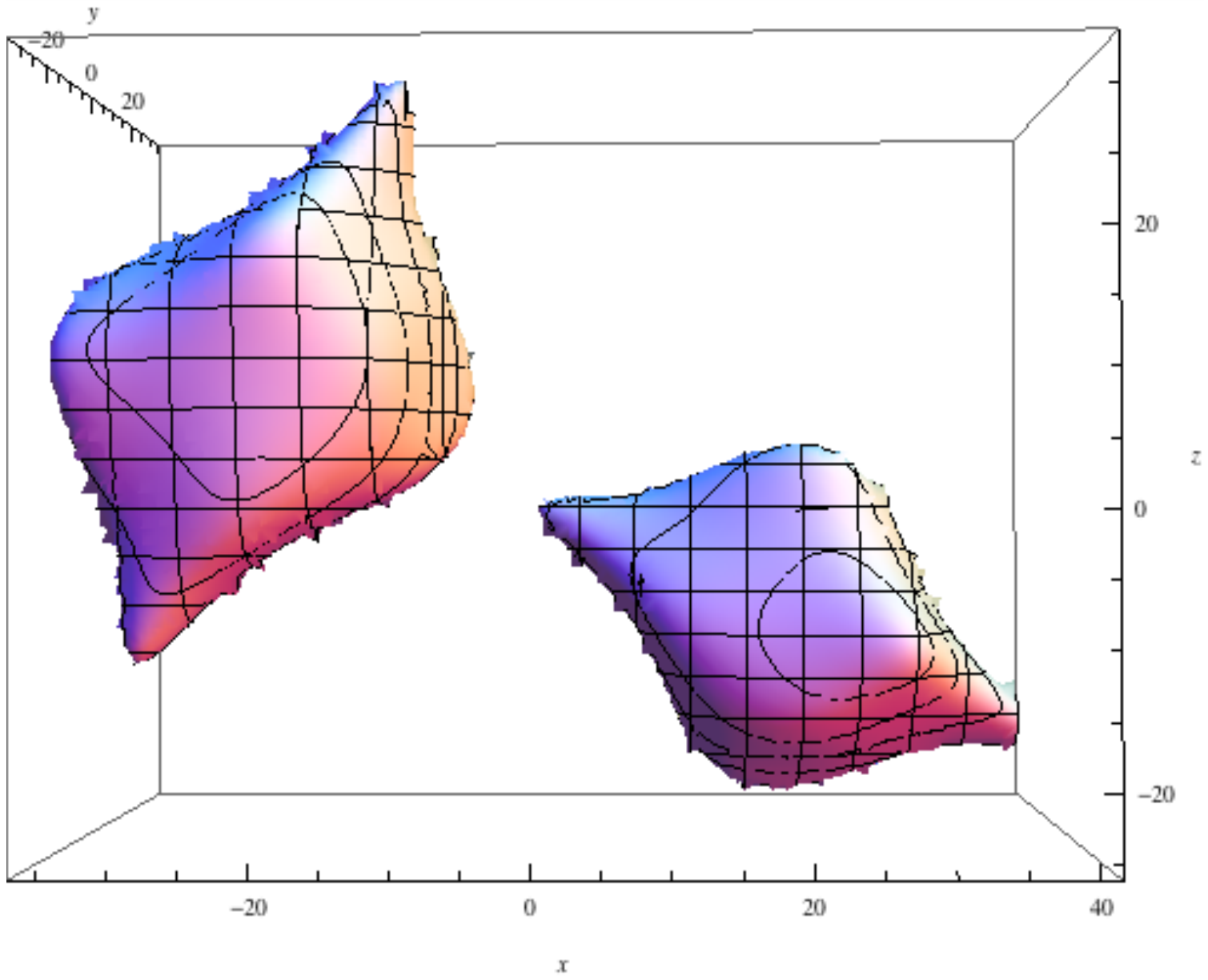}
\caption{A mixture of two curvilinear Gaussians, translated and rotated.}
\label{BiCurvilinearGaussian}
\end{center}
\end{figure}

\subsection{The Bimodal Curvilinear Gaussian.}
\label{Exp:BiCurvGauss}
    
Finally, we consider a bimodal case.  Figure \ref{BiCurvilinearGaussian} shows two copies of the curvilinear Gaussian defined in Section \ref{Exp:CurvGauss}.  One copy has been translated from $(0,0,0)$ to $(20,20,-10)$.  The other copy has been translated from $(0,0,0)$ to $(-20,-20,10)$ and rotated by $\pi / 2$ around a line parallel to the $y$-axis.  But the probability density is a \emph{mixture}.  If $U_{1}({\bf x})$ is the potential function for the first copy and $U_{2}({\bf x})$ is the potential function for the second copy, then the invariant probability density is given by:
\[
e^{2 U({\bf x})} \; \simeq \; p_{1} \, e^{\,2\,U_{1}({\bf x})}  +  p_{2} \, e^{\,2\,U_{2}({\bf x})},
\]
modulo an appropriate normalization factor. Figure \ref{BiCurvilinearGaussian} is actually showing the surface defined by the equation:
\[
e^{U_{1}(x,y,z)} + e^{U_{2}(x,y,z)} \;=\; 0.0001
 \]
The advantage of this representation lies in the fact that our calculations for each copy will be almost independent of each other.  Observe that the effective potential function for the mixture will be: 
\[
U({\bf x}) \; \simeq \; \frac{1}{2}  \log ( \, p_{1} \, e^{\,2\,U_{1}({\bf x})}  +  p_{2} \, e^{\,2\,U_{2}({\bf x})} \, )
\]
Thus the gradient of $U({\bf x})$ in a neighborhood of  $(20,20,-10)$ will be almost identical to the gradient of $U_{1}({\bf x})$ computed by itself, and the gradient of $U({\bf x})$ in a neighborhood of  $(-20,-20,10)$ will be almost identical to the gradient of $U_{2}({\bf x})$ computed by itself.
 
The mixture distribution thus provides a useful representation of \emph{clusters}.   Analyzing the situation, intuitively, in terms of our dissimilarity metric, the two clusters in Figure \ref{BiCurvilinearGaussian} will be exponentially far apart.  This picture is therefore consistent with the third hypothesis quoted from \cite{nipsRifaiDVBM11}:
 \begin{quotation}
\begin{enumerate}

\item[1.]  $\ldots$

\item[2.]  $\ldots$

\item[3.]  The {\bf manifold hypothesis for classification}, according to which points of different classes
are likely to concentrate along different sub-manifolds, separated by low density regions of
the input space.

\end{enumerate}
\end{quotation}

\section{Diffusion Coefficients and Dissimilarity Metrics}
\label{DiffuCoeffs&DissimMetrics}
\vspace{1ex}
 
Recall the main results from Section \ref{proto}:  We started with a diffusion process represented by an Ito stochastic differential equation, in Cartesian coordinates;  we transformed this into a Stratonovich equation in the coordinates ($\rho, \Theta$);  and we then converted this back into an Ito process  characterized by a differential operator with coefficients $ \alpha^{ij} (\rho, \Theta)$ and $ \beta^{i}(\rho, \Theta)$.  The one necessary ingredient was the Jacobian matrix of the coordinate transformation. 
 
As an illustration, let's try a brute force solution of these equations in the simple Gaussian case discussed in Section \ref{Exp:Gauss}.  The Jacobian is given by equation \eqref{Jgauss}.   For ease of reference, here is equation \eqref{dXtStrat}, rewritten for the three-dimensional coordinate system $(\rho, \theta,\phi)$:
 \begin{align}
\label{dXtStrat}
   dX(t) & \; = \; 
     \begin{pmatrix}
    \\
   {\mathbf \sigma}^{i}_{k}({\bf x}(\rho, \theta,\phi)) \\
    \\
    \end{pmatrix}
     \circ
      \begin{pmatrix}
         d\mathcal{B}_{1}(t)  \\
         d\mathcal{B}_{2}(t)  \\
         d\mathcal{B}_{3}(t)  \\
      \end{pmatrix}
      \; + \;    
          \begin{pmatrix}  
               \tilde{b}^{1}({\bf x}(\rho, \theta,\phi)) \\
              \tilde{b}^{2}({\bf x}(\rho, \theta,\phi))  \\
              \tilde{b}^{3}({\bf x}(\rho, \theta,\phi)) \\
          \end{pmatrix} 
          dt  \notag 
 \end{align}
For the moment, we will assume that  $\left( {\mathbf \sigma}^{i}_{k}({\bf x}(\rho, \theta,\phi)) \right)$ is an orthogonal transformation, but otherwise arbitrary.  Our procedure is to combine and solve equations \eqref{dXtStrat} and \eqref{dXtJacobian}, and then expand the result using Theorem \ref{Thm:L=sumAk+A0}.  When we do so, we discover that the ``sum of squares'' inside equation \eqref{L=sumAk+A0} yields an expression consisting of 2679 terms!  However, by using the fact that $\left( {\mathbf \sigma}^{i}_{k}({\bf x}(\rho, \theta,\phi)) \right)$ is an orthogonal transformation, we can eliminate all terms in which the factors $ {\mathbf \sigma}^{i}_{k} $ appear without derivatives.  Furthermore, all the terms that include derivatives of $ {\mathbf \sigma}^{i}_{k} $ are cancelled out by similar terms in the expansion of ${\bf A}_{0} \partial$ inside equation \eqref{L=sumAk+A0}.  The net result is equation \eqref{DiffusionCoeffs}, in the following form:
 \begin{equation}
 \mathcal{L}  \;=\; 
\frac{1}{2} \sum_{i,j=0}^{2} 
 \alpha^{ij} (\rho, \theta, \phi) \frac{\partial^{2}}{\partial u^{i} \partial u^{j}} 
\, + \,
 \sum_{i=0}^{2} \beta^{i}(\rho, \theta, \phi) \frac{\partial}{\partial u^{i} } \notag 
\end{equation}
where $u^{0} = \rho$,  $u^{1} = \theta$ and  $u^{2} = \phi$.  Thus, the exact choice we make for the transformation $\left( {\mathbf \sigma}^{i}_{k}({\bf x}(\rho, \theta,\phi)) \right)$ turns out to be irrelevant.  However, the diffusion coefficients  $ \alpha^{ij} (\rho, \theta, \phi)$ and the drift coefficients $ \beta^{i}(\rho, \theta, \phi)$ are still very complex, and they do not provide much insight into the structure of the solution, even in the simple Gaussian case.
 
 For more insight, let's separate the $\rho$ coordinate from the $\Theta$ coordinates.   The basic idea of the $(\rho, \Theta)$ coordinate system was to align the $\rho$ coordinate with the drift vector, $\nabla U({\bf x})$, so that the trajectory of our stochastic process in the direction of the $\Theta$ coordinates would be orthogonal to the drift.  The definition of our dissimilarity metric, $ \left( \,g_{ij}({\bf x})\, \right) $, also exhibited a strong separation between the $\rho$ coordinate and the $\Theta$ coordinates.  So there is a natural question here:  What is the relationship between the representation of our stochastic process in $\Theta$ coordinates and the $\Theta$ submatrix of $ \left( \,g_{ij}({\bf x})\, \right) $?
   
The answer is well known in the case of pure Brownian motion, without drift.  The earliest example is in \cite{zbMATH03323567} and \cite{ito1975}.  Stroock discovered that  if you project Brownian motion in ${\bf R}^{3}$ onto the surface of a sphere of radius $r$ centered at $(0,0,0)$, the infinitesimal generator of the resulting stochastic process, in spherical coordinates, $(r,\vartheta,\varphi)$, is:
\begin{equation*}
\mathcal{L} \;=\; \frac{1}{2} \, \frac{1}{r^{2}} \left(  \frac{\partial^{2}}{ {\partial \vartheta}^{2}}  \, + \,  
\frac{1}{\sin^{2} \vartheta} \, \frac{\partial^{2}}{ {\partial \varphi}^{2}}  \, + \,  
\frac{1}{\tan \vartheta} \, \frac{\partial}{\partial \vartheta} \right)
\end{equation*}
which is the spherical Laplacian divided by $2$.  This result can be generalized to an arbitrary Riemannian manifold, $\mathcal{M}$, embedded in ${\bf R}^{n}$.  For any $f \in C^{\infty}( \mathcal{M} ; {\bf R})$, the  \emph{Laplace-Beltrami} operator, $ \Delta_{\mathcal{M}} $, is defined by:
\begin{equation*}
 \Delta_{\mathcal{M}} \, f  \;=\; \mathrm{div}_{\mathcal{M}}  \left( \mathrm{grad}_{\mathcal{M}} \, f  \right)
\end{equation*}
in which the \emph{divergence}, $\mathrm{div}_{\mathcal{M}} $, and the \emph{gradient}, $\mathrm{grad}_{\mathcal{M}}$, can both be defined on $\mathcal{M}$ independently of a coordinate system.  
\begin{theorem}
\label{ProjectLaplaceBM}
Let $\mathcal{M}$ be an embedded submanifold of ${\bf R}^{n}$, and let $\Delta_{\mathcal{M}}$ be the Laplace-Beltrami operator on $\mathcal{M}$.  Then 
\begin{equation*}
\mathcal{L} \;=\;  \frac{1}{2} \Delta_{\mathcal{M}}
\end{equation*}
is the infinitesimal generator of a Brownian motion process in ${\bf R}^{n}$ that has been projected orthogonally onto $\mathcal{M}$.
\end{theorem}
\begin{proof}
The proof starts by showing that $\mathcal{L}$ can always be written in H\"{o}rmander form without the ${\bf V}_{0} \partial$ term.  In particular, we can write:
\begin{equation*}
\mathcal{L} \;=\;  \frac{1}{2} \Delta_{\mathcal{M}}  \;=\;  \frac{1}{2} \sum_{k = 1}^{n} \left( \Pi_{\mathcal{M}}( {\bf e}_{k} \partial ) \,\partial \right)^{2}
\end{equation*}
where $({\bf e}_{1}, {\bf e}_{2}, \ldots , {\bf e}_{n})$ is an orthonormal basis for ${\bf R}^{n}$ and $\Pi_{\mathcal{M}}$ is the orthogonal projection operator from the tangent bundle in ${\bf R}^{n}$ onto the tangent bundle in $\mathcal{M}$.  See Section 4.2.1 of \cite{stroock2000}  or Theorem 3.1.4 in \cite{hsu2002stochastic}.  From this result, it follows that we can construct a diffusion process on $\mathcal{M}$ whose increments are precisely the projections, under $\Pi_{\mathcal{M}}$, of the increments of a Brownian motion process in ${\bf R}^{n}$, and whose infinitesimal generator is $\mathcal{L} $.  For the details, see Theorem 4.37 in \cite{stroock2000}.
\end{proof}
\noindent
The diffusion process constructed in Theorem  \ref{ProjectLaplaceBM} is known as \emph{Brownian  motion on $\mathcal{M}$}. 
 
Let us now analyze the stochastic process defined by equation \eqref{wCauchy}, or \eqref{ItoXtDef}, or \eqref{StratXtDef}, projected onto the $\rho$ and $\Theta$ coordinates separately.  To simplify the calculations, we will initially focus our attention on the simple Gaussian case,  in which $\nabla U(x,y,z) = (- a x, - b y, - c z)$, and we will start with a construction borrowed from \cite{zbMATH03323567} and \cite{ito1975}, but adapted to match this example.  Consider the following matrix:
\begin{align}
 \begin{pmatrix}
    \\
   {\pi}^{i}_{j}(x,y,z) \\
    \\
\end{pmatrix}
 \; = \; 
 &\left(\begin{array}{ccc}1 & 0 & 0 \\0 & 1 & 0 \\0 & 0 & 1\end{array}\right) \; - \; \notag \\
 &\frac{1}{{| \nabla U(x,y,z) |}^{2}} \left(\begin{array}{c}-a \,x \\-b \,y \\-c \,z\end{array}\right) \left(\begin{array}{ccc}-a \,x & -b \,y & -c \,z\end{array}\right)  \notag  
 \end{align}
 in which the product in the second line should be interpreted as the multiplication of a $3 \times 1$ matrix times a $1 \times 3$ matrix, yielding a $3 \times 3$ matrix.  It is easy to check that $ \left( {\pi}^{i}_{j}({\bf x}) \right) $ is idempotent:
\begin{equation*}
\begin{pmatrix}
    \\
   {\pi}^{i}_{k}(x,y,z) \\
    \\
\end{pmatrix}
\;
\begin{pmatrix}
    \\
   {\pi}^{k}_{j}(x,y,z) \\
    \\
\end{pmatrix}
 \; = \; 
\begin{pmatrix}
    \\
   {\pi}^{i}_{j}(x,y,z) \\
    \\
\end{pmatrix}
\end{equation*}
and that it maps the vector $\nabla U(x,y,z)$ onto the origin:
\begin{equation*}
\begin{pmatrix}
    \\
   {\pi}^{i}_{j}(x,y,z) \\
    \\
\end{pmatrix}
\;
\left(\begin{array}{c}-a x \\-b y \\-c z\end{array}\right)
 \; = \; 
 \left(\begin{array}{c}0 \\0 \\0\end{array}\right)
\end{equation*}
Thus $ \left( {\pi}^{i}_{j}({\bf x}) \right) $ is a \emph{projection} onto the plane tangent to the integral manifold at $(x,y,z)$.
 
We now apply this projection operator to the right-hand side of equation \eqref{dXtStrat}, as rewritten above.  First, we set $\sigma$ equal to the identity matrix, so that $ \tilde {\bf b} = {\bf b} = \nabla U$.  (See the discussion following Lemma \ref{convertItoStrat} in Section \ref{MappingDiffusion}.)  Then the projection operator $ \left( {\pi}^{i}_{j}({\bf x}) \right) $ annihilates the second term in \eqref{dXtStrat}, and we are left with:
 \begin{align}
\label{dXtProj}
   dX(t) & \; = \; 
     \begin{pmatrix}
    \\
   {\pi}^{i}_{j}({\bf x}(\rho, \theta,\phi)) \\
    \\
    \end{pmatrix}
     \circ
      \begin{pmatrix}
         d\mathcal{B}_{1}(t)  \\
         d\mathcal{B}_{2}(t)  \\
         d\mathcal{B}_{3}(t)  \\
      \end{pmatrix}
 \end{align}
We now combine equation \eqref{dXtProj} with equation \eqref{dXtJacobian}, and solve this system of equations to obtain:
 \begin{align}
\label{SolveGaussianXProj}
       \begin{pmatrix}
         dX_{\rho}(t)  \\
         dX_{\theta}(t)  \\
         dX_{\phi}(t)  \\
      \end{pmatrix} 
      \; = \; 
      &
       \begin{pmatrix}
   \\
    \mathbf{ J }(\rho, \theta,\phi) \\
    \\
   \end{pmatrix} ^{-1}
     \begin{pmatrix}
    \\
   {\mathbf \pi}^{i}_{j}({\bf x}(\rho, \theta,\phi)) \\
    \\
    \end{pmatrix}
     \circ
      \begin{pmatrix}
         d\mathcal{B}_{1}(t)  \\
         d\mathcal{B}_{2}(t)  \\
         d\mathcal{B}_{3}(t)  \\
     \end{pmatrix} 
 \end{align}
 As a verification that our calculations are on the right track, we note that the multiplication of the two matrices on the right-hand side of \eqref{SolveGaussianXProj} produces a matrix in which the first row is identically zero.  This means that $ dX_{\rho}(t) = 0 $, which is exactly the result that we want.
  
We now continue the procedure outlined in Section \ref{proto}, applying Theorem \ref{Thm:L=sumAk+A0} to equation \eqref{SolveGaussianXProj}, and expanding the ``sum of squares'' inside \eqref{L=sumAk+A0}.  This allows us to compute the coefficients 
$ \alpha^{ij} (\rho, \theta, \phi)$  and  $\beta^{i}(\rho, \theta, \phi)$ in  \eqref{DiffusionCoeffs}.   It turns out that $ \alpha^{ij} (\rho, \theta, \phi) = 0 $ whenever $i = 0$ or $j = 0$, which is what we would expect.   For the remaining diffusion coefficients, we compute:
  \begin{align}
&\alpha^{11}(\rho, \theta, \phi)  \; = \; \frac{ a^{2} \, x(\rho, \theta,\phi)^{2} \;+\; c^{2} \, z(\rho, \theta,\phi)^{2}}{ a^{2} \, x(\rho, \theta,\phi)^{2} \, {| \nabla U |}^{2} }  \notag \\[2ex]
&\alpha^{22}(\rho, \theta, \phi)  \; = \; \frac{ a^{2} \, x(\rho, \theta,\phi)^{2} \;+\; b^{2} \, y(\rho, \theta,\phi)^{2}}{ a^{2} \, x(\rho, \theta,\phi)^{2} \, {| \nabla U |}^{2} }  \notag \\[2ex]
&\alpha^{12}(\rho, \theta, \phi) \; = \;  \alpha^{21}(\rho, \theta, \phi)  \; = \;  - \,\frac{ b \, c \, y(\rho, \theta,\phi) \, z(\rho, \theta,\phi) }{ a^{2} \, x(\rho, \theta,\phi)^{2} \, {| \nabla U |}^{2} }  \notag
 \end{align}
 Alternatively, we can write  the nonzero diffusion coefficients as a $ 2 \times 2 $ matrix:
 \begin{align}
&\left(\begin{array}{c} \alpha^{ij} (\rho, \theta, \phi) \end{array}\right) \;=\; \frac{1}{a^{2} \, x(\rho, \theta,\phi)^{2}} \;\times\; \notag \\[2ex]
&\left(  \left(\begin{array}{cc}1 & 0 \\0 & 1\end{array}\right)  - \frac{1}{{| \nabla U |}^{2}}   \left(\begin{array}{c} - b \,y(\rho, \theta,\phi) \\ - c \,z(\rho, \theta,\phi) \end{array}\right) \left(\begin{array}{cc} - b \,y(\rho, \theta,\phi) & - c \,z(\rho, \theta,\phi) \end{array}\right)  \right) \notag
\end{align}
It turns out also that the the drift coefficient $\beta^{0}(\rho, \theta, \phi) = 0$, as we would expect, and for the other drift coefficients we compute:
 \begin{align}
\beta^{1}(\rho, \theta, \phi)  \; = \;  &\frac{ b \, y(\rho, \theta,\phi)}{2 \,a \, x(\rho, \theta,\phi) \, {| \nabla U |}^{2} }\; \times \notag  \\ 
&\left(  ( b + c )
\; + \; \frac{1}{{| \nabla U |}^{2}}  \left(\begin{array}{c} a^{2} \, x(\rho, \theta,\phi) \\ b^{2} \, y(\rho, \theta,\phi) \\ c^{2} \, z(\rho, \theta,\phi) \end{array}\right) \cdot \nabla U \right)  
\notag  \\[2ex] 
\beta^{2}(\rho, \theta, \phi)  \; = \;  &\frac{ c \, z(\rho, \theta,\phi)}{2 \,a \, x(\rho, \theta,\phi) \, {| \nabla U |}^{2} }\; \times \notag  \\ 
&\left(  ( b + c )
\; + \; \frac{1}{{| \nabla U |}^{2}}  \left(\begin{array}{c} a^{2} \, x(\rho, \theta,\phi) \\ b^{2} \, y(\rho, \theta,\phi) \\ c^{2} \, z(\rho, \theta,\phi) \end{array}\right) \cdot \nabla U \right)  
\notag  
\end{align}
Keep in mind that these are the coefficients for the first-order terms  $\partial / \partial \theta$ and  $\partial / \partial \phi$.  

For comparison, we will now compute the Laplace-Beltrami operator for the simple Gaussian case, using our Riemannian dissimilarity metric,  $g_{ij}(\rho, \theta, \phi) =  g_{ij}({\bf x}(\rho, \theta, \phi))$, on the two-dimensional integral manifold given by the Theorem of Frobenius.  In a local coordinate system, the Laplace-Beltrami operator is usually written as follows:
\begin{equation*}
 \Delta_{\mathcal{M}} \, f  \;=\;  \frac{1}{\sqrt{G}} \,  \sum_{j=1}^{n}  \frac{ \partial }{ \partial u^{j}} 
 \left(  \sqrt{G} \, \sum_{i=1}^{n} g^{ij}({\bf u})  \frac{ \partial f }{ \partial u^{i}} \right)
\end{equation*}
where $G$ is the determinant of the matrix $\left( \,g_{ij}({\bf u})\, \right)$ and $\left( \,g^{ij}({\bf u})\, \right)$ is its inverse.  Alternatively, we can expand the expression inside the parentheses, and write $\mathcal{L}$ in the form of equation \eqref{DiffusionCoeffs}:
\begin{align}
\mathcal{L} \;=\;  \frac{1}{2} \Delta_{\mathcal{M}} &\;=\; \frac{1}{2} \, \sum_{i,j=1}^{n} 
 g^{ij} ({\bf u}) \frac{\partial^{2}}{\partial u^{i} \partial u^{j}} 
\, + \,
 \sum_{i=1}^{n} h^{i}({\bf u}) \frac{\partial}{\partial u^{i} },  \notag \\[1ex]
\mathrm{with} \;\; h^{i}({\bf u})  &\;=\; \frac{1}{2  \sqrt{G} } \sum_{j=1}^{n} \frac{\partial \left(  \sqrt{G} \,  g^{ij} ({\bf u}) \right)}{ \partial u^{j}}
\notag
\end{align}
When we do the calculations in the simple Gaussian case, with $n=2$, we discover that the diffusion coefficients are identical: 
\begin{align}
&\left(\begin{array}{c} \alpha^{ij} (\rho, \theta, \phi) \end{array}\right) \;=\; \left(\begin{array}{c} g^{ij} (\rho, \theta, \phi) \end{array}\right)
\notag
\end{align}
and the drift coefficients are similar, but not identical:
 \begin{align}
h^{1}(\rho, \theta, \phi)  \; = \;  &\frac{ b \, y(\rho, \theta,\phi)}{2 \,a \, x(\rho, \theta,\phi) \, {| \nabla U |}^{2} }\; \times \notag  \\ 
&\left(  ( a + b + c )
\; + \; \frac{1}{{| \nabla U |}^{2}}  \left(\begin{array}{c} a^{2} \, x(\rho, \theta,\phi) \\ b^{2} \, y(\rho, \theta,\phi) \\ c^{2} \, z(\rho, \theta,\phi) \end{array}\right) \cdot \nabla U \right)  
\notag  \\[2ex] 
h^{2}(\rho, \theta, \phi)  \; = \;  &\frac{ c \, z(\rho, \theta,\phi)}{2 \,a \, x(\rho, \theta,\phi) \, {| \nabla U |}^{2} }\; \times \notag  \\ 
&\left(  ( a + b + c )
\; + \; \frac{1}{{| \nabla U |}^{2}}  \left(\begin{array}{c} a^{2} \, x(\rho, \theta,\phi) \\ b^{2} \, y(\rho, \theta,\phi) \\ c^{2} \, z(\rho, \theta,\phi) \end{array}\right) \cdot \nabla U \right)  
\notag  
\end{align}
In fact, there is a simple relationship between the coefficients $\beta^{i}(\rho, \theta, \phi)$ and $h^{i}(\rho, \theta, \phi)$:
 \begin{align}
\label{DeltaDrifteqn} 
&\beta^{1}(\rho, \theta, \phi)  - h^{1}(\rho, \theta, \phi) \; = \;  \,-\, \frac{ b \, y(\rho, \theta,\phi)}{2 \, x(\rho, \theta,\phi) \, {| \nabla U |}^{2} }
\\
&\beta^{2}(\rho, \theta, \phi) - h^{2}(\rho, \theta, \phi) \; = \; \,-\, \frac{ c \, z(\rho, \theta,\phi)}{2 \, x(\rho, \theta,\phi) \, {| \nabla U |}^{2} }
\notag  
\end{align}
Is there an explanation for these results?
 
The key is to recognize that the stochastic process defined by equation \eqref{wCauchy}, or \eqref{ItoXtDef}, or \eqref{StratXtDef}, is \emph{not} Brownian motion.  Brownian motion in ${\bf R}^{n}$ dissipates, and does not generate an invariant probability measure.   Thus the projection of Brownian motion onto a Riemannian manifold, $\mathcal{M}$, would dissipate as well.  But the stochastic process defined by equation \eqref{wCauchy}, when projected onto the manifold, $\mathcal{M}$, would not dissipate, in general.  This difference must be reflected in the drift coefficients for $\partial / \partial \theta$ and  $\partial / \partial \phi$, as shown by equation \eqref{DeltaDrifteqn}.  

\begin{figure}[htbp]
\begin{center}
\includegraphics[width=4.0in]{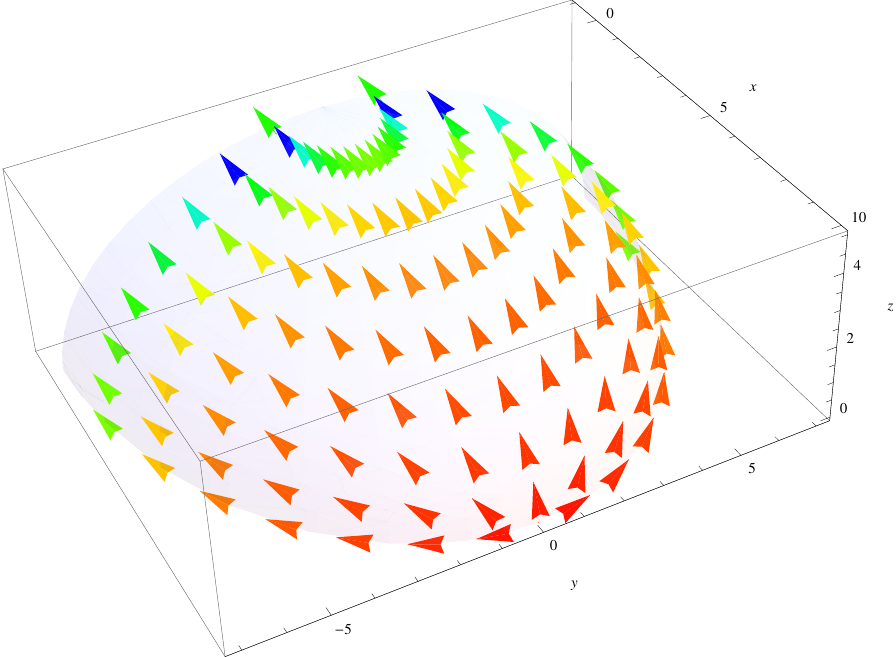}
\caption{The drift correction vector field for a Gaussian diffusion.}
\label{GaussianDeltaDrift}
\end{center}
\end{figure}

Figure \ref{GaussianDeltaDrift} shows the ``drift correction vector field'' generated by equation \eqref{DeltaDrifteqn} on one quadrant of the integral manifold through $(10,0,0)$, for the simple Gaussian case.   The magnitude of the vector field is coded by color, with red indicating that the length of the vector is near zero.  Keep in mind that we are looking at the \emph{difference} between the two vector fields, $\left( \,\beta^{i}(\rho, \theta, \phi)\, \right) $ and $ \left( \,h^{i}(\rho, \theta, \phi)\, \right) $, as defined by equation \eqref{DeltaDrifteqn}.  The vector fields themselves are oriented (approximately) in the opposite direction, but they have different magnitudes.

We have presented detailed calculations for the simple Gaussian case, so that our results would be easy to visualize.  But the same calculations work for the general case, $ \nabla U({\bf x}) = ( P({\bf x}), Q({\bf x}), R({\bf x}) ) $.  The projection operator, $ \left( \,{\pi}^{i}_{j}({\bf x}(\rho, \theta, \phi)) \, \right) $, and the Jacobian matrix,  $ \left( \, \mathbf{ J }(\rho, \theta, \phi) \, \right) $, can be defined in the same way, and the computational procedure from Section \ref{proto}, applying Theorem \ref{Thm:L=sumAk+A0} and expanding equation \eqref{L=sumAk+A0}, still goes through.  The expansion of the Laplace-Beltrami operator for the general dissimilarity metric,  $g_{ij}(\rho, \theta, \phi) =  g_{ij}({\bf x}(\rho, \theta, \phi))$, also goes through.  We end up, again, with diffusion coefficients that are identical:
\begin{align}
&\left(\begin{array}{c} \alpha^{ij} (\rho, \theta, \phi) \end{array}\right) \;=\; \left(\begin{array}{c} g^{ij} (\rho, \theta, \phi) \end{array}\right) \;=\; 
\notag \\[2ex]
& \frac{1}{P^{2}(\,{\bf x}\,)} \;\times\;  \left(  \left(\begin{array}{cc}1 & 0 \\0 & 1\end{array}\right)  - 
\frac{1}{{| \nabla U |}^{2}}   \left(\begin{array}{c} Q(\,{\bf x}\,) \\ R(\,{\bf x}\,) \end{array}\right) \left(\begin{array}{cc} Q(\,{\bf x}\,) & R(\,{\bf x}\,)  \end{array}\right)  \right) \notag
\end{align}
\noindent
and drift coefficients that differ by a single, but more complex,  term:
\begin{align}
&\beta^{1}(\rho, \theta, \phi)  \,-\, h^{1}(\rho, \theta, \phi) \; = \;  \notag \\[2ex]
& \,-\,\frac{1}{2 \, P^{2} \, {| \nabla U |}^{2}} \;
 \left( P \left(  P \, \frac{\partial Q}{\partial x} \,-\, Q \, \frac{\partial P}{\partial x} \right) + R \left(  R \, \frac{\partial Q}{\partial x} - Q \, \frac{\partial R}{\partial x}  \right) \right)
\notag  
\end{align}
 \begin{align}
&\beta^{2}(\rho, \theta, \phi)  \,-\, h^{2}(\rho, \theta, \phi) \; = \;   \notag \\[2ex]
& \,-\,\frac{1}{2 \, P^{2} \, {| \nabla U |}^{2}} \;
\left( P \left(  P \, \frac{\partial R}{\partial x} \,-\, R \, \frac{\partial P}{\partial x} \right) + Q \left(  Q \, \frac{\partial R}{\partial x} - R \, \frac{\partial Q}{\partial x}  \right) \right)
\notag  
\end{align}

\vspace{1ex}
\noindent
Note that $ {\partial P} /  {\partial x}$ is the only partial derivative in these drift correction equations which is nonzero in the case $\nabla U(x,y,z) = (- a x, - b y, - c z)$.   Thus, for the simple Gaussian case, we can easily verify that the coefficients of the drift correction vector field reduce to the two terms:
\begin{align}
 - \frac{ b \, y(\rho, \theta,\phi)}{2 \, x(\rho, \theta,\phi) \, {| \nabla U |}^{2} } 
 \hspace{1em} {\rm and} \hspace{1em} 
  - \frac{ c \, z(\rho, \theta,\phi)}{2 \, x(\rho, \theta,\phi) \, {| \nabla U |}^{2} }
\notag  
\end{align}
in agreement with equation  \eqref{DeltaDrifteqn}.

\begin{figure}[htbp]
\begin{center}
\includegraphics[width=4.0in]{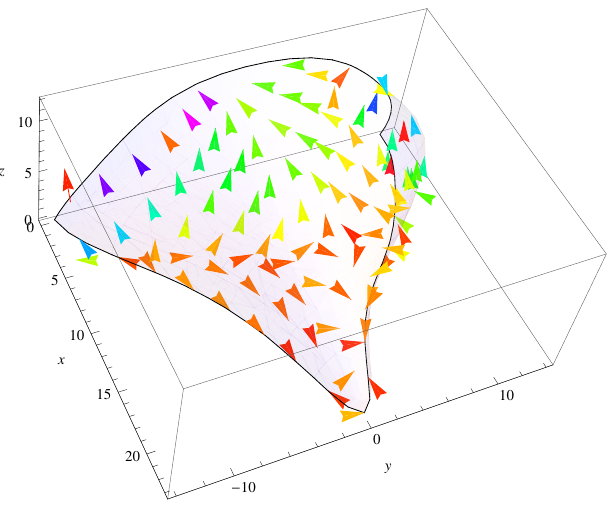}
\caption{The drift correction vector field for a curvilinear Gaussian diffusion.}
\label{CurvGaussianDeltaDrift}
\end{center}
\end{figure}
 
Figure \ref{CurvGaussianDeltaDrift}  is the analogue of Figure \ref{GaussianDeltaDrift} for the curvilinear Gaussian diffusion, using the general equation for the difference of the drift coefficients.  In this plot, the positive \emph{x}-axis has been aligned with the principal axis, and the positive \emph{y}-axis has been aligned with the maximal eigenvector $\xi_{1}({\bf x})$ displaced through the angle $\alpha = 0.114166$. Thus the \emph{y}-axis lines up with the $\theta^{1}$ coordinate curve in Figure \ref{CGCoordsWithEig}, and the drift correction vector field gives some sense of why the clustering of data points in Figure \ref{CGCoordsWithData} occurs.

\section{Future Work}
\label{FutureWork}
\vspace{1ex}

The theory of differential similarity combines a stochastic model with a geometric model, and it works because there is a common mathematical object in both models:  the gradient, $\nabla U({\bf x})$, of a potential function, $U({\bf x})$.  In the stochastic model, $\nabla U({\bf x})$ is the drift vector, which guarantees the existence of an invariant probability measure.  In the geometric model, $\nabla U({\bf x})$ guarantees the existence of an orthogonal integral manifold.  We have seen, in Section \ref{DiffuCoeffs&DissimMetrics}, that there is a theoretical connection between these two models, in which $\nabla U({\bf x})$ plays a crucial role, and we have seen the practical consequences of this connection in the computational examples in Sections \ref{Exp:Gauss} and \ref{Exp:CurvGauss}.
 
Perhaps the most striking result of this work is the distribution of data points in Figure \ref{CGCoordsWithData} in Section \ref{Exp:CurvGauss}.  The $\theta^{1}$ and $\theta^{2}$ coordinate curves were defined using only the geometric model, but our analysis of the reconstruction error shows that projection of the data onto the $\theta^{2}$ subspace has better statistical properties than projection onto the $\theta^{1}$ subspace. Thus the link between the stochastic model and the geometric model has computational implications.
 
The main deficiency in the theory, as presented in this paper, is the restriction of the geometric model to the three-dimensional case.   We imposed this restriction to simplify the calculations, and to make it easy to visualize the examples in \emph{Mathematica}.  But the theory is not inherently limited to three dimensions.  Theorem \ref{2DIntManifold} in Section \ref{Exp:IntManifolds} was written using the vector cross product and the ``curl,'' which is a three-dimensional concept, but it is actually a special case of a general result in ${\bf R}^{n}$ which follows from the dual version of the Theorem of Frobenius, expressed in terms of differential forms.  It follows that ${\bf V} \partial$ and ${\bf W} \partial$, the basis vectors for the tangent subbundle in ${\bf R}^{3}$, can be generalized to ${\bf R}^{n}$.  

Our current work extends the theory of differential similarity in three ways:

First, if $\nabla U({\bf x}) \,=\, (\,  P_{0}({\bf x}) ,  P_{1}({\bf x}) , \ldots , P_{n - 1}({\bf x}) \,)$, we define the basis vectors as follows:
 \begin{align}
\nabla U({\bf x}) \; &= \; ( \hspace{1em} P_{0}({\bf x}) ,& P_{1}({\bf x}) ,& &P_{2}({\bf x}) ,& &\ldots ,& &P_{n - 2}({\bf x}) ,& &P_{n - 1}({\bf x}) \;) \notag \\
 {\bf V}_{1}({\bf x}) \; &= \; (\; - P_{1}({\bf x}) ,& P_{0}({\bf x}) ,& &0 ,& &\ldots ,& &0 ,& &0 \;) \notag \\
 {\bf V}_{2}({\bf x}) \; &= \; (\; - P_{2}({\bf x}) ,& 0,& &P_{0}({\bf x})  ,& &\ldots ,& &0 ,& &0 \;) \notag \\
 & \ldots & & & & & & & & & \notag \\
 {\bf V}_{n - 2}({\bf x}) \; &= \; (\; - P_{n - 2}({\bf x}) ,& 0 ,& &0 ,& &\ldots ,& &P_{0}({\bf x}) ,& &0 \;) \notag \\
 {\bf V}_{n - 1}({\bf x}) \; &= \; (\; - P_{n - 1}({\bf x}) ,& 0 ,& &0 ,& &\ldots ,& &0 ,& &P_{0}({\bf x}) \;) \notag
 \end{align}
It is straightforward to verify that $\nabla U({\bf x})$ is orthogonal to each $  {\bf V}_{i}({\bf x}) $, and that the tangent subbundle spanned by $ \left\{ {\bf V}_{i} \partial =  {\bf V}_{i}({\bf x}) / P_{0}({\bf x})\right\} $ satisfies the Frobenius integrability conditions.  For the Riemannian dissimilarity metric, the obvious generalization is to define a metric tensor on all of ${\bf R}^{n}$, using the inner products of  $\nabla U({\bf x})$, $ {\bf V}_{1}({\bf x})$,  ${\bf V}_{2}({\bf x})$, $\ldots$ ,  and ${\bf V}_{n - 1}({\bf x}) $.  Thus:
\begin{align}
&
 \begin{pmatrix}
    \\
   {g}_{i,j}({\bf x}) \\
    \\
\end{pmatrix}
 \; = \;  \notag \\
 \notag \\
 &
\left(\begin{array}{cccccc}  | \nabla U |^{2}  & 0 & 0 & \ldots & 0 & 0 \\[1ex]
0 & P_{0}^{2} + P_{1}^{2} & P_{2}P_{1} & \ldots & P_{n-2}P_{1} & P_{n-1}P_{1} \\[1ex]
0 & P_{1}P_{2} & P_{0}^{2} + P_{2}^{2} & \ldots & P_{n-2}P_{2} & P_{n-1}P_{2} \\[1ex]
 \ldots & & & & & \ldots \\[1ex]
0 & P_{1}P_{n-2} & P_{2}P_{n-2}  & \ldots & P_{0}^{2} + P_{n-2}^{2}& P_{n-1}P_{n-2} \\[1ex] 
0 & P_{1}P_{n-1} & P_{2}P_{n-1} & \ldots & P_{n-2}P_{n-1} & P_{0}^{2} + P_{n-1}^{2} \\[1ex]
\end{array}\right)
\notag
\end{align}
Note that we can recover the three-dimensional case by restricting these definitions to $P_{0}({\bf x})$, $P_{1}({\bf x})$, $P_{2}({\bf x})$.

Second, when we diagonalize the Riemannian dissimilarity matrix in $n$ dimensions, we discover that there are only two distinct eigenvalues. The largest eigenvalue has multiplicity $2$: $\lambda_{0} = \lambda_{1} =  | \nabla U |^{2}$, with corresponding eigenvectors: 
\begin{align}
\xi_{0} =\left(\begin{array}{c} 1 \\ 0 \\ 0 \\ \cdots \\ 0 \\ 0 \end{array}\right) \;\; \mbox{\rm and} \;\;
\xi_{1} =\left(\begin{array}{c} 0 \\ P_{1} \\ P_{2} \\ \cdots \\ P_{n-2} \\ P_{n-1} \end{array}\right) \;
\notag 
\end{align} 
The smallest eigenvalue has multiplicity $n-2$: $\lambda_{2} = \lambda_{3} = \ldots = \lambda_{n-2} = \lambda_{n-1} =  P_{0}^{2} $, with corresponding eigenvectors $\xi_{2}, \;\xi_{3}, \;\ldots , \;\xi_{n-2}, \;\xi_{n-1}$, as follows:
\begin{align}
\left(\begin{array}{c} 0 \\ -P_{2} \\ P_{1} \\ 0 \\ \cdots \\ 0 \\ 0 \\ \end{array}\right) , \;
\left(\begin{array}{c} 0 \\ -P_{3} \\ 0 \\ P_{1} \\ \cdots \\ 0 \\ 0 \\ \end{array}\right) , \;
\ldots \; , \;
\left(\begin{array}{c} 0 \\ -P_{n-2} \\ 0 \\ 0 \\ \cdots \\  P_{1} \\ 0 \\ \end{array}\right) , \;
\left(\begin{array}{c} 0 \\ -P_{n-1} \\ 0 \\ 0 \\ \cdots \\  0 \\ P_{1} \\ \end{array}\right) 
\notag 
\end{align} 
We still face the problem that we addressed in Sections \ref{Exp:Gauss} and \ref{Exp:CurvGauss}:  How to convert a locally optimal coordinate system (based on infinitesimal eigenvectors) into a globally optimal solution (based on geodesic curves over finite distances)?  But the rich structure of the eigenvectors in the \emph{n}-dimensional case suggests a somewhat different strategy.

Third, to estimate $\nabla U({\bf x})$ from sample data in a high-dimensional Euclidean space, we borrow a technique from the literature on the \emph{mean shift algorithm} \cite{Fukunaga&Hostetler:1975} \cite{Cheng:1995} \cite{Comaniciu&Meer:2002}.  Consider a \emph{kernel density estimator} with a Gaussian kernel:
\begin{equation*}
K(\mathbf{s}_{k}, \mathbf{x}) = \exp( - \beta \; \| \mathbf{s}_{k} - \mathbf{x} \|^{2}),
\end{equation*}
in which $\mathbf{s}_{k}$ is a sample data point and $\beta$ is a smoothing parameter.  We can approximate the probability density by taking the average over these kernels:
\begin{equation*}
\hat{\mu}( \mathbf{x} ) = \frac{1}{m} \sum_{k=1}^{m} K(\mathbf{s}_{k}, \mathbf{x}).
\end{equation*}
Now recall that $\nabla U({\bf x})$ is the \emph{gradient} of the \emph{log} of the stationary probability density.  So we can differentiate explicitly:
\begin{equation*}
\frac{\partial}{\partial x^{j}} \log \hat{\mu}( \mathbf{x} ) = 2 \beta \Bigg[ \;  \frac{ \sum_{k=1}^{m} K(\mathbf{s}_{k}, \mathbf{x}) \; s_{k}^{j} }{ \sum_{k=1}^{m} K(\mathbf{s}_{k}, \mathbf{x}) }  \; - \; x^{j} \; \Bigg],
\end{equation*}
to obtain an estimate for $\nabla U({\bf x})$.  From there, we can derive an analytical expression for the Euler-Lagrange equations that depends only on the sample data points: $\{ \mathbf{s}_{k} \}$.

The details of these extensions will be presented in a forthcoming paper, with the working title:  ``Differential Similarity in Higher Dimensional Spaces:  Theory and Applications.''  To illustrate the theory, outlined above, the forthcoming paper will also describe applications of these ideas to the classical MNIST dataset \cite{LeCun_etal_1990} and to the CIFAR-10 dataset  \cite{Krizhevsky:2009}. 
   
Once the theory is extended to higher dimensions, it will become apparent that there are various connections to recent work on \emph{manifold learning}, as described in Section \ref{Intro}.  Work in this area tends to follow either a geometric approach or a probabilistic approach, but not both.  
Examples of the geometric approach include:  \cite{TenenbaumEtAl2000}  \cite{Roweis2000}  \cite{belkin2003laplacian} \cite{DonohoGrimes2003}.  Belkin and Niyogi \cite{belkin2003laplacian}, for example, work with the eigenvectors of the graph Laplacian and the eigenfunctions of the Laplace-Beltrami operator, and show that the solution to these eigenproblems yields an ``optimal'' embedding of a low-dimensional manifold into a higher-dimensional space, but their arguments are geometric rather than probabilistic.  Examples of the probabilistic approach include:  \cite{TippingBishop1999}  \cite{HintonRoweis2002} \cite{Chen_etal.MFA.2010}.  Tipping and Bishop \cite{TippingBishop1999} work with a mixture of low-dimensional Gaussians embedded in a higher-dimensional space, each with its own mean and covariance matrix, and they use the EM algorithm to estimate the parameters of this model.  Chen, et al., \cite{Chen_etal.MFA.2010} adopt a similar model, along with the assumption that the Gaussian mixture covers a low-dimensional manifold, and they estimate both the number of components in the mixture and the dimensionality of the subspaces, using Bayesian techniques.  But neither paper makes use of the geometric structure of the embedded manifold.
 
One exception to this dichotomy between geometric and probabilistic approaches is a paper by Lee and Wasserman \cite{LeeWasserman2010}, which has some interesting connections to the present work.   The paper starts out by defining a Markov chain on ${\bf R}^{n}$ with a transition kernel $\Omega_{\epsilon} ( {\bf x} , \cdot )$ which gives preference to nearby points, ${\bf y}$, that have a high probability density, $p({\bf y})$.  This kernel is then used to define the one-step diffusion operator, $A_{\epsilon}$, and its $m$-step version, $A_{\epsilon, m}$.  The authors then construct a continuous time operator:  ${\bf A}_{t} = \lim_{\epsilon \rightarrow 0} A_{\epsilon, \, t / \epsilon}$.  The analogous mathematical object in our theory would be the operator ${\bf Q}_{t}$ in equation \eqref{defQt}.  Lee and Wasserman are primarily interested in the eigenfunctions of $A_{\epsilon, m}$ and ${\bf A}_{t}$, which have applications to various \emph{spectral clustering} problems, following the work of Belkin and Niyogi \cite{belkin2003laplacian} and others.  They also use $\Omega_{\epsilon} ( {\bf x} , \cdot )$ to define a \emph{diffusion distance}, $D^{2}_{\epsilon, m} ({\bf x},{\bf z})$, and its continuous time version, ${\bf D}^{2}_{t} ({\bf x},{\bf z})$, but there does not seem to be a straightforward relationship between this distance and our dissimilarity metric, ${g}_{ij}({\bf x})$. The paper concludes with several examples that demonstrate the utility of these concepts.
 
The other important contribution of Lee and Wasserman \cite{LeeWasserman2010} is their analysis of the statistical estimators for the population quantities, ${\bf A}_{t}$ and ${\bf D}^{2}_{t}$.  This is essential future work for our theory as well.  The technical results of Arias-Castro, Mason and Pelletier \cite{Arias-Castro_etal:2016} on the convergence and consistency of the mean shift algorithm suggest one way to proceed, although their analysis would have to be extended from integral curves in ${\bf R}^{n}$ to geodesics on manifolds. 

There are many other studies of low-dimensional Riemannian manifolds embedded in higher-dimensional Euclidean spaces, and various techniques to estimate their properties.  One early example is Brand \cite{Brand2002}, and subsequent work includes \cite{Zhang&Zha:2004} \cite{Yu&Zhang&Gong:NIPS2009} \cite{Chen&Zhang&Fleischer:2010}.  Many of these studies also make use of a diffusion process on the manifold, as a basic tool.  The main reference is Coifman and Lafon \cite{Coifman&Lafon:2006}. Often, the stochastic process is defined initially on a finite graph (e.g., as a random walk) and the diffusion on a manifold is shown to be the limiting case. See, e.g., \cite{Belkin&Nyogi:COLT2005} \cite{Hein&Audibert&vonLuxburg:2007} \cite{Ting&Huang&Jordan:ICML2010}.  The present research appears to be novel in two respects: (1) we work with a diffusion process in which the drift vector plays the primary role, and (2) we work with the Riemannian manifold generated by the Theorem of Frobenius.  It is unclear whether there are discrete approximations to this model, but it would be an interesting question to investigate.  

Throughout the paper, we have cited and quoted the three hypotheses motivating the work of Rifai, et al.\  \cite{nipsRifaiDVBM11}. Here is their first hypothesis again, without elisions:
\begin{quotation}
\begin{enumerate}

\item[1.]  The {\bf semi-supervised learning hypothesis}, according to which learning aspects of the input distribution $p(x)$ can improve models of the conditional distribution of the supervised target $p(y | x)$, i.e., $p(x)$ and $p(y | x)$ share something (Lasserre, \emph{et al.}, \cite{Lasserre&Bishop&Minka:CVPR2006}). This hypothesis underlies not only the strict semi-supervised setting where one has many more unlabeled examples at his disposal than labeled ones, but also the successful \emph{unsupervised pre-training} approach for learning deep architectures, which has been shown to significantly improve supervised performance even without using \emph{additional} unlabeled examples (Hinton, \emph{et al.} \cite{Hinton_etal_2006}; Bengio \cite{Bengio:2009}; Erhan, \emph{et al.} \cite{Erhan_etal:2010}).

\item[2.]  $\ldots$

\item[3.]  $\ldots$

\end{enumerate}
\end{quotation}
In the body of their paper, Rifai, et al., show how an unsupervised pre-training model for learning deep architectures can be built on top of a form of manifold learning.  In several experiments, they extract a tangent plane at each training point using a Contractive Auto-Encoder (CAE), which is an unsupervised learning algorithm, and they then exploit these learned tangents to train a network using a supervised learning algorithm that is sensitive to tangent directions. 
They write: 
\begin{quotation}
To the best of our knowledge this is the first time that the implicit relationship between an unsupervised learned mapping and the tangent space of a manifold is rendered explicit and successfully exploited for the training of a classifier.
\end{quotation}
Although, in practice, unsupervised pre-training is no longer a popular machine learning technique, since fully supervised deep learning has become so successful, 
there is still a great deal of interest in the connections between various types of auto-encoders and the field of manifold learning.  The basic ideas were summarized in an influential paper by Bengio, Courville and Vincent in 2013 \cite{Bengio_etal:2013}.  For example, in Section 7 of their paper, the authors show that Denoising Auto-Encoders (DAEs) compute the gradient of the log of the probability density, i.e., they compute our familiar vector field $\nabla U({\bf x})$.  See \cite{Vincent:2014} \cite{Alain&Bengio:2012}.  A broader and more speculative view of manifold learning is presented in Section 8 of their paper, which has the title:  ``Representation Learning as Manifold Learning.'' It thus seems clear that our agenda for future research should include a study of the role of differential similarity in deep learning.
 
\bibliographystyle{alpha}
\bibliography{CCCSimilarity2016}
 
\end{document}